\documentclass[twoside,11pt]{article}

\usepackage{jmlr2e}
\usepackage{amsmath}
\usepackage{mdframed}
\usepackage{algorithm}
\usepackage{algpseudocode}
\usepackage{subfig}

\usepackage{macros}
\allowdisplaybreaks

\usepackage{lastpage}

\jmlrheading{21}{2020}{1-\pageref{LastPage}}{1/19; Revised
5/20}{7/20}{19-073}{Andrei Kulunchakov and Julien Mairal}
\ShortHeadings{Estimate Sequences for Stochastic Composite Optimization}{Kulunchakov and Mairal}
\firstpageno{1}

\begin{document}

\title{Estimate Sequences for Stochastic Composite Optimization: Variance Reduction, Acceleration, and Robustness to Noise}

\author{\name Andrei Kulunchakov \email andrei.kulunchakov@inria.fr \\
       \name Julien Mairal \email  julien.mairal@inria.fr \\
       \addr 
      Univ. Grenoble Alpes, Inria, CNRS, Grenoble INP, LJK, 38000 Grenoble,
      France}

\editor{Sathiya Keerthi}

\maketitle

\begin{abstract}
In this paper, we propose a unified view of gradient-based algorithms for
stochastic convex composite optimization by extending the concept of estimate
sequence introduced by Nesterov. More precisely, we interpret a large class of stochastic
optimization methods as procedures that iteratively minimize a surrogate of
the objective, which 
covers the stochastic gradient descent method and variants of 
the incremental approaches SAGA, SVRG, and MISO/Finito/SDCA.
This point of view has several advantages: (i) we provide a simple generic
proof of convergence for all of the aforementioned methods; (ii) we naturally
obtain new algorithms with the same guarantees; (iii) we derive generic
strategies to make these algorithms robust to stochastic noise, which is useful
when data is corrupted by small random perturbations. Finally, we propose a new
accelerated stochastic gradient descent algorithm and
a new accelerated SVRG algorithm that is robust to stochastic noise.
\end{abstract}


\section{Introduction}
We consider convex composite optimization problems of the form
\begin{equation}
\min_{x \in \Real^p} \left\{ F(x) \defin f(x)  + \psi(x) \right\}, \label{eq:prob}
\end{equation}
where $f$ is convex and $L$-smooth\footnote{A function is $L$-smooth when it is differentiable and its derivative is Lipschitz continuous with constant~$L$.},
and we call $\mu$ its strong convexity modulus with respect to the Euclidean norm.\footnote{Then, $\mu=0$ means that the function is convex but not strongly convex.}
The function $\psi$ is convex lower
semi-continuous and is not assumed to be necessarily differentiable.
For instance, $\psi$ may be the $\ell_1$-norm, which is very popular in signal
processing and machine learning for its sparsity-inducing properties~\citep[see][and references therein]{mairal2014sparse}; $\psi$ may also be the
extended-valued indicator function of a convex set~$\Ccal$ that takes the value
$+\infty$ outside of~$\Ccal$ and $0$ inside such that the previous setting
encompasses constrained problems~\citep[see][]{hiriart_urruty_lemarechal_1993ii}.

More specifically, we focus on stochastic objective functions, which are of
utmost importance in machine learning, where $f$ is an expectation or a finite
sum of smooth functions
\begin{equation}
f(x) = \E_\xi\left[\tilde{f}(x,\xi)\right] \qquad \text{or} \qquad f(x) = \frac{1}{n} \sum_{i=1}^n f_i(x). 
\label{eq:f}
\end{equation}
On the left, $\xi$ is a random variable representing a data point drawn according
to some distribution and $\tilde{f}(x,\xi)$ measures the fit of some model
parameter~$x$ to the data point~$\xi$. Whereas the
explicit form of the data distribution is unknown, we assume that we can draw random i.i.d. samples $\xi_1, \xi_2, \ldots$
Either an infinite number of such samples are available and the
problem of interest is to minimize~(\ref{eq:prob}) with $f(x)=\E_\xi[\tilde{f}(x,\xi)]$, or
one has access to a finite training set only, leading to the finite-sum setting
on the right of~(\ref{eq:f}), called empirical risk~\citep{vapnik2013nature}.

While the finite-sum setting is obviously a particular case of expectation with
a discrete probability distribution, the \emph{deterministic} nature of the
resulting cost function drastically changes performance guarantees.
In particular, when
an algorithm is only allowed to access unbiased measurements of the objective function
and gradient---which we assume is the case when~$f$ is an expectation---it may
be shown that the worst-case convergence rate in expected function value cannot
be better than $O(1/k)$ in general, where $k$ is the number of
iterations~\cite{nemirovski,agarwal2009information}.  Such a sublinear rate of
convergence is notably achieved by stochastic gradient descent (SGD) algorithms or
their variants~\citep[see][]{bottou2018optimization}.

Even though this pessimistic result applies to the general stochastic case, linear
convergence rates can be obtained for the finite-sum setting~\cite{schmidt2017minimizing}.
Specifically, a large body of work in machine learning has
led to many randomized incremental approaches obtaining linear convergence rates, such as
SAG~\cite{schmidt2017minimizing}, SAGA~\cite{saga},
SVRG~\cite{johnson2013accelerating,proxsvrg}, SDCA~\cite{accsdca}, MISO~\cite{miso}, Katyusha~\cite{accsvrg}, MiG~\cite{MiG2018}, SARAH~\cite{sarah}, directly accelerated SAGA~\cite{zhou2018direct} or the method of~\citet{conjugategradient}.
For non-convex objectives, recent approaches have also improved known convergence rates for finding first-order stationary points~\cite{spider2018,paquette2018catalyst,scsg2017}, which is however beyond the scope of our paper.
These algorithms have about the same cost per-iteration as the stochastic
gradient descent method, since they access only a single or two gradients $\nabla f_i(x)$
at each iteration, and they may achieve lower computational complexity  than accelerated gradient descent
methods~\cite{nesterov1983,nesterov,nesterov2013gradient,fista} in expectation.
A common interpretation is to see these algorithms as performing SGD steps with an estimate of the full gradient that has lower variance~\cite{proxsvrg}.

In this paper, we are interested in providing a unified view of stochastic optimization
algorithms, but we also want to investigate their \emph{robustness} to random
perturbations. Specifically, we may consider objective functions with an explicit
finite-sum structure such as~(\ref{eq:f}) when only noisy
estimates of the gradients~$\nabla f_i(x)$ are available.
Such a setting may occur for various reasons. For instance, perturbations may be injected during training in order to achieve better
generalization on new test data~\cite{srivastava_dropout:_2014}, perform stable feature selection~\citep{meinshausen2010stability},
improve the model robustness~\cite{zheng2016improving}, or for privacy-aware learning~\citep{wainwright2012privacy}.

Each training point indexed by~$i$ is
corrupted by a random perturbation $\rho_i$ and the resulting function~$f$ may
be written as
\begin{equation}
f(x) =  \frac{1}{n} \sum_{i=1}^n f_i(x) ~~~~\text{with}~~~~ f_i(x) = \E_{\rho_i} \left[ \tilde{f}_i(x, \rho_i)\right], \label{eq:f2}
\end{equation}
with convex terms~$f_i(x)$ for each index~$i$.
Whereas~(\ref{eq:f2}) is a finite sum of functions, we now assume that one has now
only access to unbiased estimates of the gradients $\nabla f_i(x)$ due to the
stochastic nature of $f_i$. Then, all the aforementioned variance-reduction
methods do not apply anymore and the standard approach to address this problem
is to ignore the finite-sum structure and use SGD or
one of its variants. At each iteration, an estimate of the full gradient is
obtained by randomly drawing an index $\hati$ in $\{1,\ldots,n\}$ along with a
perturbation. Typically, the variance of the gradient estimate then
decomposes into two parts $\sigma^2=\sigma_s^2 + \tilde{\sigma}^2$, where
$\sigma_s^2$ is due to the random sampling of the index $\hati$ and
$\tilde{\sigma}^2$ is due to the random data perturbation. In such a context,
variance reduction consists of building gradient estimates with variance
$\tilde{\sigma}^2$, which is potentially much smaller than~$\sigma^2$.
The SAGA and SVRG methods were adapted for such a purpose
by~\citet{hofmann_variance_2015}, though the resulting algorithms have non-zero
asymptotic error; the MISO method was adapted by~\citet{smiso} at the cost of a
memory overhead of $O(np)$, whereas other variants of SAGA and SVRG were proposed
by~\citet{zheng2018lightweight} for linear models in machine learning.

The framework we adopt is that of estimate sequences introduced by~\citet{nesterov},
which consists of building iteratively a quadratic model of the objective.
Typically, estimate sequences may be used to analyze the convergence of existing
algorithms, but also to design new ones, in particular with acceleration.
Our construction is however slightly different than the original one since it
is based on stochastic estimates of the gradients, and some classical
properties of estimate sequences are satisfied only approximately.
We note that estimate sequences have been used before for stochastic optimization~\citep{Lu2015,devolder_2011,lin_pena2014}, but not for the same generic purpose as ours. 

Specifically, our paper makes to the following contributions:
\begin{itemize}
\item We revisit many stochastic optimization algorithms dealing with composite convex
problems; we consider variants of incremental methods such as SVRG,
SAGA, SDCA, or MISO. We provide a common convergence proof for these methods and show that they can be modified and become adaptive to the strong convexity constant $\mu$, when only a lower bound
is available.
\item We provide improvements to the previous algorithms by making them
robust to stochastic perturbations.
We analyze these approaches under a non-uniform sampling strategy $Q=\{q_1,\ldots,q_n\}$ where~$q_i$ is the probability of drawing example~$i$ at each iteration. Typically, when the $n$ gradients~$\nabla f_i$ have different Lipschitz constants~$L_i$, the uniform distribution $Q$ yields complexities that depend on $L_Q=\max_i L_i$, whereas a non-uniform $Q$ may yield $L_Q=\frac{1}{n}\sum_i L_i$.
For strongly convex problems, we propose approaches with the following worst-case iteration complexity for minimizing~(\ref{eq:f2})---that is, the number of iterations to guarantee $\E[F(x_k)-F^\star] \leq \varepsilon$---is upper bounded~by
\begin{equation*}
O\left( \left(n + \frac{L_Q}{\mu}\right)\log\left(\frac{F(x_0)-F^\star}{\varepsilon} \right)  \right) + O\left( \frac{ \rho_Q\tilde{\sigma}^2}{\mu\varepsilon} \right), 
\end{equation*}
      where $L_Q = \max_i L_i/(q_i n)$ and $\rho_Q = 1/(n \min q_i) \geq 1$ (note that $\rho_Q=1$ for uniform distributions).
The term on the left corresponds to the
complexity of the variance-reduction methods for a deterministic objective
without perturbation, and
$O(\tilde{\sigma}^2/\mu \varepsilon)$ is the optimal sublinear rate of convergence
for a stochastic optimization problem when the gradient estimates have
variance~$\tilde{\sigma}^2$. In contrast, a variant of stochastic gradient descent
for composite optimization applied to~(\ref{eq:f2}) has worst-case
complexity $O(\sigma^2/\mu \varepsilon)$, with potentially
$\sigma^2 \gg \tilde{\sigma}^2$.
Note that the non-uniform sampling strategy potentially reduces~$L_Q$ and improves the left part, whereas it increases $\rho_Q$ and degrades the dependency on the noise $\tilde{\sigma}^2$.
Whereas non-uniform sampling strategies for incremental methods are now classical~\citep{proxsvrg,saganonu},
the robustness to stochastic perturbations has not been studied for all these methods and
existing approaches such as \citep{hofmann_variance_2015,smiso,zheng2018lightweight} have various limitations as discussed earlier.
\item We show that our construction of estimate sequence naturally leads to
an accelerated stochastic gradient method for composite optimization as
\cite{ghadimi2012optimal,ghadimi2013optimal,kwok2009}, but 
simpler as our approach requires to maintain two sequences of iterates instead of three. The resulting complexity in terms of gradient evaluations for $\mu$-strongly convex objectives~is
\begin{equation*}
O\left( \sqrt{\frac{L}{\mu}} \log\left(\frac{F(x_0)-F^\star}{\varepsilon} \right)  \right) + O\left( \frac{{\sigma}^2}{\mu \varepsilon} \right), 
\end{equation*}
      which has also been achieved by~\citet{ghadimi2013optimal,aybat2019universally,cohen2018acceleration}.
When the objective is convex, but non-strongly convex, we
also provide a sublinear convergence rate for finite horizon. Given a budget of $K$ iterations, the algorithm returns an iterate $x_K$ such that
\begin{equation}
\E[ F(x_K)-F^\star] \leq \frac{2 L \|x_0-x^\star\|^2}{(K+1)^2} + \sigma \sqrt{\frac{8\|x_0-x^\star\|^2}{K+1}}, \label{eq:acc_sgd_rate_convex}
\end{equation}
which is also optimal for stochastic first-order optimization~\citep{ghadimi2012optimal}.
\item We design a new accelerated algorithm for finite sums based on the SVRG gradient estimator,
with complexity, for $\mu$-strongly convex functions, 
\begin{equation}
O\left( \left(n + \sqrt{n\frac{L_Q}{\mu}}\right)\log\left(\frac{F(x_0)-F^\star}{\varepsilon} \right)  \right) + O\left( \frac{\rho_Q \tilde{\sigma}^2}{\mu\varepsilon} \right), \label{eq:acc_compl}
\end{equation}
where the term on the left is the classical optimal complexity for
deterministic finite sums, which has been well studied when
$\tilde{\sigma}^2=0$~\cite{arjevani2016dimension,accsvrg,MiG2018,zhou2018direct,kovalev2020don}.
To the best of our knowledge, our algorithm is nevertheless the first to achieve such a
complexity when $\tilde{\sigma}^2 > 0$. 
Most related to our work, the general case $\tilde{\sigma}^2>0$ was indeed
considered recently by~\citet{lan_zhou_distributed2018} in the context of
distributed optimization, with
an approach that was shown to be optimal in terms of communication rounds. Yet, when applied in
the same context as ours (in a non-distributed setting), the complexity they achieve is suboptimal.
Specifically, their dependence in $\tilde{\sigma}^2$ involves an additional logarithmic factor~$\mathcal{O}(\log(1/\mu\epsilon))$ and the deterministic part is sublinear in~$\mathcal{O}(1/\epsilon)$.

      When the problem is convex but not strongly convex, given a budget of $K$ greater than $O(n\log(n))$,  the algorithm returns a solution $x_K$ such that
\begin{equation}
\E[ F(x_K)-F^\star] \leq \frac{18 n L_Q \|x_0-x^\star\|^2}{(K+1)^2} + 9\tilde{\sigma} \|x_0-x^\star\|\sqrt{\frac{\rho_Q}{K+1}}, \label{eq:acc_svrg_rate_convex}
\end{equation}
where the term on the right is potentially better
than~(\ref{eq:acc_sgd_rate_convex}) for large $K$ when $\tilde{\sigma} \ll \sigma$ (see discussion above on full variance vs. variance due to stochastic perturbations).
When the objective is deterministic ($\tilde{\sigma}=0$), the term~(\ref{eq:acc_svrg_rate_convex}) yields the complexity $O(\sqrt{n L_Q}/\sqrt{\varepsilon})$, which is potentially
better than the $O(n\sqrt{L}/\sqrt{\varepsilon})$  complexity of accelerated
gradient descent, unless $L$ is significantly smaller than $L_Q$.
\end{itemize}
This paper is organized as follows. Section~\ref{sec:lower} introduces the proposed framework
based on stochastic estimate sequences; Section~\ref{sec:theory} is devoted to the convergence
analysis and Section~\ref{sec:acc} introduces accelerated stochastic
optimization algorithms; Section~\ref{sec:exp} presents various experiments to
compare the effectiveness of the proposed approaches, and Section~\ref{sec:ccl}
concludes the paper.

We note that a short version of this paper was presented
by~\cite{kulunchakov2019estimate} at the International Conference on Machine
Learning (ICML) in 2019. This paper extends this previous work by (i)
providing complexity results for convex but not strongly convex objectives
($\mu=0$),  (ii) extending the framework to variants of MISO/Finito/SDCA
algorithms, in the context of non-accelerated incremental methods, (iii)
providing more experiments with additional baselines and objective functions.

\section{Proposed Framework Based on Stochastic Estimate Sequences}\label{sec:lower}
In this section, we present two generic stochastic optimization algorithms to
address the composite problem~(\ref{eq:prob}). Then, we show their relation
to variance-reduction methods.

\subsection{A Classical Iteration Revisited}\label{subsec:dk}
Consider an algorithm that performs the following updates:
\custombox{
\begin{equation}
x_{k} \leftarrow \text{Prox}_{\eta_k\psi}\left[ x_\kmone - \eta_k g_k\right]~~~\text{with}~~~ \E[g_k | {\mathcal F}_\kmone] = \nabla f(x_\kmone), \tag{\textsf{A}} \label{eq:opt1}
\vspace*{0.1cm}
\end{equation}
}%
where $\Fcal_{\kmone}$ is the filtration representing all information up to iteration $\kmone$, $g_k$ is an unbiased estimate of the gradient $\nabla f(x_\kmone)$, $\eta_k > 0$ is a step size, and $\text{Prox}_{\eta \psi}[.]$ is the proximal operator~\citep{moreau1962fonctions} defined for any scalar $\eta > 0$ as the unique solution of
\begin{equation}
\text{Prox}_{\eta \psi}[u] \defin \argmin_{x \in \Real^p} \left\{ \eta \psi (x) + \frac{1}{2}\|x-u\|^2 \right\}. \label{eq:prox}
\end{equation}
The iteration~(\ref{eq:opt1}) is generic and encompasses many existing
algorithms, which we review later. Key to our analysis, we are interested in a simple interpretation corresponding to the iterative minimization of strongly convex surrogate functions.

\paragraph{Interpretation with stochastic estimate sequence.}
Consider now the function
\begin{equation}
d_0(x) =  d_0^\star + \frac{\gamma_0}{2}\|x-x_0\|^2, \label{eq:d0}
\end{equation}
with $\gamma_0 \geq \mu$ and $d_0^\star$ is a scalar value that is left unspecified at the moment. Then, it is easy to show that $x_k$ in~(\ref{eq:opt1}) minimizes the following quadratic function $d_k$ defined for $k \geq 1$  as
\begin{multline}
d_k(x) = (1-\delta_k)d_\kmone(x) \\ + \delta_k \left( f(x_\kmone) + g_k^\top( x-x_\kmone) + \frac{\mu}{2}\|x-x_\kmone\|^2 + \psi(x_k) + \psi'(x_k)^\top(x-x_{k}) \right),\label{eq:surrogate1}
\end{multline}
where $\delta_k,\gamma_k$ satisfy the system of equations
\begin{equation}
\delta_k = \eta_k\gamma_k ~~~\text{and}~~~
\gamma_k = (1-\delta_k)\gamma_\kmone +\mu \delta_k,
\label{eq:gamma}
\end{equation}
and
\begin{displaymath}
\psi'(x_k)  = \frac{1}{\eta_k} (x_{\kmone} - x_k) - g_k.
\end{displaymath}
We note that
$\psi'(x_k)$ is a subgradient in $\partial \psi(x_k)$. By simply using the definition of the
proximal operator~(\ref{eq:prox}) and considering first-order optimality conditions, we indeed have that $0
\in x_k - x_{\kmone} + \eta_k g_k + \eta_k \partial
\psi(x_k)$ and~$x_k$ coincides with the minimizer of~$d_k$. This allows us to write $d_k$ in the generic form
\begin{displaymath}
d_k(x) = d_k^\star + \frac{\gamma_k}{2}\|x-x_k\|^2~~~\text{for all}~k \geq 0.
\end{displaymath}
The construction~(\ref{eq:surrogate1}) is akin to that of estimate sequences
introduced by~\citet{nesterov}, which are typically used for
designing accelerated gradient-based optimization algorithms. In this section,
we are however not interested in acceleration, but instead in stochastic optimization
and variance reduction. One of the main property of estimate sequences that we will
use is their ability do behave asymptotically as a lower bound of the objective
function near the optimum. Indeed, we have
\begin{equation}
\begin{split}
\E[d_k(x^\star)]  \leq (1-\delta_k)\E[d_\kmone(x^\star)] + \delta_k F^\star
& \leq \Gamma_k d_0(x^\star) + (1-\Gamma_k)F^\star,
\end{split}\label{eq:est}
\end{equation}
where $\Gamma_k = \prod_{t=1}^k (1-\delta_t)$ and $F^\star=F(x^\star)$.
The first inequality comes from  a strong convexity inequality since $\E[g_k^\top (x^\star-x_\kmone)|\Fcal_\kmone] = \nabla f(x_\kmone)^\top (x^\star-x_\kmone)$, and the second inequality is obtained by unrolling the relation obtained between $\E[d_k(x^\star)]$ and $\E[d_\kmone(x^\star)]$. When~$\Gamma_k$ converges to zero, the contribution of the initial surrogate $d_0$ disappears and $\E[d_k(x^\star)]$ behaves as a lower bound of $F^\star$.

\paragraph{Relation with existing algorithms.} The iteration~(\ref{eq:opt1})
encompasses many approaches such as ISTA (proximal gradient descent), which
uses the exact gradient $g_k = \nabla f(x_{\kmone})$ leading to deterministic
iterates~$(x_k)_{k \geq 0}$ \citep{fista,nesterov2013gradient} or
proximal variants of the stochastic gradient descent method to deal with a
composite objective~\citep[see][for instance]{Lan2012}.
Of interest for us, the variance-reduced stochastic optimization approaches
SVRG~\citep{proxsvrg} and SAGA~\cite{saga} also follow the
iteration~(\ref{eq:opt1}) but with an unbiased gradient estimator whose
variance reduces over time.  Specifically, the basic form of these estimators is
\begin{equation}
g_k = \nabla f_{i_k}(x_{\kmone}) - z^{i_k}_{\kmone} + \bar{z}_{\kmone} ~~~\text{with}~~~ \bar{z}_\kmone = \frac{1}{n} \sum_{i=1}^n z^{i}_\kmone, \label{eq:svrg}
\end{equation}
where $i_k$ is an index chosen uniformly in $\{1,\ldots,n\}$ at random, and each
auxiliary variable~$z^{i}_k$ is equal to the gradient $\nabla
f_i(\tilde{x}_k^i)$, where~$\tilde{x}_k^i$ is one of the previous iterates.
The motivation is that given two random variables $X$ and $Y$,
it is possible to define a new variable $Z=X-Y + \E[Y]$ which has the same
expectation as~$X$ but potentially a lower variance if $Y$ is positively
correlated with $X$.  SVRG and SAGA are two different approaches to build such
positively correlated variables.
SVRG uses the same anchor point $\tilde{x}_k^i = \tilde{x}_k$ for all $i$, where~$\tilde{x}_k$ is
updated every $m$ iterations. Typically, the memory cost of SVRG is that of
storing the variable $\tilde{x}_k$ and the gradient $\bar{z}_k = \nabla f(\tilde{x}_k)$, which is thus $O(p)$.
On the other hand, SAGA updates only $z^{i_k}_k = \nabla f_{i_k}({x}_{\kmone})$ at iteration~$k$,
such that $z^{i}_{k}  =  z^{i}_{\kmone}$ if $i \neq i_k$.
Thus, SAGA requires storing $n$ gradients. While in general the
overhead cost in memory is of order $O(np)$, it may be reduced to $O(n)$ when
dealing with linear models in machine learning~\citep[see][]{saga}.
Note that variants with non-uniform sampling of the indices $i_k$ have been proposed~by \citet{proxsvrg,saganonu}.

In order to make our proofs consistent for all considered incremental methods, we analyze a variant of SVRG with a randomized gradient updating schedule \cite{hofmann_variance_2015}. Remarkably, this variant was recently used in a concurrent work~\cite{kovalev2020don} to get the accelerated rate when $\tilde{\sigma}^2=0$. 

\subsection{A Less Classical Iteration with a Different Estimate Sequence}\label{subsec:dk2}
In the previous section, we have interpreted the classical iteration~(\ref{eq:opt1}) as the iterative minimization of the stochastic surrogate~(\ref{eq:surrogate1}).
Here, we show that a slightly different construction leads to a new algorithm.
To obtain a lower bound, we have indeed used basic properties of the proximal operator to obtain a subgradient $\psi'(x_k)$ and we have exploited the following convexity inequality
$\psi(x) \geq \psi(x_k) + \psi'(x_k)^\top (x-x_k)$.
Another natural choice to build a lower bound consists then of using directly $\psi(x)$ instead of $\psi(x_k) + \psi'(x_k)^\top (x-x_k)$, leading to the construction
\begin{equation}
d_k(x) = (1-\delta_k)d_\kmone(x) + \delta_k \left( f(x_\kmone) + g_k^\top( x-x_\kmone) + \frac{\mu}{2}\|x-x_\kmone\|^2 + \psi(x) \right), \label{eq:surrogate2}
\end{equation}
where $x_\kmone$ is assumed to be the minimizer of the composite function
$d_\kmone$, $\delta_k$ is defined as in Section~\ref{subsec:dk}, and $x_k$ is a minimizer of $d_k$. To initialize the recursion, we define then~$d_0$ as
\begin{equation*}
d_0(x) = c_0 + \frac{\gamma_0}{2}\|x-\bar{x}_0\|^2 + \psi(x)  \geq d_0^\star + \frac{\gamma_0}{2}\|x-x_0\|^2, 
\end{equation*}
with $x_0 = \text{Prox}_{\psi/\gamma_0}[\bar{x}_0]$ is the minimizer of $d_0$ and $d_0^\star=d_0(x_0)=c_0+\frac{\gamma_0}{2}\|x_0-\bar{x}_0\|^2 + \psi(x_0)$ is the minimum value of $d_0$; $c_0$ is left unspecified since it does not affect the algorithm.
Typically, one may choose $\bar{x}_0$ to be a minimizer of $\psi$ such that $x_0=\bar{x}_0$.
Unlike in the previous section, the surrogates~$d_k$ are not quadratic, but
they remain $\gamma_k$-strongly convex.  It is also easy to check that the
relation~(\ref{eq:est}) still holds.

\paragraph{The corresponding algorithm.}
It is also relatively easy to show that the iterative minimization of the stochastic lower bounds~(\ref{eq:surrogate2}) leads to the following iterations
\custombox{
\vspace*{-0.25cm}
\begin{equation}
\bar{x}_{k} \leftarrow (1-\mu \eta_k)\bar{x}_\kmone + \mu \eta_k x_{\kmone} - \eta_k g_k~~~\text{and}~~~x_k =\text{Prox}_{\frac{\psi}{\gamma_k}}\left[ \bar{x}_k \right] ~~~\text{with}~~~ \E[g_k | {\mathcal F}_\kmone] = \nabla f(x_\kmone). \tag{\textsf{B}} \label{eq:opt2}
\end{equation}
}
As we will see, the convergence analysis for algorithm~(\ref{eq:opt1}) also
holds for algorithm~(\ref{eq:opt2}) such that both variants enjoy similar
theoretical properties.
In one case, the function~$\psi(x)$
appears explicitly, whereas a lower bound $\psi(x_k) + \psi'(x_k)^\top(x-x_k)$
is used in the other case. The introduction of the variable $\bar{x}_k$ allows us to write the surrogates $d_k$ in the canonical form
\begin{displaymath}
d_k(x) = c_k + \frac{\gamma_k}{2}\|x-\bar{x}_k\|^2 + \psi(x) \geq d_k^\star + \frac{\gamma_k}{2}\|x-x_k\|^2,
\end{displaymath}
where $c_k$ is constant and the inequality on the right is due to the strong convexity of $d_k$.

\paragraph{Relation to existing approaches.}
The approach~(\ref{eq:opt2}) is related to several optimization methods. When
the objective is a deterministic finite sum, it is possible to
relate the update~(\ref{eq:opt2}) to 
the MISO~\cite{miso},
and Finito~\cite{defazio2014finito} algorithms, even though they were derived
from a significantly different point of view. 
This is also the case of a primal variant of SDCA~\citep{ShalevShwartz2016SDCAWD} 
For instance, SDCA is a dual coordinate ascent approach, whereas MISO and
Finito are explicitly derived from the iterative surrogate minimization we
adopt in this paper.  As the links between~(\ref{eq:opt2}) and these previous
approaches are not obvious at first sight, we detail them in
Appendix~\ref{appendix:miso.sdca}. 

\subsection{Gradient Estimators and Algorithms} \label{subsec:gradients}
In this paper, we consider the iterations~(\ref{eq:opt1}) and~(\ref{eq:opt2}) with the following gradient estimators. 
\begin{itemize}
\item {\bfseries exact gradient} with $g_k = \nabla f(x_\kmone)$, when the problem is deterministic and we have an access to the full gradient;
\item {\bfseries stochastic gradient}, when we just assume that $g_k$ has bounded variance. When $f(x) = \E_\xi[\tilde{f}(x,\xi)]$, a data point $\xi_k$ is drawn at iteration $k$ and $g_k = \nabla \tilde{f}(x,\xi_k)$.
\item {\bfseries random-SVRG}: for finite sums, we consider a variant of the SVRG gradient estimator with non-uniform sampling and a random update of the anchor point~$\tilde{x}_\kmone$, 
proposed originally by~\citet{hofmann_variance_2015}.
Specifically, $g_k$ is also an unbiased estimator of $\nabla f(x_\kmone)$, defined as
\begin{equation}
g_k = \frac{1}{q_{i_k} n}\left(\tildenabla f_{i_k} (x_\kmone) - z_{\kmone}^{i_k} \right) + \bar{z}_\kmone, \label{eq:gk}
\end{equation}
where $i_k$ is sampled from a distribution $Q=\{q_1,\ldots,q_n\}$ and $\tildenabla$ denotes that the gradient is perturbed by a zero-mean noise variable with variance $\tilde{\sigma}^2$.
More precisely, if $f_i(x)=\E_\rho[\tilde{f}_i(x,\rho)]$ for all $i$, where~$\rho$ is a stochastic perturbation,
instead of accessing $\nabla f_{i_k}(x_\kmone)$, we draw a perturbation $\rho_k$ and observe
\begin{displaymath}
\tildenabla f_{i_k}(x_\kmone) = \nabla \tilde{f}_{i_k}(x_\kmone,\rho_k)  = \nabla f_{i_k} (x_\kmone) + \underbrace{\nabla \tilde{f}_{i_k}(x_\kmone,\rho_k) - \nabla f_{i_k} (x_\kmone)}_{\zeta_k},
\end{displaymath}
where the perturbation $\zeta_k$ has zero mean given $\Fcal_\kmone$ and its variance is bounded by~$\tilde{\sigma}^2$. When there is no perturbation, we simply have $\tildenabla = \nabla$ and $\zeta_k=0$.

Then, the variables $z_k^{i}$ and $\bar{z}_k$ also involve noisy estimates of the gradients:
\begin{displaymath}
z_{k}^{i}  = \tildenabla f_{i}(\tilde{x}_k)  ~~~~\text{and}~~~ \bar{z}_k = \frac{1}{n}\sum_{i=1}^n z_k^i,
\end{displaymath}
where $\tilde{x}_k$ is an anchor point that is updated on average every $n$ iterations.
Whereas the classical SVRG approach~\cite{proxsvrg}
updates~$\tilde{x}_k$ on a fixed schedule, we perform random
updates: with probability $1/n$, we choose $\tilde{x}_k = x_k$ and
recompute~$\bar{z}_k=\tildenabla f(\tilde{x}_k)$; otherwise $\tilde{x}_k$ is
kept unchanged. In comparison with the fixed schedule, the analysis with the random one is
simplified and can be unified with that of SAGA/SDCA or MISO.  The use
of this estimator with iteration~(\ref{eq:opt1}) is illustrated in
Algorithm~\ref{alg:svrg}. It is then easy to modify it to use
variant~(\ref{eq:opt2}) instead.

In terms of memory, the random-SVRG gradient estimator requires to
store an anchor point~$\tilde{x}_\kmone$ and the average gradients
$\bar{z}_\kmone$. The variables $z_k^i$ do not need to be stored; only the~$n$
random seeds to produce the perturbations are kept into memory,
which allows us to compute
$z^{i_k}_\kmone = \tildenabla f_{i_k}(\tilde{x}_\kmone)$ at
iteration~$k$, with the same perturbation for index~$i_k$ that was used
to compute $\bar{z}_\kmone = \frac{1}{n}\sum_{i=1}^n z_\kmone^i$ when the anchor point was last updated.
The overall cost is thus $O(n+p)$.

\begin{algorithm}[hbtp]
\caption{Variant~(\ref{eq:opt1}) with random-SVRG estimator}\label{alg:svrg}
\begin{algorithmic}[1]
\State {\bfseries Input:} $x_0$ in $\Real^p$; $K$ (number of iterations); $(\eta_k)_{k \geq 0}$ (step sizes); $\gamma_0 \geq \mu$ (if averaging);
\State {\bfseries Initialization:} $\tilde{x}_0 = \hat{x}_0 = x_0$; $\bar{z}_0=\frac{1}{n}\sum_{i=1}^n \tildenabla f_i(\tilde{x}_0)$;
\For{$k=1,\ldots,K$}
\State Sample $i_k$ according to the distribution $Q=\{q_1,\ldots,q_n\}$;
\State Compute the gradient estimator, possibly corrupted by random perturbations:
\begin{equation*}
g_k = \frac{1}{q_{i_k} n}\left(\tildenabla f_{i_k} (x_\kmone) - \tildenabla f_{i_k} (\tilde{x}_\kmone) \right) + \bar{z}_\kmone; 
\end{equation*}
\State Obtain the new iterate
$x_{k} \leftarrow \text{Prox}_{\eta_k\psi}\left[ x_\kmone - \eta_k g_k\right];$
\State With probability $1/n$,
\begin{displaymath}
\tilde{x}_k = x_k~~~~\text{and}~~~~ \bar{z}_k =\frac{1}{n} \sum_{i=1}^n\tildenabla f_i(\tilde{x}_k);
\end{displaymath}
\State Otherwise, with probability $1-1/n$, keep $\tilde{x}_k = \tilde{x}_\kmone$ and $\bar{z}_k = \bar{z}_\kmone$;
\State {\bfseries Optional}: Use the online averaging strategy using $\delta_k$ obtained from~(\ref{eq:gamma}):
$$\hat{x}_k =
(1-\tau_k) \hat{x}_{\kmone} + \tau_k x_k~~~~\text{with}~~~~ \tau_k = \min\left( \delta_k, \frac{1}{5n}\right);$$
\EndFor
\State {\bfseries Output:} $x_K$ or $\hat{x}_K$ if averaging.
\end{algorithmic}
\end{algorithm}

\item {\bfseries SAGA}: The estimator has a form similar to~(\ref{eq:gk}) but with a different choice of variables
$z_k^i$.  Unlike SVRG that stores an anchor
point~$\tilde{x}_k$, the SAGA estimator requires storing and
incrementally updating the $n$ auxiliary variables $z_k^i$ for
$i=1,\ldots,n$, while maintaining the relation $\bar{z}_k =
\frac{1}{n}\sum_{i=1}^n z_k^i$.
We consider variants such that each time a gradient~$\nabla f_i(x)$ is computed,
it is corrupted by a zero-mean random perturbation with variance $\tilde{\sigma}^2$.
The procedure is described in Algorithm~\ref{alg:miso} for
variant~(\ref{eq:opt1}) when using uniform sampling. 
When $\beta=0$, we recover the original SAGA algorithm,
whereas the choice $\beta > 0$ corresponds to a more general estimator that we
will discuss next.

The case with non-uniform sampling is slightly different and is described in Algorithm~\ref{alg:miso2};
it requires an additional index~$j_k$ for updating a variable $z_k^{j_k}$.
The reason for that is to  remove a difficulty in the convergence
proof, a strategy also adopted by~\citet{saganonu} for a variant of SAGA
with non-uniform sampling.
\item {\bfseries SDCA/MISO}:
   To put SAGA, MISO and SDCA under the same umbrella, we introduce a lower bound $\beta$ on the strong convexity constant $\mu$,
and a correcting term involving $\beta$ that appears only when the sampling distribution $Q$ is not uniform:
\begin{equation}
   g_k = \frac{1}{q_{i_k} n}\left(\tildenabla f_{i_k} (x_\kmone) - z_{\kmone}^{i_k} \right) + \bar{z}_\kmone + \beta \left(1 - \frac{1}{q_{i_k} n}\right) x_\kmone. \label{eq:gk2}
\end{equation}
It is then possible to show that when
$Q$ is uniform and under the big data condition $L/\mu \leq n$ (used for
instanced by~\citealt{miso,defazio2014finito,schmidt2017minimizing}) and with $\beta=\mu$,
variant~(\ref{eq:opt2}) combined with the estimator~(\ref{eq:gk2}) yields the MISO algorithm,
which performs similar updates as a primal variant of SDCA~\citep{ShalevShwartz2016SDCAWD}.
These links are highlighted in Appendix~\ref{appendix:miso.sdca}.

The motivation for introducing the parameter~$\beta$ in $[0,\mu]$ comes
from empirical risk minimization problems, where the functions $f_i$ may have the form $f_i(x) =
\phi(a_i^\top x) + \frac{\beta}{2}\|x\|^2$, where~$a_i$ in~$\Real^p$ is
a data point; then, $\beta$ is a lower bound on the strong
convexity modulus~$\mu$, and $\nabla f_i(x)-\beta x$ is proportional to
$a_i$ and can be stored with a single additional scalar value, assuming $a_i$ is already in memory.
\end{itemize}
\paragraph{Summary of the new features.}
As we combine different types of iterations and gradient estimators, we 
recover both known and new algorithms. Specifically, we obtain the following new features:
\begin{itemize}
\item {\bfseries robustness to noise}: we introduce mechanisms to deal with stochastic perturbations and make all these previous approaches robust to noise.
\item {\bfseries adaptivity to the strong convexity when $\tilde{\sigma}=0$}: Algorithms~\ref{alg:svrg}, \ref{alg:miso}, and~\ref{alg:miso2} without averaging do not require knowing the strong convexity constant $\mu$ (it may only need a lower-bound~$\beta$, which is often trivial to obtain). 
\item {\bfseries new variants:} Whereas SVRG/SAGA were developed with the iterations~(\ref{eq:opt1}) and MISO in the context of~(\ref{eq:opt2}), we show that these gradient estimators are both compatible with~(\ref{eq:opt1}) and~(\ref{eq:opt2}), leading to new algorithms with similar guarantees.
\end{itemize}

\begin{algorithm}[htp!]
\caption{Variant~(\ref{eq:opt1}) with SAGA/SDCA/MISO estimator and uniform sampling}\label{alg:miso}
\begin{algorithmic}[1]
   \State {\bfseries Input:} $x_0$ in $\Real^p$; $K$ (num. iterations); $(\eta_k)_{k \geq 0}$ (step sizes); $\beta$ in $[0,\mu]$;  $\gamma_0 \geq \mu$ (optional).
\State {\bfseries Initialization:} $z_0^i = \tildenabla f_i(x_0) - \beta x_0$ for all $i=1,\ldots,n$ and $\bar{z}_0 = \frac{1}{n}\sum_{i=1}^n {z_0^i}$.
\For{$k=1,\ldots,K$}
\State Sample $i_k$ in $\{1,\ldots,n\}$ according to the uniform distribution;
\State Compute the gradient estimator, possibly corrupted by random perturbations:
\begin{displaymath}
g_k = \tildenabla f_{i_k} (x_\kmone) - z^{i_k}_\kmone + \bar{z}_\kmone;
\end{displaymath}
\State Obtain the new iterate
$x_{k} \leftarrow \text{Prox}_{\eta_k\psi}\left[ x_\kmone - \eta_k g_k\right];$
\State Update the auxiliary variables
\begin{displaymath}
z^{i_k}_k  = \tildenabla f_{i_k} (x_\kmone) - \beta x_\kmone ~~~\text{and}~~~ z^i_k  = z^i_\kmone~~~\text{for all}~~~~i \neq i_k;
\end{displaymath}
\State Update the average variable $\bar{z}_k = \bar{z}_\kmone + \frac{1}{n}(z^{j_k}_k - z^{j_k}_\kmone)$.
\State {\bfseries Optional}: Use the same averaging strategy as in Algorithm~\ref{alg:svrg}.
\EndFor
\State {\bfseries Output:} $x_K$ or $\hat{x}_K$ (if averaging).
\end{algorithmic}
\end{algorithm}

\begin{algorithm}[h!]
\caption{Variant~(\ref{eq:opt1}) with SAGA/SDCA/MISO estimator and non-uniform sampling}\label{alg:miso2}
\begin{algorithmic}[1]
   \State {\bfseries Input:} $x_0$ in $\Real^p$; $K$ (num. iterations); $(\eta_k)_{k \geq 0}$ (step sizes); $\beta$ in $[0,\mu]$;  $\gamma_0 \geq \mu$ (optional).
\State {\bfseries Initialization:} $z_0^i = \tildenabla f_i(x_0) - \beta x_0$ for all $i=1,\ldots,n$ and $\bar{z}_0 = \frac{1}{n}\sum_{i=1}^n {z_0^i}$.
\For{$k=1,\ldots,K$}
\State Sample $i_k$ according to the distribution $Q=\{q_1,\ldots,q_n\}$;
\State Compute the gradient estimator, possibly corrupted by random perturbations:
\begin{displaymath}
   g_k = \frac{1}{q_{i_k} n}\left(\tildenabla f_{i_k} (x_\kmone) - z^{i_k}_\kmone \right) + \bar{z}_\kmone + \beta\left(1 -  \frac{1}{q_{i_k} n}\right) x_\kmone;
\end{displaymath}
\State Obtain the new iterate
$x_{k} \leftarrow \text{Prox}_{\eta_k\psi}\left[ x_\kmone - \eta_k g_k\right];$
\State Draw $j_k$ from the uniform distribution in $\{1,\ldots,n\}$;
\State Update the auxiliary variables
\begin{displaymath}
z^{j_k}_k  = \tildenabla f_{j_k} (x_k) - \beta x_k ~~~\text{and}~~~ z^j_k  = z^j_\kmone~~~\text{for all}~~~~j \neq j_k;
\end{displaymath}
\State Update the average variable $\bar{z}_k = \bar{z}_\kmone + \frac{1}{n}(z^{j_k}_k - z^{j_k}_\kmone)$.
\State {\bfseries Optional}: Use the same averaging strategy as in Algorithm~\ref{alg:svrg}.
\EndFor
\State {\bfseries Output:} $x_K$ or $\hat{x}_K$ (if averaging).
\end{algorithmic}
\end{algorithm}

\section{Convergence Analysis and Robustness}\label{sec:theory}
We now present the convergence analysis for iterations~(\ref{eq:opt1}) or~(\ref{eq:opt2}). 
In Section~\ref{subsec:conv1}, we present a generic convergence result. Then, in Section~\ref{subsec:conv2},
we present specific results for the variance-reduction approaches in
including strategies to make them robust to stochastic noise.  Acceleration
is discussed in the next section.

\subsection{Generic Convergence Result Without Variance Reduction}\label{subsec:conv1}
Key to our complexity results, the following proposition gives a first relation between the quantity $F(x_k)$, the surrogate $d_k$, $d_\kmone$ and the variance of the gradient estimates.
\begin{proposition}[\bf Key relation]\label{prop:keyprop}
For either variant~(\ref{eq:opt1}) or~(\ref{eq:opt2}), when using the construction of $d_k$ from Sections~\ref{subsec:dk} or~\ref{subsec:dk2}, respectively, and assuming $\eta_k \leq 1/L$, we have for all $k \geq 1$,
\begin{equation}
\delta_k (\E[F(x_k)]-F^\star) +\E[d_k(x^\star)-d_k^\star]  \leq  (1-\delta_k)\E[d_\kmone(x^\star)-d_{\kmone}^\star] + \eta_k {\delta_k \omega_k^2}, \label{eq:relation}
\end{equation}
where $F^\star$ is the minimum of $F$, $x^\star$ is one of its minimizers, and
$\omega_k^2 = \E[ \|g_k - \nabla f(x_\kmone)\|^2]$.
\end{proposition}
\begin{proof}
We first consider the variant~(\ref{eq:opt1}) and later show how to modify the convergence proofs to accommodate the variant~(\ref{eq:opt2}).
\begin{displaymath}
\begin{split}
d_k^\star = d_k(x_k) & = (1-\delta_k) d_{\kmone}(x_k) + \delta_k \left( f(x_\kmone) + g_k^\top( x_k-x_\kmone) + \frac{\mu}{2}\|x_k-x_\kmone\|^2 + \psi(x_k) \right) \\
& \geq (1-\delta_k) d_{\kmone}^\star + \frac{\gamma_k}{2}\|x_k-x_{\kmone}\|^2 + \delta_k \left( f(x_\kmone)  + g_k^\top( x_k-x_\kmone) + \psi(x_k) \right)  \\
& \geq (1-\delta_k) d_{\kmone}^\star + \delta_k \left( f(x_\kmone)  + g_k^\top( x_k-x_\kmone) + \frac{L}{2}\|x_k-x_\kmone\|^2 + \psi(x_k) \right)  \\
& \geq (1-\delta_k) d_{\kmone}^\star + \delta_k F(x_k) +  \delta_k (g_k - \nabla f(x_\kmone))^\top (x_k-x_\kmone), \\
\end{split}
\end{displaymath}
where the first inequality comes from Lemma~\ref{lemma:second}---it is in fact an equality when considering Algorithm~(\ref{eq:opt1})---and the second inequality simply uses the assumption $\eta_k \leq 1/L$, which yields $\delta_k = \gamma_k\eta_k \leq \gamma_k/L$. Finally, the last inequality uses a classical upper-bound for $L$-smooth functions presented in Lemma~\ref{lemma:upper}.
Then, after taking expectations,
\begin{equation*}
\begin{split}
E[d_k^\star]      & \geq (1-\delta_k) \E[d_{\kmone}^\star] + \delta_k \E[F(x_k)] + \delta_k \E[(g_k-\nabla f(x_\kmone))^\top(x_k-x_\kmone)] \\
& = (1-\delta_k) \E[d_{\kmone}^\star] + \delta_k \E[F(x_k)] + \delta_k \E[(g_k-\nabla f(x_\kmone))^\top x_k]  \\
& = (1-\delta_k) \E[d_{\kmone}^\star] + \delta_k \E[F(x_k)] + \delta_k \E\left[(g_k-\nabla f(x_\kmone))^\top\left( x_k - w_\kmone \right) \right], \\
\end{split} 
\end{equation*}
where we have defined the following quantity
\begin{displaymath}
w_{\kmone} = \text{Prox}_{\eta_k \psi}\left[ x_\kmone - \eta_k\nabla f(x_{\kmone})\right].
\end{displaymath}
In the previous relations, we have used twice the fact that
$\E[(g_k-\nabla f(x_\kmone))^\top y | {\Fcal_{\kmone}}]=0$, for all deterministic variable~$y$ given $x_{\kmone}$, such as $y=x_{\kmone}$ or $y=w_{\kmone}$.
We may now use the non-expansiveness property of the proximal operator~\citep{moreau1965} to control the quantity $\|x_k-w_\kmone\|$, which gives us
\begin{equation*}
\begin{split}
\E[d_k^\star]
& \geq (1-\delta_k) \E[d_{\kmone}^\star] + \delta_k \E[F(x_k)] - \delta_k \E\left[\|g_k-\nabla f(x_\kmone)\|\|x_k - w_\kmone\| \right] \\
& \geq (1-\delta_k) \E[d_{\kmone}^\star] + \delta_k \E[F(x_k)] - \delta_k \eta_k \E\left[\|g_k-\nabla f(x_\kmone)\|^2\right] \\
& = (1-\delta_k) \E[d_{\kmone}^\star] + \delta_k \E[F(x_k)] - \delta_k \eta_k\omega_k^2. \\
\end{split}
\end{equation*}
This relation can now be combined with~(\ref{eq:est}) when $z=x^\star$,
and we obtain~(\ref{eq:relation}).
It is also easy to see that the proof also works with variant~(\ref{eq:opt2}).
The convergence analysis is identical, except that we take $w_\kmone$ to be
\begin{displaymath}
w_\kmone = \text{Prox}_{\frac{\psi}{\gamma_k}}\left[ (1-\mu \eta_k)\bar{x}_\kmone + \mu \eta_k x_{\kmone} - \eta_k\nabla f(x_{\kmone})\right],
\end{displaymath}
and the same result follows.
\end{proof}
Then, without making further assumption on~$\omega_k$, we have the following general convergence result, which is a direct consequence of the averaging Lemma~\ref{lemma:averaging}, inspired by~\citet{ghadimi2012optimal}, and presented in Appendix~\ref{appendix:averaging}:
\begin{theorem}[\bf General convergence result]\label{thm:conv}
Under the same assumptions as in Proposition~\ref{prop:keyprop}, we have for all $k \geq 1$, and either variant~(\ref{eq:opt1}) or~(\ref{eq:opt2}),
\begin{equation}
\E[\delta_k\left(F(x_k)-F^\star\right)  + d_k(x^\star)-d_k^\star] \leq {\Gamma_k}\left( d_0(x^\star)-d_{0}^\star +  \sum_{t=1}^k \frac{\delta_t\eta_t\omega_t^2}{\Gamma_t}\right), \label{eq:conv1}
\end{equation}
where $\Gamma_k = \prod_{t=1}^k (1-\delta_t)$.
   Then, by using the averaging strategy $\hat{x}_k = (1-\delta_k) \hat{x}_{\kmone} + \delta_k x_k$ of Lemma~\ref{lemma:averaging}, for any point $\hat{x}_0$ (possibly equal to $x_0$), we have
\begin{equation}
   \E\left[F(\hat{x}_k)- F^\star + d_k(x^\star)-d_k^\star\right] \leq  {\Gamma_k}\left(F(\hat{x}_0)- F^\star + d_0(x^\star)-d_{0}^\star  +  \sum_{t=1}^k \frac{\delta_t\eta_t\omega_t^2}{\Gamma_t} \right). \label{eq:conv2}
\end{equation}
\end{theorem}
Theorem~\ref{thm:conv} allows us to recover convergence rates for various algorithms.
Note that the effect of the averaging strategy is to remove the factor $\delta_k$ in front of $F(x_k)-F^\star$ on the left part of~(\ref{eq:conv1}), thus improving the convergence rate by a factor $1/\delta_k$. 
Regarding the quantity $d_0(x^\star)-d_0^\star$, we have the following relations
\begin{itemize}
  \item For variant~(\ref{eq:opt1}), $d_0(x^\star)-d_0^\star = \frac{\gamma_0}{2}\|x^\star-x_0\|^2$;
  \item For variant~(\ref{eq:opt2}), this quantity may be larger and we may simply say that $d_0(x^\star)-d_0^\star=\frac{\gamma_0}{2}\|x^\star-x_0\|^2 + \psi(x^\star) - \psi(x_0) - \psi'(x_0)^\top(x_0-x^\star)$ for variant~(\ref{eq:opt2}), where $\psi'(x_0)=\gamma_0(x_0-\bar{x}_0)$ is a subgradient in $\partial \psi(x_0)$.
     Note that if $\bar{x}_0$ is chosen to be a minimizer of $\psi$, then $d_0(x^\star)-d^\star =\frac{\gamma_0}{2}\|x^\star-x_0\|^2 + \psi(x^\star) - \psi(x_0)$.  
\end{itemize}

In the next section, we will focus on variance reduction mechanisms, which are
able to improve the previous convergence rates by better exploiting the
structure of the objective. By controlling the variance $\omega_k$ of the
corresponding gradient estimators, we will apply Theorem~\ref{thm:conv} to
obtain convergence rates. Before that, we remark that it is relatively
straightforward to use this theorem to recover complexity results for proximal
SGD, both for the usual variant~(\ref{eq:opt1}) or the new
one~(\ref{eq:opt2}). Since these results are classical, we present them in
Appendix~\ref{app.rec}. As a sanity check, we note that we recover the optimal
noise-dependency~\citep[see][]{nemirovski}, both for strongly convex cases, or when $\mu=0$.

\subsection{Faster Convergence with Variance Reduction}\label{subsec:conv2}

Stochastic variance-reduced gradient descent algorithms rely on gradient estimates whose variance
decreases as fast as the objective function value. Here, we provide a unified proof of
convergence for our variants of SVRG, SAGA, and MISO, and we
show how to make them robust to stochastic perturbations.
Specifically, we consider the minimization
of a finite sum of functions as in~(\ref{eq:f2}), but, as explained in Section~\ref{sec:lower}, each observation of the
gradient $\nabla f_i(x)$ is corrupted by a random noise variable.
The next proposition extends a proof for SVRG~\cite{proxsvrg} to stochastic perturbations, and characterizes the variance of~$g_k$.

As we now consider finite sums, we introduce the quantity $\tilde{\sigma}^2$, which is an upper-bound on the
noise variance due to stochastic perturbations for all $x$ in $\Real^p$ and for $i$ in $\{1,\ldots,n\}$:
\begin{equation}
   \E\left[\n{\tilde{\nb} {f}_i(x) - \nb f_i(x)}\right] \leq \tilde{\sigma}_i^2 
     ~~\text{with a related quantity}~~
     \tilde{\sigma}^2 = 
     \average \frac{1}{q_i n} \tilde{\sigma}_{i}^2,
     \label{eq:ts}
\end{equation}
where the expectation is with respect to the gradient perturbation, and
$Q=\{q_1,\ldots,q_n\}$ is the sampling distribution. As having the
variance to be bounded across the domain of $x$ may be a strong assumption,
even though classical, we also introduce the quantity 
 \eqm{
     \tilde{\sigma}_{i,\star}^2 = 
     \E\left[\n{\tilde{\nb} {f}_i(\xo) - \nb f_i(\xo)}\right]
     ~~\text{with a related quantity}~~
     \tilde{\sigma}_\star^2 = 
     \average \frac{1}{q_i n} \tilde{\sigma}_{i,\star}^2,
     \label{eq:tss}
}
where $x^\star$ is a solution of the optimization problem.  As we will show, in
this section, our complexity results for unaccelerated methods when $\mu > 0$ under the bounded variance assumption
$\tilde{\sigma}^2 < +\infty$ will also hold when simply assuming
$\tilde{\sigma}_\star^2 < +\infty$ at the cost of slightly degrading the complexity by
constant factors. The next proposition provides
an upper-bound on the variance of gradient estimators $g_k$,
which we have introduced earlier, as a first step to use Theorem~\ref{thm:conv}.

\begin{proposition}[\bf Generic variance reduction with non-uniform sampling]\label{prop:nonu} \hfill \break
Consider problem~(\ref{eq:prob}) when $f$ is a finite sum of functions $f=\frac{1}{n}\sum_{i=1}^n f_i$ where each~$f_i$ is convex and $L_i$-smooth with $L_i \geq \mu$. Then, the gradient estimates $g_k$ of the random-SVRG and MISO/SAGA/SDCA strategies defined in Section~\ref{subsec:gradients} satisfy
\begin{equation}
\E[\|g_k- \nabla f(x_{\kmone})\|^2] \leq 
4 L_Q \E[F(x_{\kmone})  - F^\star] + \frac{2}{n}\E\left[ \sum_{i=1}^n \frac{1}{n q_i}\| u^i_\kmone - u^i_\star\|^2\right] + 
3\rho_Q\tilde{\sigma}^2, 
\label{eq:var2}
\end{equation}
where $L_Q = \max_i L_i/(q_i n)$, $\rho_Q = 1/(n \min_i q_i)$, and for all $i$ and $k$, $u_k^i$ is equal to $z_k^i$ without noise---that is
\begin{displaymath}
\begin{split}
u_k^i & = \nabla f_i(\tilde{x}_k) ~~~\text{for random-SVRG}\\
u_{k}^{j_k} & = \nabla f_{j_k}(x_k) - \beta x_k ~~~\text{and}~~~ u_k^j  = u_\kmone^j ~~~\text{if}~~~j \neq j_k ~~~\text{for SAGA/MISO/SDCA},
\end{split}
\end{displaymath}
and $u^i_\star = \nabla f_i(x^\star)-\beta x^\star$ (with $\beta=0$ for random-SVRG).

   If we additionally assume that each function $f_i$ may be written as $f_i(x)=\E_{\xi}\brb{\tilde{f}_i(x,\xi)}$ where~$\tilde{f}_i(.,\xi)$ is $L_i$-smooth with $L_i \geq \mu$ for all~$\xi$,  then
\begin{equation}
\E[\|g_k- \nabla f(x_{\kmone})\|^2] \leq 
16 L_Q \E[F(x_{\kmone})  - F^\star] + \frac{2}{n}\E\left[ \sum_{i=1}^n \frac{1}{n q_i}\| u^i_\kmone - u^i_\star\|^2\right] + 
6\rho_Q\tilde{\sigma}_\star^2. 
\label{eq:var3}
\end{equation}
\end{proposition}
In particular, choosing the uniform distribution $q_i = 1/n$ gives $L_Q=\max_i L_i$; choosing $q_i =
L_i/\sum_j L_j$ gives $L_Q = \frac{1}{n} \sum_i L_i$, which may be
significantly smaller than the maximum Lipschitz constant. We note that
non-uniform sampling can significantly improve the dependency of the bound to
the Lipschitz constants since the average $\frac{1}{n}\sum_i L_i$ may be
significantly smaller than the maximum $\max_i L_i$, but it may worsen the dependency with
the variance $\tilde{\sigma}^2$ since $\rho_Q > 1$ unless $Q$ is the uniform distribution. The proof of the proposition is given in Appendix~\ref{appendix:prop.nonu}.

For simplicity,
we will present our complexity results in terms of~$\tilde{\sigma}^2$.
However, when the conditions for~(\ref{eq:var3}) are satisfied, it is easy to adapt all results of this section to replace~$\tilde{\sigma}^2$ by~$\tilde{\sigma}_{\star}^2$, by paying a small price in terms of constant factors.
Note that this substitution will not work for accelerated algorithms in the next section.
The general convergence result is given next; it applies to both variants~(\ref{eq:opt1}) and~(\ref{eq:opt2}).

\begin{proposition}[\bf Lyapunov function for variance-reduced algorithms]\label{thm:lyapunov}
Consider the same setting as Proposition~\ref{prop:nonu}.
For either variant~(\ref{eq:opt1}) or~(\ref{eq:opt2}) with the random-SVRG or SAGA/SDCA/MISO gradient estimators defined in Section~\ref{subsec:gradients}, when using the construction of $d_k$ from Sections~\ref{subsec:dk} or~\ref{subsec:dk2}, respectively, and assuming $\gamma_0 \geq \mu$ and $(\eta_k)_{k \geq 0}$ is non-increasing with $\eta_k \leq \frac{1}{12 L_Q}$, we have for all $k \geq 1$,
\begin{equation}
\frac{\delta_k}{6}\E[F(x_k)-F^\star] + T_k  \leq \left( 1 - \tau_k \right)T_\kmone  + {3 \rho_Q \eta_k \delta_k\tilde{\sigma}^2}~~~~\text{with}~~~~ \tau_k = \min\left(\delta_k, \frac{1}{5n}\right), \label{eq:aux2}
\end{equation}
where
$$T_k = {5} L_Q\eta_k \delta_k \E[F(x_k)-F^\star] + \E[d_k(x^\star)-d_k^\star] + \frac{5 \eta_k \delta_k}{2}\E\left[\frac{1}{n} \sum_{i=1}^n \frac{1}{q_i n} \| u^i_k - u^i_\star\|^2\right].
$$
\end{proposition}
The proof of the previous proposition is given in Appendix~\ref{appendix:prop.lyapunov}. From the Lyapunov function, we obtain a general convergence result for the variance-reduced stochastic algorithms. 
\begin{theorem}[\bf Convergence of variance-reduced algorithms]\label{thm:svrg}
   Consider the same setting as Proposition~\ref{thm:lyapunov}, which applies to both variants~(\ref{eq:opt1}) and~(\ref{eq:opt2}). Then, by using the averaging strategy of Lemma~\ref{lemma:averaging} with any point $\hat{x}_0$,
\begin{equation}
   \E\left[F(\hat{x}_k)-F^\star  + \frac{6 \tau_k}{\delta_k} T_k\right] \leq \Theta_k \left( F(\hat{x}_0)-F^\star + \frac{6 \tau_k}{\delta_k} T_0 + \frac{18 \rho_Q \tau_k \tilde{\sigma}^2}{\delta_k} \sum_{t=1}^k \frac{\eta_t \delta_t}{\Theta_t}     \right),
\label{eq:lyapunov.svrg}
\end{equation}
where $\Theta_k = \prod_{t=1}^k(1-\tau_t)$. Note that we also have
   \begin{equation}
      T_0 \leq   {10} L_Q\eta_0 \delta_0 (F(x_0)-F^\star) + d_0(x^\star)-d_0^\star.\label{eq:T0}
   \end{equation}
\end{theorem}
The proof is given in Appendix~\ref{appendix:thm:svrg}.
From this generic convergence theorem, we now study particular cases. 
The first corollary studies the strongly-convex case with constant step size. 

\begin{corollary}[\bf Variance-reduction, $\mu > 0$, constant step size independent of $\mu$]\label{corollary:svrg0}~\newline
   Consider the same setting as in Theorem~\ref{thm:svrg}, where $f$ is $\mu$-strongly convex, $\gamma_0=\mu$, and $\eta_k = \frac{1}{12L_Q}$. Then, for any point $\hat{x}_0$, 
 \begin{equation}
    \E\left[F(\hat{x}_k)-F^\star  + \alpha T_k \right] \leq \Theta_k\left( F(\hat{x}_0)-F^\star + \alpha T_0 \right)  + \frac{3 \rho_Q \tilde{\sigma}^2}{2L_Q} 
 \label{eq:svrg_constant}
 \end{equation}
   with $\tau = \min\left( \frac{\mu}{12L_Q}, \frac{1}{5n}\right)$, $\Theta_k= (1-\tau)^k$, and $\alpha=6 \min\left(1, \frac{12 L_Q}{5\mu n} \right)$.
   Note that $T_k \geq \frac{\mu}{2}\|x_k-x^\star\|^2$ and for Algorithm~(\ref{eq:opt1}), we also have $T_0 \leq (13/12)(F(x_0)-F^\star)$.
\end{corollary}
The proof is given in Appendix~\ref{subsec.corollary:svrg0}. This corollary shows that the algorithm achieves a linear convergence rate to a noise-dominated region and produces converging iterates $(x_k)_{k \geq 0}$ that do not require to know the strong convexity constant~$\mu$. It shows that all estimators we consider can become \emph{adaptive} to~$\mu$. 
Note that the non-uniform strategy slightly degrades the dependency in $\tilde{\sigma}^2$: indeed, $L_Q/\rho_Q = \max_{i=1} L_i$ if~$Q$ is uniform, but if $q_i = \max_i L_i/\sum_{j} L_j$, we have instead $L_Q/\rho_Q = \min_{i=1} L_i$. The next corollary shows that a slightly better noise dependency can be achieved when the step sizes rely on $\mu$.
\begin{corollary}[\bf Variance-reduction, $\mu > 0$, constant step size depending on $\mu$]\label{corollary:svrg1}~\newline
   Consider the same setting as Theorem~\ref{thm:svrg}, where $f$ is $\mu$-strongly convex, $\gamma_0=\mu$, and $\eta_k = \eta = \min\left(\frac{1}{12L_Q}, \frac{1}{5\mu n}\right)$. Then, for all $\hat{x}_0$, 
\begin{equation}
\begin{split}
   \E\left[F(\hat{x}_k)-F^\star  + 6 T_k\right] & \leq \Theta_k\left(  F(\hat{x}_0)-F^\star + 6 T_0  \right)  + {18 \rho_Q \eta \tilde{\sigma}^2}. \\
\end{split}\label{eq:svrg_constant2}
\end{equation}
\end{corollary}
The proof follows similar steps as the proof of Corollary~\ref{corollary:svrg0}, after noting that
we have $\delta_k = \tau_k$ for all $k$ for this particular choice of step size.
We are now in shape to study a converging algorithm.

\begin{corollary}[\bf Variance-reduction, $\mu > 0$, decreasing step sizes]\label{corollary:svrg2}
Consider the same setting as Theorem~\ref{thm:svrg}, where $f$ is $\mu$-strongly convex and target an accuracy $\varepsilon \leq  {24 \rho_Q \eta\tilde{\sigma}^2}$,
with $\eta = \min\left(\frac{1}{12L_Q}, \frac{1}{5\mu n}\right)$.
   Then, we use the constant step-size strategy of Corollary~\ref{corollary:svrg1} with $\hat{x}_0=x_0$, and stop the optimization when we find points~$\hat{x}_k$ and~$x_k$ such that $\E[F(\hat{x}_k)-F^\star+6 T_k] \leq {24 \rho_Q \eta \tilde{\sigma}^2}$.
Then, we restart the optimization procedure with decreasing step-sizes
   $\eta_k = \min \left(\frac{1}{12 L_Q}, \frac{1}{5\mu n}, \frac{2}{\mu (k+2)}\right)$ and generate a new sequence~$(\hat{x}_k')_{k \geq 0}$. The
resulting number of gradient evaluations to achieve $\E[F(\hat{x}_k')- F^\star] \leq  \varepsilon$
is upper bounded by
   $$O\left( \left(n + \frac{L_Q}{\mu}\right) \log\left(\frac{F(x_0)-F^\star + d_0(x^\star)-d_0^\star}{\varepsilon} \right)  \right) + O\left( \frac{\rho_Q \tilde{\sigma}^2}{\mu \varepsilon}\right).$$
   Note that $d_0(x^\star)-d_0^\star \leq F(x_0)-F^\star$ for variant~(\ref{eq:opt1}).
\end{corollary}
The proof is given in Appendix~\ref{appendix:cor.svrg2} and shows that variance-reduction algorithms may exhibit an optimal dependency on the noise level $\tilde{\sigma}^2$ when the objective is strongly convex.
Next, we analyze the complexity of variant~(\ref{eq:opt1}) when $\mu=0$. Note that it is possible to conduct a similar analysis for variant~(\ref{eq:opt2}), which exhibits a slightly worse complexity (as the corresponding quantity $d_0(x^\star)-d_0^\star$ is larger).
\begin{corollary}[\bf Variance-reduced algorithms with constant step-size, $\mu = 0$]
\label{thm:svrg.mu0}~\newline
Consider the same setting as Theorem~\ref{thm:svrg}, where $f$ is convex  and proceed in two steps.
   First, run one iteration of~(\ref{eq:opt1}) with step-size $\frac{1}{12L_Q}$ with the
gradient estimator~$(1/n)\sum_{i=1}^n \tildenabla f_i(\xz)$.
Second, use the resulting point to initialize the 
   variant~(\ref{eq:opt1}) with the random-SVRG or SAGA/SDCA/MISO gradient
   estimators, with a constant step size $\eta \leq \frac{1}{12L_Q}$,
   $\gamma_0=1/\eta$, for a total of~$K \ge 5n\log(5n)$ iterations.
   Then,
\begin{equation*}
   \E\left[F(\hat{x}_K)-F^\star\right] \leq \frac{9 n}{\eta (K+1)}\|x_0 - x^\star\|^2 + 36 \eta \tilde{\sigma}^2 \rho_Q.
\end{equation*}
   If in addition we choose $\eta=\min\left( \frac{1}{12L_Q}, \frac{\|x_0-x^\star\|}{2 \tilde{\sigma}}\sqrt{ \frac{n}{\rho_Q (K+1)}}\right) $. 
\begin{equation}
   \E\left[F(\hat{x}_K)-F^\star\right] \leq \frac{108 n L_Q}{(K+1)}\|x_0 - x^\star\|^2 + 36 \tilde{\sigma}\|x_0 - x^\star\|\sqrt{ \frac{\rho_Q n}{K+1} }.
\label{eq:cor.lyapunov.svrg}
\end{equation}
\end{corollary}
The proof is provided in Appendix~\ref{app.thm:svrg.mu0}. 
The second part of the corollary is not a practical result since the optimal step size depends on unknown quantities such as $\tilde{\sigma}^2$, but it allows us to highlight the best possible dependence between the budget of iterations~$K$, the initial point~$x_0$, and the noise~$\tilde{\sigma}^2$.
We will show in the next section that acceleration is useful to improve the previous complexity.

\section{Accelerated Stochastic Algorithms}\label{sec:acc}
We now consider the following iteration, involving an extrapolation sequence $(y_k)_{k \geq 1}$, which is a classical mechanism from accelerated first-order algorithms~\cite{fista,nesterov2013gradient}. Given a sequence of step-sizes $(\eta_k)_{k \geq 0}$ with $\eta_k \leq 1/L$  for all $k \geq 0$, and some parameter $\gamma_0 \geq \mu$, we consider the sequences $(\delta_k)_{k \geq 0}$ and $(\gamma_k)_{k \geq 0}$ that satisfy
\begin{displaymath}
\begin{split}
\delta_k & = \sqrt{\eta_k \gamma_k} ~~~\text{for all}~k \geq 0 \\
\gamma_k & = (1-\delta_k) \gamma_\kmone + \delta_k \mu ~~~\text{for all}~k \geq 1.\\
\end{split}
\end{displaymath}
Then, for $k \geq 1$, we consider the iteration
\custombox{
\vspace*{-0.2cm}
\begin{equation}
\begin{split}
x_{k}  & = \text{Prox}_{\eta_k\psi}\left[ y_\kmone - \eta_k g_k\right]~~~\text{with}~~~ \E[g_k | {\mathcal F}_\kmone] = \nabla f(y_\kmone) \\
y_k & = x_k + \beta_k(x_k - x_\kmone)~~~~\text{with}~~~ \beta_k = \frac{\delta_k(1-\delta_k)\eta_{k+1}}{\eta_k \delta_{k+1} + \eta_{k+1}\delta_k^2},
\end{split}
\tag{\textsf{C}} \label{eq:opt3}
\vspace*{0.1cm}
\end{equation}
}
where with constant step size $\eta_k=1/L$, we recover a classical extrapolation parameter of accelerated gradient based methods~\cite{nesterov}.
Traditionally, estimate sequences are used to analyze the convergence of
accelerated algorithms.
We show in this section how to proceed for stochastic composite optimization
and later, we show how to directly accelerate the random-SVRG approach we have introduced.
Note that Algorithm~(\ref{eq:opt3}) resembles the approaches introduced by~\citet{kwok2009,ghadimi2012optimal} but is simpler since our approach involves a single extrapolation step.

\subsection{Convergence Analysis Without Variance Reduction}
Consider then the stochastic estimate sequence for $k \geq 1$
\begin{equation*}
d_k(x) = (1-\delta_k) d_\kmone(x) + \delta_k l_k(x), 
\end{equation*}
with $d_0$ defined as in~(\ref{eq:d0}) and
\begin{equation}
l_k(x) = f(y_\kmone) + g_k^\top (x - y_\kmone) + \frac{\mu}{2}\|x-y_\kmone\|^2 + \psi(x_k) + \psi'(x_k)^\top (x-x_k), \label{eq:lk}
\end{equation}
and $\psi'(x_k)=\frac{1}{\eta_k}(y_\kmone-x_k) - g_k$ is in $\partial \psi(x_k)$ by definition of the proximal operator.
As in Section~\ref{sec:lower}, $d_k(x^\star)$ asymptotically becomes a lower bound on $F^\star$ since~(\ref{eq:est}) remains satisfied.
This time, the iterate $x_k$ does not minimize $d_k$, and we denote by $v_k$ instead its minimizer, allowing us to write $d_k$ in the canonical form
\begin{displaymath}
d_k(x) = d_k^\star + \frac{\gamma_k}{2}\|x- v_k\|^2.
\end{displaymath}
The first lemma highlights classical relations between the iterates $(x_k)_{k
\geq 0}$, $(y_k)_{k \geq 0}$ and the minimizers of the estimate sequences
$d_k$, which also appears in~\cite[][p. 78]{nesterov} for constant step sizes $\eta_k$. The proof is given in Appendix~\ref{appendix:cor.svrg2}.
\begin{lemma}[\bf Relations between $y_k$, $x_k$ and $d_k$]\label{lemma:acc}
The sequences $(x_k)_{k \geq 0}$ and $(y_k)_{k \geq 0}$ produced by
Algorithm~(\ref{eq:opt3}) satisfy for all $k \geq 0$, with $v_0=y_0=x_0$,
\begin{displaymath}
\begin{split}
y_k & = (1-\theta_k) x_k + \theta_k v_k ~~~~\text{with}~~~~ \theta_k = \frac{\delta_k\gamma_k}{ \gamma_k + \delta_{k+1}\mu}.
\end{split}
\end{displaymath}
\end{lemma}

Then, the next lemma is key to prove the convergence of Algorithm~(\ref{eq:opt3}). Its proof is given in Appendix~\ref{appendix:lemma.key_acc}.
\begin{lemma}[\bf Key lemma for stochastic estimate sequences with acceleration]\label{lemma:key_acc}~\newline
Assuming $(x_k)_{k \geq 0}$ and $(y_k)_{k \geq 0}$ are given by Algorithm~(\ref{eq:opt3}). Then, for all $k \geq 1$,
\begin{displaymath}
\E[F(x_k)]  \leq \E\left[l_k(y_\kmone)\right] + \left(\frac{L\eta_k^2}{2} - \eta_k\right)\E\left[\|\tilde{g}_k\|^2\right]  +  \eta_k\omega_k^2,
\end{displaymath}
with $\omega_k^2= \E[\|\nabla f(y_\kmone)-g_k\|^2]$ and $\tilde{g}_k=g_k + \psi'(x_k)$.
\end{lemma}
Finally, we obtain the following convergence result.
\begin{theorem}[\bf Accelerated stochastic optimization algorithm]\label{thm:acc_sgd}
Under the assumptions of Lemma~\ref{lemma:acc}, we have for all $k \geq 1$,
\begin{equation}
\E\left[F(x_k) - F^\star + \frac{\gamma_k}{2}\|v_k-x^\star\|^2\right] \leq \Gamma_k \left(F(x_0)-F^\star + \frac{\gamma_0}{2}\|x_0-x^\star\|^2 + \sum_{t=1}^k \frac{\eta_t \omega_t^2}{\Gamma_t}\right), \label{eq:acc1}
\end{equation}
where, as before, $\Gamma_t = \sum_{i=1}^t (1-\delta_i)$.
\end{theorem}
\begin{proof}
First, the minimizer $v_k$ of the quadratic surrogate $d_k$ may be written as
\begin{displaymath}
v_k = \frac{(1-\delta_k)\gamma_\kmone}{\gamma_k} v_\kmone + \frac{\mu \delta_k}{\gamma_k} y_\kmone - \frac{\delta_k}{\gamma_k} \tilde{g}_k 
 = y_\kmone + \frac{(1-\delta_k)\gamma_\kmone}{\gamma_k}(v_\kmone - y_\kmone) - \frac{\delta_k}{\gamma_k} \tilde{g}_k.
\end{displaymath}
Then, we characterize the quantity $d_k^\star$:
\begin{displaymath}
\begin{split}
d_k^\star  & = d_k(y_\kmone) - \frac{\gamma_k}{2}\|v_k - y_\kmone\|^2 \\
& = (1-\delta_k)d_\kmone(y_\kmone) + \delta_k l_k(y_\kmone) - \frac{\gamma_k}{2}\|v_k - y_\kmone\|^2 \\
& = (1-\delta_k)\left(d_\kmone^\star + \frac{\gamma_\kmone}{2}\|y_\kmone- v_\kmone\|^2\right) + \delta_k l_k(y_\kmone) - \frac{\gamma_k}{2}\|v_k - y_\kmone\|^2 \\
& = (1-\delta_k)d_\kmone^\star  + \left(\frac{\gamma_\kmone(1-\delta_k)(\gamma_k - (1-\delta_k)\gamma_\kmone)}{2\gamma_k} \right)\|y_\kmone- v_\kmone\|^2 + \delta_k l_k(y_\kmone)  \\
& ~~~~~~~~~~~  - \frac{\delta_k^2}{2\gamma_k}\|\tilde{g}_k\|^2+ \frac{\delta_k (1-\delta_k)\gamma_\kmone}{\gamma_k} \tilde{g}_k^\top (v_\kmone - y_\kmone) \\
& \geq  (1-\delta_k)d_\kmone^\star   + \delta_k l_k(y_\kmone) - \frac{\delta_k^2}{2\gamma_k}\|\tilde{g}_k\|^2
+ \frac{\delta_k (1-\delta_k)\gamma_\kmone}{\gamma_k} \tilde{g}_k^\top (v_\kmone - y_\kmone).
\end{split}
\end{displaymath}
Assuming by induction that $\E[d_\kmone^\star] \geq \E[F(x_\kmone)] - \xi_\kmone$ for some $\xi_\kmone \geq 0$, we have after taking expectation
\begin{multline*}
\E[d_k^\star]  \geq  (1-\delta_k)(\E[F(x_\kmone)] - \xi_\kmone) + \\ \delta_k \E[l_k(y_\kmone)] - \frac{\delta_k^2}{2\gamma_k}\E\|\tilde{g}_k\|^2
+ \frac{\delta_k (1-\delta_k)\gamma_\kmone}{\gamma_k} \E[\tilde{g}_k^\top (v_\kmone - y_\kmone)].
\end{multline*}
Then, note that $\E[F(x_\kmone)] \geq \E[l_k(x_\kmone)] \geq \E[l_k(y_\kmone)] + \E[\tilde{g}_k^\top(x_\kmone-y_\kmone)]$, and
\begin{multline*}
\E[d_k^\star] \geq  \E[l_k(y_\kmone)] - (1-\delta_k)\xi_\kmone  - \frac{\delta_k^2}{2\gamma_k}\E\|\tilde{g}_k\|^2 \\
+ (1-\delta_k)\E\left[\tilde{g}_k^\top \left(\frac{\delta_k\gamma_\kmone}{\gamma_k} (v_\kmone - y_\kmone) + (x_\kmone-y_\kmone)\right)\right].
\end{multline*}
By Lemma~\ref{lemma:acc}, we can show that the last term is equal to zero, and we are left with
\begin{equation*}
\E[d_k^\star] \geq  \E[l_k(y_\kmone)] - (1-\delta_k)\xi_\kmone -
\frac{\delta_k^2}{2\gamma_k}\E\|\tilde{g}_k\|^2.
\end{equation*}
We may then use Lemma~\ref{lemma:key_acc}, which gives us
\begin{equation*}
\begin{split}
\E[d_k^\star] & \geq  \E[F(x_k)] - (1-\delta_k)\xi_\kmone - \eta_k \omega_k^2
+ \left( \eta_k - \frac{L \eta_k^2}{2} - \frac{\delta_k^2}{2\gamma_k}\right)\E\|\tilde{g}_k\|^2 \\
& \geq \E[F(x_k)] - \xi_k~~~~\text{with}~~~~ \xi_k = (1-\delta_k)\xi_\kmone + \eta_k \omega_k^2,
\end{split}
\end{equation*}
where we used the fact that $\eta_k \leq 1/L$ and $\delta_k=\sqrt{\gamma_k \eta_k}$.

It remains to choose $d_0^\star = F(x_0)$ and $\xi_0=0$ to initialize the induction at $k=0$ and we conclude that
\begin{equation*}
\E\left[F(x_k) - F^\star + \frac{\gamma_k}{2}\|v_k-x^\star\|^2\right] \leq \E[d_k(x^\star)-F^\star] + \xi_k \leq \Gamma_k (d_0(x^\star) - F^\star) + \xi_k,  
\end{equation*}
which gives us~(\ref{eq:acc1}) when noticing that $\xi_k = \Gamma_k \sum_{t=1}^k \frac{\eta_t \omega_t^2}{\Gamma_t}$.
\end{proof}

Next, we specialize the theorem to various practical cases. 
For the corollaries below, we assume the variances~$\brc{\omega_k\sq}_{k\ge1}$ to be upper bounded by~$\sa\sq$.

\begin{corollary}[\bf Proximal accelerated SGD with constant step-size, $\mu > 0$]\label{corollary:acc_sgd2a}
Assume that $f$ is $\mu$-strongly convex, 
and choose $\gamma_0=\mu$ and $\eta_k = 1/L$ with Algorithm~(\ref{eq:opt3}). Then,
\begin{equation}
\E\left[F(x_k)- F^\star\right] \leq   \left(1-\sqrt{\frac{\mu}{L}}\right)^k\left(F(x_0)- F^\star + \frac{\mu}{2} \|x_0 - x^\star\|^2\right) + \frac{\sigma^2}{\sqrt{\mu L}}. \label{eq:acc_sgd}
\end{equation}
\end{corollary}
We now show that with decreasing step sizes, we obtain an algorithm with optimal complexity similar to~\citep{ghadimi2013optimal}.
\begin{corollary}[\bf Proximal accelerated SGD with decreasing step-sizes and $\mu > 0$] \label{corollary:acc_sgd2}
Assume that $f$ is $\mu$-strongly convex and that we target an accuracy
$\varepsilon$ smaller than $2\sigma^2/\sqrt{\mu L}$. First, use a constant step-size $\eta_k=1/L$ with $\gamma_0=\mu$ within Algorithm~(\ref{eq:opt3}),
leading to the convergence rate~(\ref{eq:acc_sgd}),  until $\E[F({x}_k)- F^\star] \leq  2 \sigma^2/\sqrt{\mu L}$.
Then, we restart the optimization procedure with decreasing step-sizes $\eta_k
= \min \left(\frac{1}{L},\frac{4}{\mu (k+2)^2}\right)$  and generate a new sequence~$(\hat{x}_k)_{k \geq 0}$. The resulting number of gradient evaluations to achieve $\E[F({x}_k)- F^\star] \leq  \varepsilon$ is upper bounded by
$$O\left( \sqrt{\frac{L}{\mu}} \log\left(\frac{F(x_0)- F^\star}{\varepsilon}\right)\right) + O\left( \frac{\sigma^2}{\mu \varepsilon}\right).$$
\end{corollary}
The proof is provided in Appendix~\ref{appendix:corollary:acc_sgd2}.
We note that despite the ``optimal'' theoretical complexity, we have observed
that Algorithm~(\ref{eq:opt3}) with the parameters of Corollaries~\ref{corollary:acc_sgd2a} and~\ref{corollary:acc_sgd2} could be relatively unstable, as shown in Section~\ref{sec:exp}, due to the large radius
$\sigma^2/\sqrt{\mu L}$ of the noise region. When $\mu$ is small, such a
quantity may be indeed arbitrarily larger than $F(x_0)-F^\star$. Instead, we have found a minibatch strategy to be more effective in practice.
When using a minibatch of size $b=\lceil L/\mu\rceil$, the theoretical complexity becomes the same as SGD, given in Corollary~\ref{corollary:sgd}, but the algorithm enjoys the benefits of easy parallelization.

\begin{corollary}[\bf Proximal accelerated SGD with with $\mu = 0$] \label{corollary:acc_sgd3}
   Assume that $f$ is convex.
   Consider a step-size $\eta \leq 1/L$ and run one iteration of Algorithm~(\ref{eq:opt1}) with a stochastic gradient estimate. 
   Use the resulting point to initialize Algorithm~(\ref{eq:opt3}) still with constant step size $\eta$, and choose $\gamma_0=1/\eta$.
   Then,
   \begin{displaymath}
      \E[F(x_k)-F^\star] \leq \frac{2\|x_0-x^\star\|^2}{(1+K)^2\eta} + \sigma^2 \eta (K+1)
   \end{displaymath}
   If in addition we choose $\eta 
= \min\left(\frac{1}{L}, \sqrt{\frac{2 \|x_0-x^\star\|^2}{\sigma^2}}\frac{1}{(K+1)^{3/2}}\right)$, then
   \begin{equation}
      \E[F(x_k)-F^\star] \leq \frac{2L \|x_0-x^\star\|^2}{(1+K)^2}  + 2 \|x_0-x^\star\| \sigma \sqrt{\frac{2}{1+K}}.    \label{eq:sgd_simple2}
   \end{equation}

\end{corollary}
The proof is given in Appendix~\ref{appendix:corollary:acc_sgd3}.
These convergence results are relatively similar to those obtained in~\cite{ghadimi2013optimal} for a different algorithm and
is optimal for convex functions.

\subsection{An Accelerated Algorithm with Variance Reduction}
In this section, we show how to combine the previous methodology with variance reduction, and introduce Algorithm~\ref{alg:acC_svrg} based on random-SVRG. Then, we present the convergence analysis, which requires controlling the variance of the estimator in a similar manner to~\cite{accsvrg}, as stated in the next proposition. Note that the estimator does not require storing the seed of the random perturbations, unlike in the previous section.

\begin{algorithm}[tp]
\caption{Accelerated algorithm with random-SVRG estimator}\label{alg:acC_svrg}
\begin{algorithmic}[1]
\State {\bfseries Input:} $x_0$ in $\Real^p$ (initial point); $K$ (number of iterations); $(\eta_k)_{k \geq 0}$ (step sizes); $\gamma_0 \geq \mu$;
\State {\bfseries Initialization:} $\tilde{x}_0 = v_0 = x_0$; $\bar{z}_0= \tildenabla f(x_0)$;
\For{$k=1,\ldots,K$}
\State Find $(\delta_k,\gamma_k)$ such that
\begin{displaymath}
\gamma_k = (1-\delta_k)\gamma_\kmone + \delta_k \mu ~~~~\text{and}~~~\delta_k = \sqrt{\frac{5\eta_k\gamma_k}{3n}};
\end{displaymath}
\State Choose
\begin{displaymath}
y_\kmone =  \theta_k v_\kmone + (1-\theta_k) \tilde{x}_\kmone~~~\text{with}~~~ \theta_k = \frac{3 n \delta_k - 5 \mu \eta_k}{3-5\mu \eta_k};
\end{displaymath}
\State Sample $i_k$ according to the distribution $Q=\{q_1,\ldots,q_n\}$;
\State Compute the gradient estimator, possibly corrupted by stochastic perturbations:
\begin{equation*}
g_k = \frac{1}{q_{i_k} n}\left(\tildenabla f_{i_k} (y_\kmone) - \tildenabla f_{i_k}(\tilde{x}_\kmone) \right) + \bar{z}_\kmone; 
\end{equation*}
\State Obtain the new iterate
$x_{k} \leftarrow \text{Prox}_{\eta_k\psi}\left[ y_\kmone - \eta_k g_k\right];$
\State Find the minimizer $v_k$ of the estimate sequence $d_k$:
\begin{displaymath}
v_k = \left(1- \frac{\mu \delta_k}{\gamma_k}\right)v_\kmone + \frac{\mu\delta_k}{\gamma_k}y_\kmone + \frac{\delta_k}{\gamma_k \eta_k}(x_k-y_\kmone);
\end{displaymath}
\State With probability $1/n$, update the anchor point
\begin{displaymath}
\tilde{x}_k = x_k ~~~~\text{and}~~~~ \bar{z}_k = \tildenabla f(\tilde{x}_k);
\end{displaymath}
\State Otherwise, keep the anchor point unchanged $\tilde{x}_k = \tilde{x}_\kmone$ and $\bar{z}_k = \bar{z}_\kmone$;
\EndFor
\State {\bfseries Output:} ${x}_K$.
\end{algorithmic}
\end{algorithm}

\begin{proposition}[\bf Variance reduction for random-SVRG estimator]\label{prop:nonu2}
Consider problem~(\ref{eq:prob}) when $f$ is a finite sum of
functions $f=\frac{1}{n}\sum_{i=1}^n f_i$ where each $f_i$ is
$L_i$-smooth with $L_i \geq \mu$ and $f$ is $\mu$-strongly convex.
Then, the variance of $g_k$ defined in Algorithm~\ref{alg:acC_svrg}
satisfies
\begin{displaymath}
\omega_k^2 \leq {2 L_Q}\left[  f(\tilde{x}_\kmone) - f(y_{\kmone}) -  g_k^\top (\tilde{x}_\kmone-y_\kmone)\right]  + {3 \rho_Q \tilde{\sigma}^2}.
\end{displaymath}
\end{proposition}
The proof is given in Appendix~\ref{appendix:prop.nonu2}. Then, we extend Lemma~\ref{lemma:key_acc} that was used in the previous analysis to the variance-reduction setting.
\begin{lemma}[\bf Lemma for accelerated variance-reduced stochastic optimization]\label{lemma:key_acc_svrg}~\newline
Consider the iterates provided by Algorithm~\ref{alg:acC_svrg} and
call $a_k = 2L_Q \eta_k$. Then,
\begin{multline*}
\E[F(x_k)]   \leq  \E\left[a_k F(\tilde{x}_\kmone) + (1- a_k) l_k(y_\kmone)\right] \\ 
+ \E\left[a_k \tilde{g}_k^\top (y_\kmone - \tilde{x}_\kmone) + \left(\frac{L\eta_k^2}{2} - \eta_k\right)\|\tilde{g}_k\|^2\right] + {3\rho_Q \eta_k\tilde{\sigma}^2}.
\end{multline*}
\end{lemma}
The proof of this lemma is given in Appendix~\ref{appendix:lemma:key_acc_svrg}.
With this lemma in hand, we may now state our main convergence result.
\begin{theorem}[\bf Convergence of the accelerated SVRG algorithm]\label{thrm:acc_svrg}
Consider the iterates provided by Algorithm~\ref{alg:acC_svrg} and assume that the step sizes satisfy
$\eta_k \leq \min \left( \frac{1}{3L_Q}, \frac{1}{15 \gamma_k n} \right)$ for all $k \geq 1$. Then,
\begin{equation}
\E\left[F(x_k)-F^\star+ \frac{\gamma_k}{2}\|v_k-x^\star\|^2\right] \leq \Gamma_k \left(F(x_0)-F^\star + \frac{\gamma_0}{2}\|x_0-x^\star\|^2 + \frac{3\rho_Q\tilde{\sigma}^2 }{n}\sum_{t=1}^k \frac{\eta_t}{\Gamma_t} \right). \label{eq:acc_rate_svrg}
\end{equation}
\end{theorem}
\begin{proof}
Following similar steps as in the proof of Theorem~\ref{thm:acc_sgd}, we have
\begin{displaymath}
d_k^\star \geq  (1-\delta_k)d_\kmone^\star   + \delta_k l_k(y_\kmone) - \frac{\delta_k^2}{2\gamma_k}\|\tilde{g}_k\|^2
+ \frac{\delta_k (1-\delta_k)\gamma_\kmone}{\gamma_k} \tilde{g}_k^\top (v_\kmone - y_\kmone).
\end{displaymath}
Assume now by induction that $\E[d_\kmone^\star] \geq \E[F(\tilde{x}_\kmone)] - \xi_\kmone$ for some $\xi_\kmone \geq 0$ and
note that $\delta_k \leq \frac{1-a_k}{n}$ since $a_k = 2L_Q\eta_k \leq \frac{2}{3}$ and $\delta_k = \sqrt{\frac{5\eta_k\gamma_k}{3n}} \leq \frac{1}{3n} \leq \frac{1-a_k}{n}$. Then,
\begin{displaymath}
\begin{split}
\E[d_k^\star] & \geq (1-\delta_k) (\E[F(\tilde{x}_\kmone)] - \xi_\kmone) + \delta_k\E[l_k(y_\kmone)] - \frac{\delta_k^2}{2\gamma_k}\E[\|\tilde{g}_k\|^2] \\ 
   & \qquad \qquad + \E\left[\tilde{g}_k^\top \left(\frac{\delta_k (1-\delta_k)\gamma_\kmone}{\gamma_k} (v_\kmone - y_\kmone)\right)\right] \\
& \geq \left(1-\frac{1-a_k}{n}\right) \E[F(\tilde{x}_\kmone)] + \left(\frac{1-a_k}{n}-\delta_k\right) \E[F(\tilde{x}_\kmone)]  + \delta_k\E[l_k(y_\kmone)] - \frac{\delta_k^2}{2\gamma_k}\E\|\tilde{g}_k\|^2 \\
& \qquad \qquad\qquad\qquad
+ \E\left[\tilde{g}_k^\top \left(\frac{\delta_k (1-\delta_k)\gamma_\kmone}{\gamma_k} (v_\kmone - y_\kmone)\right)\right] - (1-\delta_k)\xi_\kmone.
\end{split}
\end{displaymath}
Note that
\begin{displaymath}
\E[F(\tilde{x}_\kmone)] \geq \E[l_k(\tilde{x}_\kmone)] \geq \E[l_k(y_\kmone)] + \E[\tilde{g}_k^\top(\tilde{x}_\kmone - y_\kmone)].
\end{displaymath}
Then,
\begin{multline*}
\E[d_k^\star] \geq \left(1-\frac{1-a_k}{n}\right)\E[F(\tilde{x}_\kmone)] + \frac{1-a_k}{n}\E[l_k(y_\kmone)] - \frac{\delta_k^2}{2\gamma_k}\E[\|\tilde{g}_k\|^2] \\ + \E\left[\tilde{g}_k^\top \left(\frac{\delta_k (1-\delta_k)\gamma_\kmone}{\gamma_k} (v_\kmone - y_\kmone) + \left(\frac{1-a_k}{n}-\delta_k\right)(\tilde{x}_\kmone - y_\kmone) \right)\right] - (1-\delta_k)\xi_\kmone.
\end{multline*}
We may now use Lemma~\ref{lemma:key_acc_svrg}, which gives us
\begin{multline}
\E[d_k^\star] \geq \left(1-\frac{1}{n}\right) \E[F(\tilde{x}_\kmone)] + \frac{1}{n}\E[F(x_k)] + \left(\frac{1}{n}\left(\eta_k - \frac{L \eta_k^2}{2}\right)- \frac{\delta_k^2}{2\gamma_k}\right)\E[\|\tilde{g}_k\|^2] \\ + \E\left[\tilde{g}_k^\top \left(\frac{\delta_k (1-\delta_k)\gamma_\kmone}{\gamma_k} (v_\kmone - y_\kmone) + \left(\frac{1}{n}-\delta_k\right)(\tilde{x}_\kmone - y_\kmone)\right) \right] - \xi_k, \label{eq:aux_svrg_acc}
\end{multline}
with $\xi_k = (1-\delta_k)\xi_\kmone + \frac{3 \rho_Q \eta_k \tilde{\sigma}^2}{n}$.
Then, since $\delta_k = \sqrt{\frac{5\eta_k\gamma_k}{3n}}$ and $\eta_k \leq \frac{1}{3L_Q} \leq \frac{1}{3L}$,
\begin{displaymath}
\frac{1}{n}\left(\eta_k - \frac{L \eta_k^2}{2}\right)- \frac{\delta_k^2}{2\gamma_k} \geq \frac{5\eta_k}{6n}- \frac{\delta_k^2}{2\gamma_k} = 0,
\end{displaymath}
and the term in~(\ref{eq:aux_svrg_acc}) involving $\|\tilde{g}_k\|^2$ may disappear. Similarly, we have
\begin{displaymath}
\frac{\delta_k(1-\delta_k)\gamma_{\kmone}}{\delta_k(1-\delta_k)\gamma_{\kmone} + \gamma_k/n - \delta_k \gamma_k}  = \frac{\delta_k\gamma_k - \delta_k^2\mu}{\gamma_k/n - \delta_k^2\mu} = \frac{3n\delta_k^3/5 \eta_k - \delta_k^2\mu}{3 \delta_k^2/5\eta_k - \delta_k^2\mu}= \frac{3n - 5\mu\eta_k}{3 - 5\mu\eta_k}= \theta_k,
\end{displaymath}
and the term in~(\ref{eq:aux_svrg_acc}) that is linear in $\tilde{g}_k$ may disappear as well.
Then, we are left with
$\E[d_k^\star] \geq \E[F(\tilde{x}_k)] - \xi_k$. Initializing the induction requires choosing $\xi_0=0$ and $d_0^\star = F(x_0)$. Ultimately, we note that $\E[d_k(x^\star)-F^\star] \leq (1-\delta_k)\E[d_\kmone(x^\star)-F^\star]$ for all $k \geq 1$, and
$$ \E\left[ F(\tilde{x}_k) - F^\star \!+\! \frac{\gamma_k}{2}\|x^\star-v_k\|^2\right] \leq \E[d_k(x^\star)-F^\star] + \xi_k \leq \Gamma_k\left(F(x_0) - F^\star \!+\! \frac{\gamma_0}{2}\|x^\star \!-\! x_0\|^2 \right) + \xi_k,$$
and we obtain~(\ref{eq:acc_rate_svrg}).
\end{proof}
We may now derive convergence rates of our accelerated SVRG algorithm under various settings. The proofs of the following corollaries, when not straightforward, are given in the appendix. The first corollary simply uses Lemma~\ref{lemma:simple}.
\begin{corollary}[\bf Accelerated proximal SVRG - constant step size - $\mu > 0$]\label{corollary:accsvrg_constant}\hfill \break
With $\eta_k =  \min \left( \frac{1}{3L_Q}, \frac{1}{15 \mu n} \right)$ and $\gamma_0 = \mu$,
the iterates produced by Algorithm~\ref{alg:acC_svrg} satisfy
\begin{itemize}
\item if $\frac{1}{3L_Q} \leq \frac{1}{15 \mu n}$,
\begin{equation*}
\E\left[F(x_k)-F^\star\right] \leq \left( 1- \sqrt{\frac{5\mu}{9L_Qn}}\right)^k \left(F(x_0)-F^\star + \frac{\mu}{2}\|x_0-x^\star\|^2 \right) + \frac{3\rho_Q \tilde{\sigma}^2}{\sqrt{5 \mu L_Q n}};
\end{equation*}
\item otherwise,
\begin{equation*}
\E\left[F(x_k)-F^\star\right] \leq \left( 1- \frac{1}{3n} \right)^k \left(F(x_0)-F^\star + \frac{\mu}{2}\|x_0-x^\star\|^2 \right) + \frac{3\rho_Q \tilde{\sigma}^2}{5\mu n}.
\end{equation*}
\end{itemize}
\end{corollary}
The corollary uses the fact that $\Gamma_k\sum_{t=1}^k \eta/\Gamma_t \leq \eta/\delta = \sqrt{3 n \eta/5\mu}$ and thus
the algorithm converges linearly to an area of radius $3\rho_Q \tilde{\sigma}^2 \sqrt{3 \eta/ 5 \mu n}  = O\left( \rho_Q{\tilde{\sigma}^2}\min \left( \frac{1}{\sqrt{n \mu L_Q}}, \frac{1}{\mu n} \right) \right)$, where as before, $\rho_Q=1$ if the distribution $Q$ is uniform. When $\tilde{\sigma}^2=0$, the corresponding algorithm achieves the optimal complexity for finite sums~\cite{arjevani2016dimension}.
Interestingly, we see that here non-uniform sampling may hurt the convergence guarantees in some situations. Whenever $\frac{1}{\max_i L_i}  > \frac{1}{5\mu n}$, the optimal sampling strategy is indeed the uniform one.
Next, we show how to obtain a converging algorithm in the next corollary.
\begin{corollary}[\bf Accelerated proximal SVRG - diminishing step sizes - $\mu > 0$]\label{corollary:acc_svrg}~\newline
Assume that $f$ is $\mu$-strongly convex and that we target an accuracy
$\varepsilon$ smaller than $B=3\rho_Q \tilde{\sigma}^2\sqrt{\eta/\mu}$ with the same step size $\eta$ as in the previous corollary.
First, use such a constant step-size strategy $\eta_k=\eta$ with $\gamma_0=\mu$ within Algorithm~\ref{alg:acC_svrg},
leading to the convergence rate of the previous corollary,  until $\E[F({x}_k)- F^\star] \leq B$.
Then, we restart the optimization procedure with decreasing step-sizes $\eta_k
= \min \left(\eta ,\frac{12 n}{5 \mu (k+2)^2}\right)$  and generate a new sequence~$(\hat{x}_k)_{k \geq 0}$. The resulting number of gradient evaluations to achieve $\E[F({x}_k)- F^\star] \leq  \varepsilon$ is upper bounded by
$$O\left( \left(n + \sqrt{\frac{{nL_Q}}{\mu}}\right) \log\left(\frac{F(x_0)- F^\star}{\varepsilon}\right)\right) + O\left( \frac{\rho_Q \sigma^2}{\mu \varepsilon}\right).$$
\end{corollary}
The proof is given in Appendix~\ref{appendix:corollary:acc_svrg}. Next, we study the case when $\mu=0$.

\begin{corollary}[\bf Accelerated proximal SVRG - $\mu =0$]\label{corollary:acc_svrg_convex2}
   Consider the same setting as in Theorem \ref{thrm:acc_svrg}, where $f$ is convex
   and proceed in two steps.
   First, run one iteration of~(\ref{eq:opt1}) with step-size $\eta \leq \frac{1}{3L_Q}$ with the
gradient estimator~$(1/n)\sum_{i=1}^n \tildenabla f_i(\xz)$.
Second, use the resulting point to initialize Algorithm~\ref{alg:acC_svrg} and 
   use step size $\eta_t = \min\left( \frac{1}{3 L_Q}, \frac{1}{15\gamma_t n}\right)$, with
   $\gamma_0=1/\eta$, for a total of~$K \geq 6n\log(15n)+1$ iterations.
Then
\begin{equation*}
   \E\left[F({x}_K)-F^\star\right] \leq \frac{6n\|x_0-x^\star\|^2}{\eta (K+1)^2} + \frac{3\eta\rho_Q\tilde{\sigma}^2(K + 1)}{n}.
\end{equation*}
   If in addition we choose $\eta=\min \left(\frac{1}{3L_Q}, \frac{\sqrt{2}n \|x_0-x^\star\|}{\tilde{\sigma} \sqrt{\rho_Q} (K+1)^{3/2}}  \right)$, 
\begin{equation}
   \E\left[F({x}_K)-F^\star\right] \leq  \frac{18 L_Q n\|x_0-x^\star\|^2}{(K+1)^2} + \frac{6 \tilde{\sigma} \|x_0-x^\star\|\sqrt{2\rho_Q}}{\sqrt{K+1}}.  \label{eq:svrg.convex2}
\end{equation}
\end{corollary}
The proof is provided in Appendix~\ref{appendix:corollary:acc_svrg_convex2}.
When $\tilde{\sigma}^2=0$ (deterministic setting), the 
first part of the corollary with $\eta=1/3L_Q$ gives us the same complexity as Katyusha~\citep{accsvrg}, and in the stochastic case, we obtain a significantly better complexity than
the same algorithm without acceleration, which was analyzed in Corollary~\ref{thm:svrg.mu0}.

\section{Experiments}\label{sec:exp}
In this section, we evaluate numerically the approaches introduced in the previous sections.

\subsection{Datasets, Formulations, and Methods}
Following classical benchmarks in optimization methods for machine learning~\citep[see, \eg][]{schmidt2017minimizing},
we consider empirical risk minimization formulations. Given training data
$(a_i,b_i)_{i=1,\ldots,n}$, with $a_i$ in $\Real^p$ and $b_i$ in $\{-1,+1\}$, we consider the optimization problem
\begin{displaymath}
\min_{x \in \Real^p} \frac{1}{n}\sum_{i=1}^n \phi(b_i a_i^\top x)  + \frac{\lambda}{2}\|x\|^2,
\end{displaymath}
where $\phi$ is either the logistic loss $\phi(u)=\log(1+e^{-u})$, or the squared hinge loss $\phi(u)=\max(0,1-u)^2$.
Both functions are $L$-smooth; when the vectors $a_i$ have unit norm, we may indeed choose $L=0.25$ for the logistic loss
and $L=1$ for the squared hinge loss. Studying the squared hinge loss is interesting: whereas the
logistic loss has bounded gradients on $\Real^p$, this is not the case for the squared hinge loss. With unbounded optimization domain,
the gradient norms may be indeed large in some regions of the solution space, which may lead in turn to large variance $\sigma^2$ of
the gradient estimates obtained by SGD, causing instabilities.

The scalar $\lambda$ is a regularization parameter that acts as a lower bound on the
strong convexity constant of the problem. We consider the parameters
$\mu=\lambda=1/10n$ in our problems, which is of the order of the smallest
values that one would try when doing a parameter search, \eg, by cross-validation.
For instance, this is empirically observed for the dataset~\textrm{cifar-ckn} described below, where a test set is available, allowing us to check that the ``optimal'' regularization parameter leading to the lowest generalization error is indeed of this order.
We also report an experiment with $\lambda=1/100n$ in order to study the effect of the problem conditioning on the method's performance.

Following~\citet{smiso,zheng2018lightweight}, we consider DropOut
perturbations~\citep{srivastava_dropout:_2014} to illustrate the robustness to noise of the algorithms. DropOut consists of randomly
setting to zero each entry of a data point with probability $\delta$, leading to the optimization problem
\begin{equation}
\min_{x \in \Real^p} \frac{1}{n}\sum_{i=1}^n \E_\rho\left[ \phi(b_i (\rho \circ a_i)^\top x)\right]  + \frac{\lambda}{2}\|x\|^2, \label{eq:expectation}
\end{equation}
where $\rho$ is a binary vector in $\{0,1\}^p$ with i.i.d. Bernoulli entries, and $\circ$ denotes the elementwise multiplication between two vectors.
We consider two DropOut regimes, with $\delta$ in $\{0.01,0.1\}$, representing small and medium perturbations, respectively.

We consider the following three datasets coming from different scientific fields
\begin{itemize}
\item \textrm{alpha} is from  the  Pascal  Large  Scale Learning Challenge
website\footnote{\url{http://largescale.ml.tu-berlin.de/}} and contains $n=250\,000$
points in dimension $p=500$.
\item \textrm{gene} consists of gene expression data and the binary labels $b_i$ characterize two different types of breast cancer. This is a small dataset with $n=295$ and $p=8\,141$.
\item \textrm{ckn-cifar} is an image classification task where each image from the CIFAR-10 dataset\footnote{\url{https://www.cs.toronto.edu/~kriz/cifar.html}} is represented by using a two-layer unsupervised convolutional neural network~\citep{mairal2016end}. Since CIFAR-10 originally contains 10 different classes, we consider the binary classification task consisting of predicting the class 1 vs. other classes. The dataset contains $n=50\,000$ images and the dimension of the representation is $p=9\,216$.
\end{itemize}
For simplicity, we normalize the features of all datasets and thus we use a uniform sampling strategy~$Q$ in all algorithms.
Then, we consider several methods with their theoretical step sizes, described in Table~\ref{table:algs}.
Note that we also evaluate the strategy \textrm{random-SVRG} with step size $1/3L$, even though our analysis requires $1/12L$, in order to get a fair comparison with the accelerated SVRG method.
In all figures, we consider that $n$ iterations of SVRG
count as $2$ effective passes over the data since it appears to be a good
proxy of the computational time.
Indeed, (i) if one is allowed to store the variables $z_i^k$, then $n$ iterations exactly correspond to two passes over the data; (ii) the gradients $\tildenabla f_i(x_\kmone)- \tildenabla f_i(\tilde{x}_\kmone)$ access the same training point which reduces the data access overhead; (iii) computing the full gradient~$\bar{z}_k$ can be done in practice in a much more efficient manner than computing individually the~$n$ gradients $\tildenabla f_i(x_k)$, either through parallelization or by using more efficient routines (\eg, BLAS2 vs BLAS1 routines for linear algebra).
Each experiment is conducted five times and we always report the average of the five experiments in each figure.
We also include in the comparison two baselines from the literature: \textrm{AC-SA} is the accelerated stochastic gradient descent method of~\citet{ghadimi2013optimal}, and \textrm{adam-heur} is the Adam method of~\citet{kingma2014adam} with its recommended step size. 
As Adam is not converging, we adopt a standard heuristics from the deep learning literature, consisting of reducing the step size by $10$ after $50$ and $150$ passes over the data, respectively, which performs much better than using a constant step size in practice.

\begin{table}
\definecolor{alizarin}{rgb}{0.82, 0.1, 0.26}
\centering
\addtolength{\tabcolsep}{-0.07cm}
\begin{tabular}{|c|c|c|c|c|}
\hline
Algorithm & step size $\eta_k$ & Theory & Complexity $O(.)$ & Bias $O(.)$ \\
\hline
\textrm{SGD} & $\frac{1}{L}$ &  Cor.~\ref{corollary:sgd_constant}  & $\frac{L}{\mu}\log\left(\frac{C_0}{\varepsilon}  \right)$ & $\frac{\sigma^2}{L}$ \\
\hline
\textrm{SGD-d} & $\min
\left(\frac{1}{L},\frac{2}{\mu (k+2)}\right)$ &  Cor.~\ref{corollary:sgd}  & $ \frac{L}{\mu}\log\left(\frac{C_0}{\varepsilon}\right) + \frac{\sigma^2}{\mu \varepsilon}  $ & 0 \\
\hline
\textrm{acc-SGD} & $\frac{1}{L}$ &  Cor.~\ref{corollary:acc_sgd2a}  & $ \sqrt{\frac{L}{\mu}}\log\left(\frac{C_0}{\varepsilon}\right)$ & $\frac{\sigma^2}{\sqrt{\mu L}}$ \\
\hline
\textrm{acc-SGD-d} & $\min
\left(\frac{1}{L},\frac{4}{\mu (k+2)^2}\right)$ &  Cor.~\ref{corollary:acc_sgd2}  & $ \sqrt{\frac{L}{\mu}}\log\left(\frac{C_0}{\varepsilon}\right) + \frac{\sigma^2}{\mu \varepsilon}  $ & 0 \\
\hline
\textrm{acc-mb-SGD-d} & $\min
\left(\frac{1}{L},\frac{4}{\mu (k+2)^2}\right)$ &  Cor.~\ref{corollary:acc_sgd2}  & $ \frac{L}{\mu}\log\left(\frac{C_0}{\varepsilon}\right) + \frac{\sigma^2}{\mu \varepsilon}  $ & 0 \\
\hline
\textrm{rand-SVRG} & $\frac{1}{3 L}$ &  Cor.~\ref{corollary:svrg0}  & $ \left(n + \frac{L}{\mu}\right)\log\left(\frac{C_0}{\varepsilon}\right) $ & $\frac{{\color{alizarin} \tilde{\sigma}^2}}{L}$ \\
\hline
\textrm{rand-SVRG-d} & $\min\left(\frac{1}{12L_Q}, \frac{1}{5\mu n}, \frac{2}{\mu(k+2)}\right)$ &  Cor.~\ref{corollary:svrg2}  & $ \left(n + \frac{L}{\mu}\right)\log\left(\frac{C_0}{\varepsilon}\right) + \frac{{\color{alizarin}\tilde{\sigma}^2}}{\mu \varepsilon}  $ & 0 \\
\hline
\textrm{acc-SVRG} & $\min \left( \frac{1}{3L_Q}, \frac{1}{15 \mu n} \right)$ &  Cor.~\ref{corollary:accsvrg_constant}  & $ \left(n + \sqrt{\frac{nL}{\mu}}\right)\log\left(\frac{C_0}{\varepsilon}\right) $ & $\frac{{\color{alizarin}\tilde{\sigma}^2}}{\sqrt{n \mu L}+n \mu}$ \\
\hline
\textrm{acc-SVRG-d} & $\min \left( \frac{1}{3L_Q}, \frac{1}{15 \mu n} ,\frac{12 n}{5 \mu (k+2)^2}\right)$ &  Cor.~\ref{corollary:acc_svrg}  & $ \left(n + \sqrt{\frac{nL}{\mu}}\right)\log\left(\frac{C_0}{\varepsilon}\right)  + \frac{{\color{alizarin}\tilde{\sigma}^2}}{\mu \varepsilon}$ & 0 \\
\hline
\end{tabular}
\caption{List of algorithms used in the experiments, along with the step size used and the pointer to the corresponding convergence guarantees, with $C_0=F(x_0)-F^\star$. In the experiments, we also use the method \textrm{rand-SVRG} with step size $\eta=1/3L$, even though our analysis requires $\eta \leq 1/12L$.
The approach \textrm{acc-mb-SGD-d} uses minibatches of size $\lceil\sqrt{L/\mu}\rceil$ and could thus easily be parallelized.
Note that we potentially have ${\color{alizarin}\tilde{\sigma}} \ll \sigma$.}\label{table:algs}
\end{table}

\subsection{Evaluation of Algorithms without Perturbations}
First, we study the behavior of all methods when $\tilde{\sigma}^2=0$. We
report the corresponding results in Figures~\ref{fig:nodropout}, \ref{fig:nodropout2}, and~\ref{fig:nodropout3}.  Since the
problem is deterministic, we can check that the value $F^\star$ we consider is
indeed optimal by computing a duality gap using Fenchel duality.
For \textrm{SGD} and random-SVRG, we do not use any averaging strategy, which we found
to empirically slow down convergence, when used from the start;
knowing when to start averaging is indeed not easy and requires heuristics which we do not evaluate here.

\begin{figure}[h!]
\centering
\subfloat[Dataset \textrm{gene}]{\includegraphics[width=0.33\linewidth]{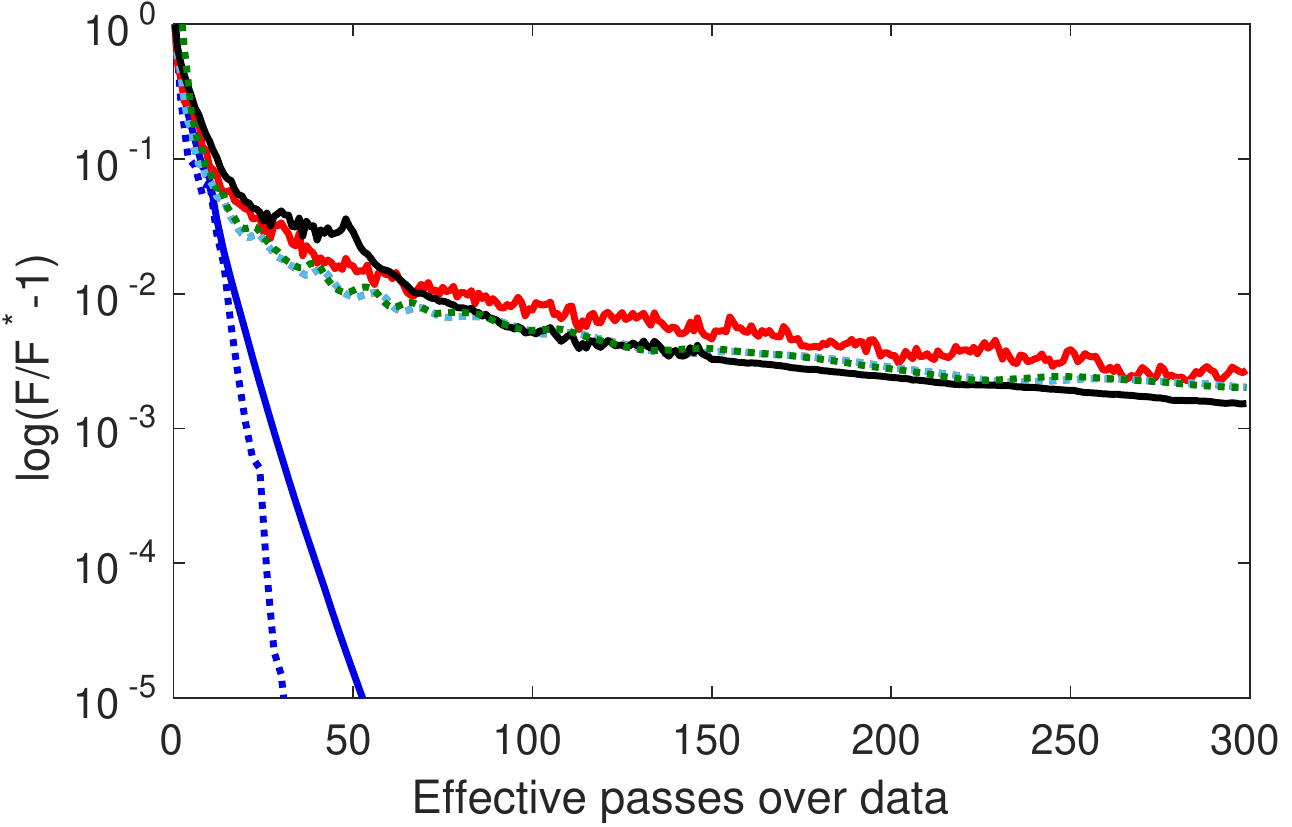}}
\subfloat[Dataset \textrm{ckn-cifar}]{\includegraphics[width=0.33\linewidth]{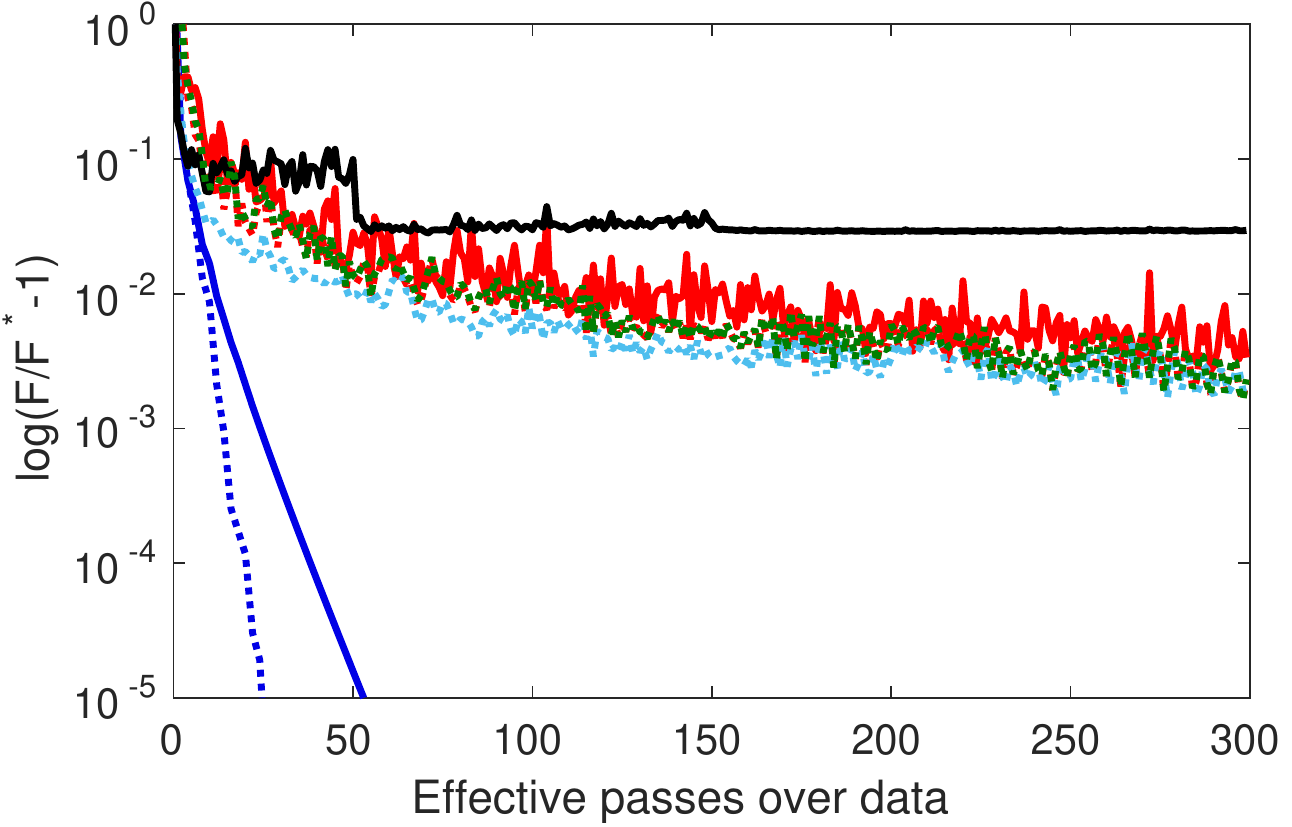}}
\subfloat[Dataset \textrm{alpha}]{\includegraphics[width=0.33\linewidth]{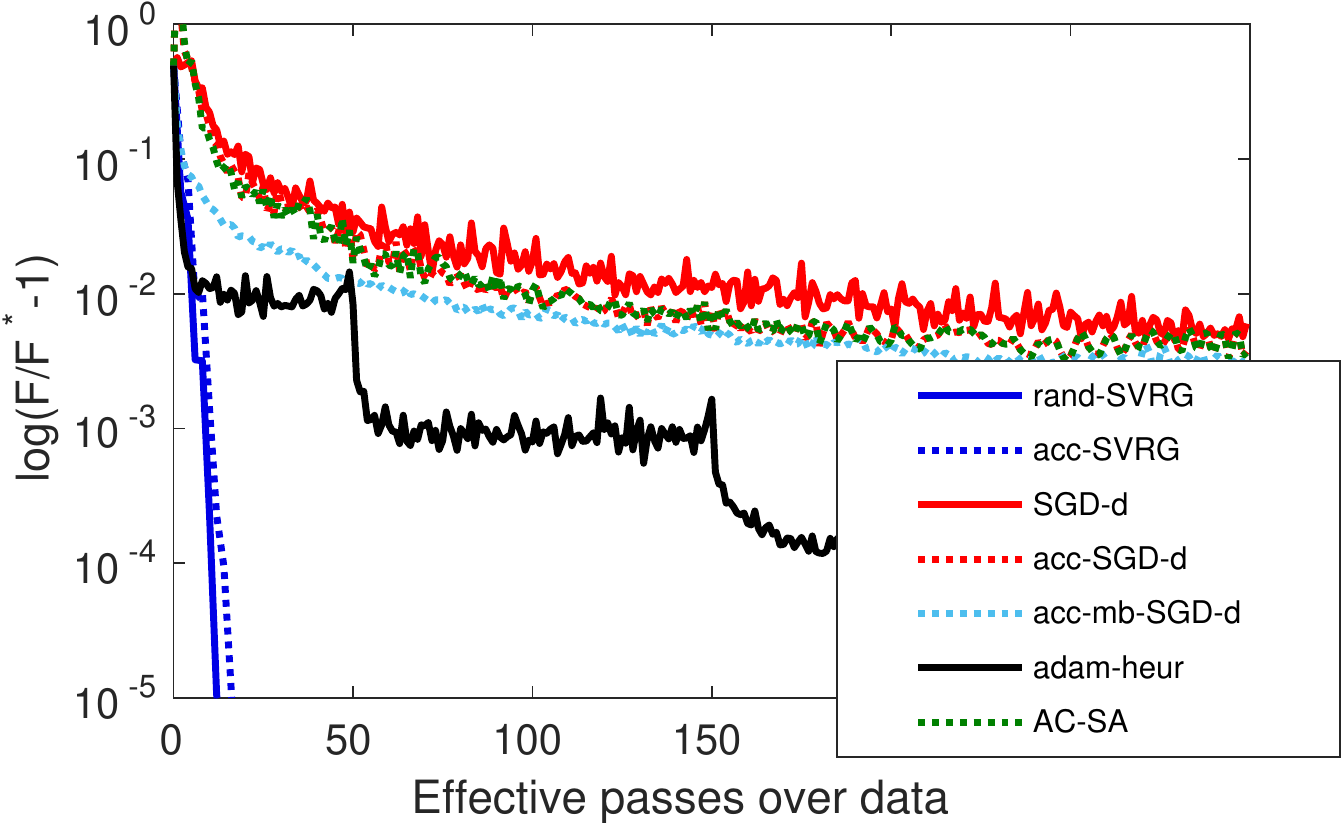}}
   \caption{Optimization curves without perturbations when using the logistic loss and the parameter $\lambda=1/10n$. We plot the value of the objective function on a logarithmic scale as a function of the effective passes over the data (see main text for details). Best seen in color by zooming on a computer screen. Note that the method \textrm{Adam} is not converging.}\label{fig:nodropout}
\end{figure}

\begin{figure}[h!]
\centering
\subfloat[Dataset \textrm{gene}]{\includegraphics[width=0.33\linewidth]{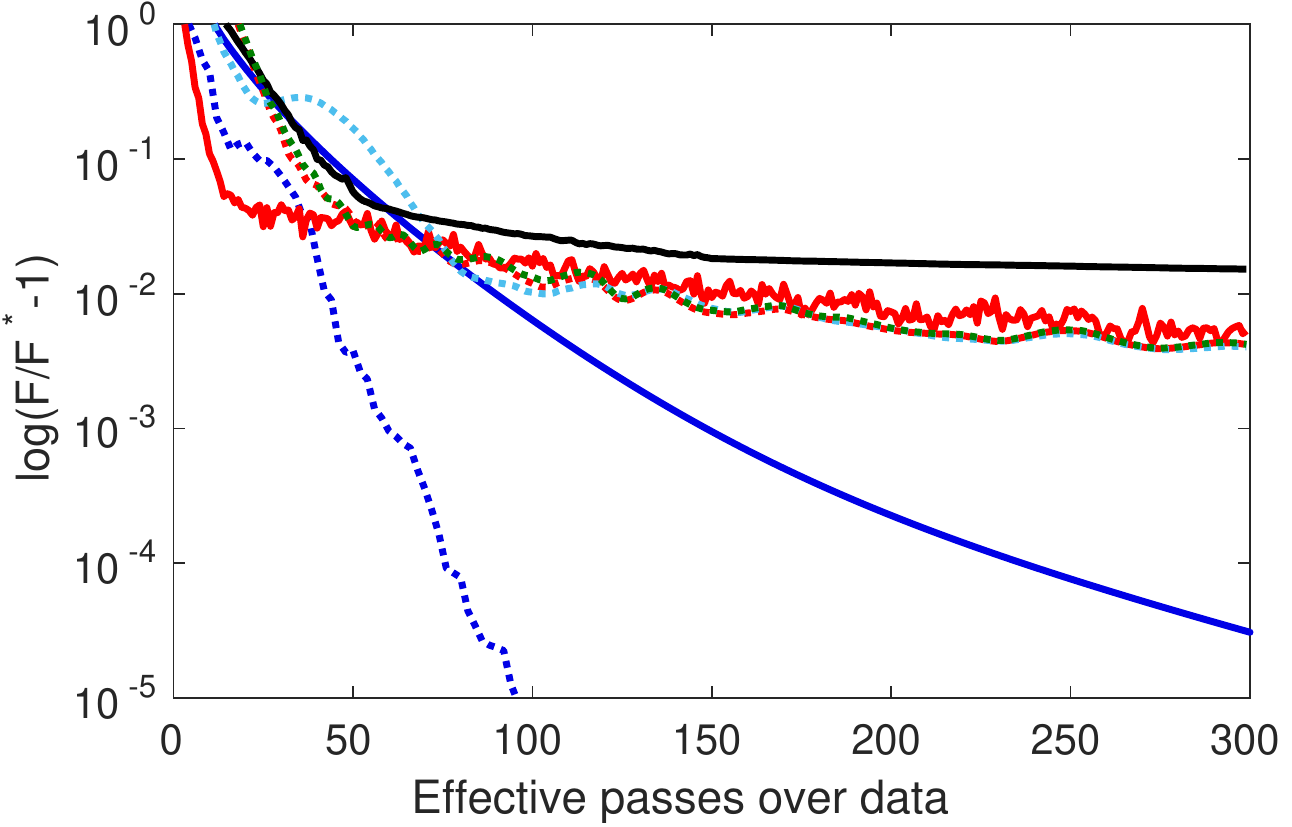}}
\subfloat[Dataset \textrm{ckn-cifar}]{\includegraphics[width=0.33\linewidth]{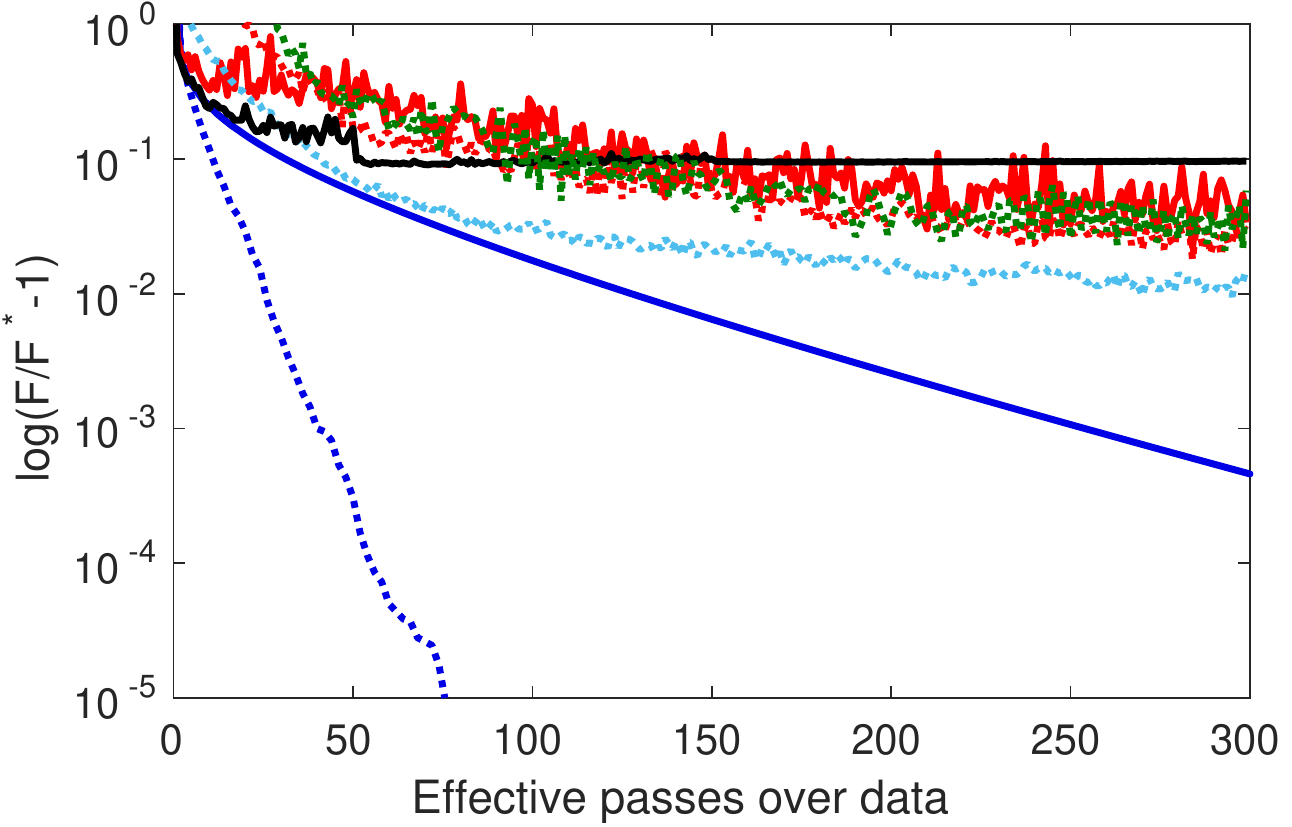}}
\subfloat[Dataset \textrm{alpha}]{\includegraphics[width=0.33\linewidth]{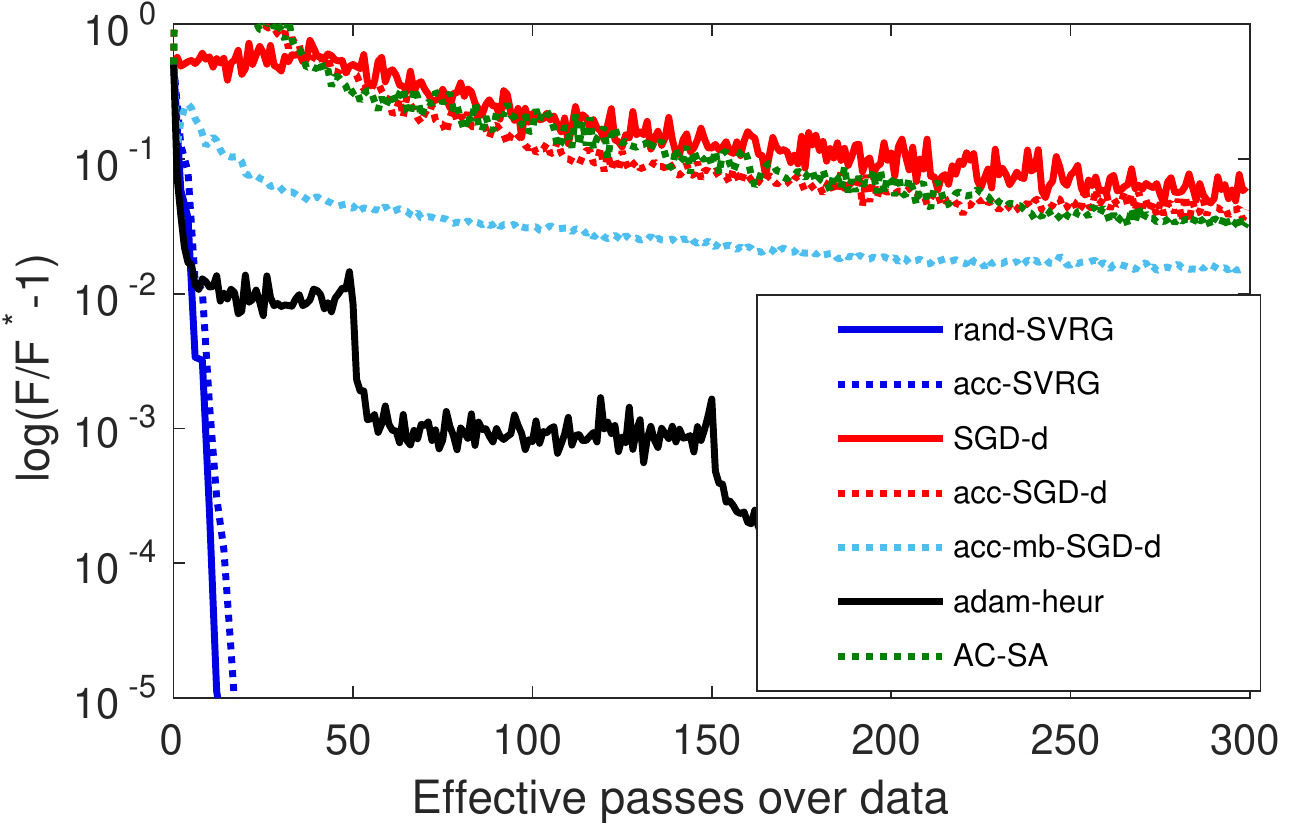}}
\caption{Same experiment as in Figure~\ref{fig:nodropout} with $\lambda=1/100n$.}\label{fig:nodropout2}
\end{figure}

\begin{figure}[h!]
\centering
\subfloat[Dataset \textrm{gene}]{\includegraphics[width=0.33\linewidth]{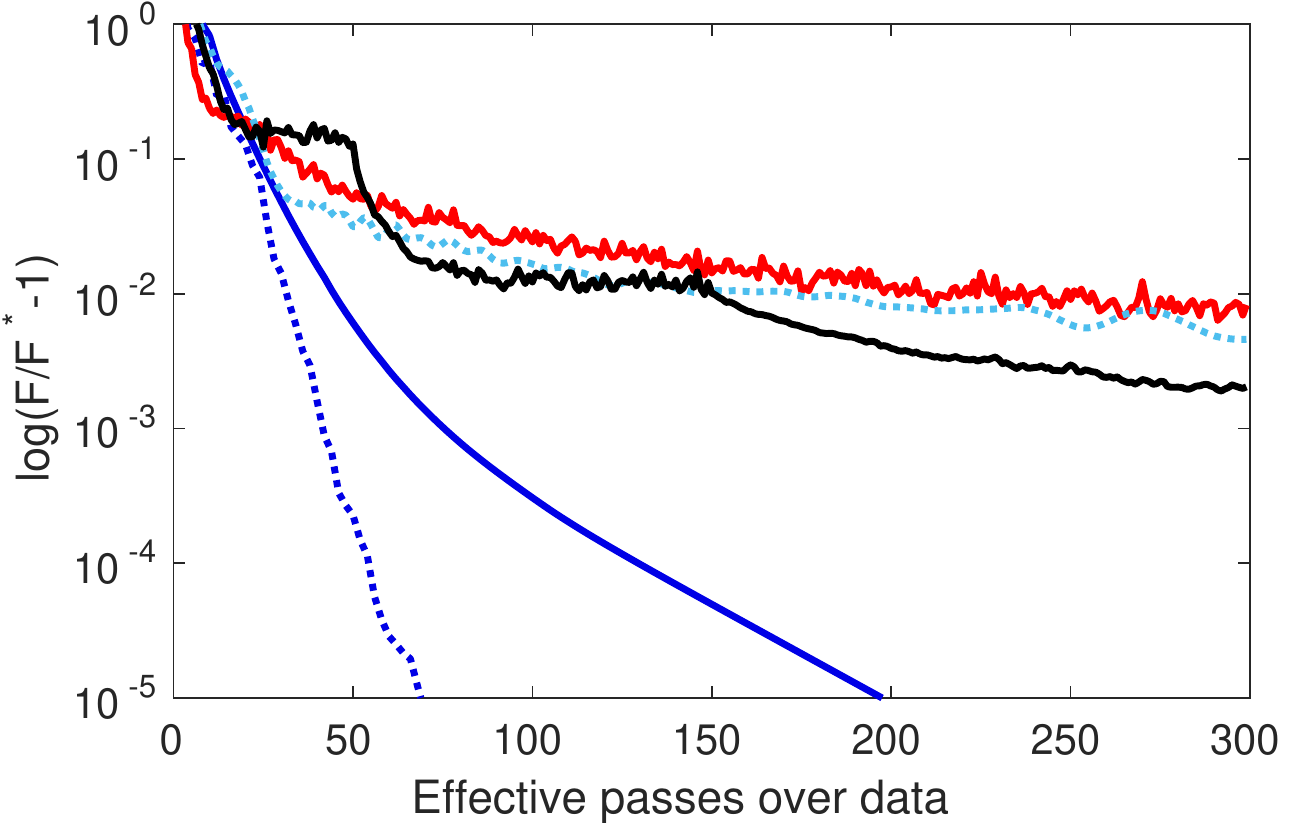}}
\subfloat[Dataset \textrm{ckn-cifar}]{\includegraphics[width=0.33\linewidth]{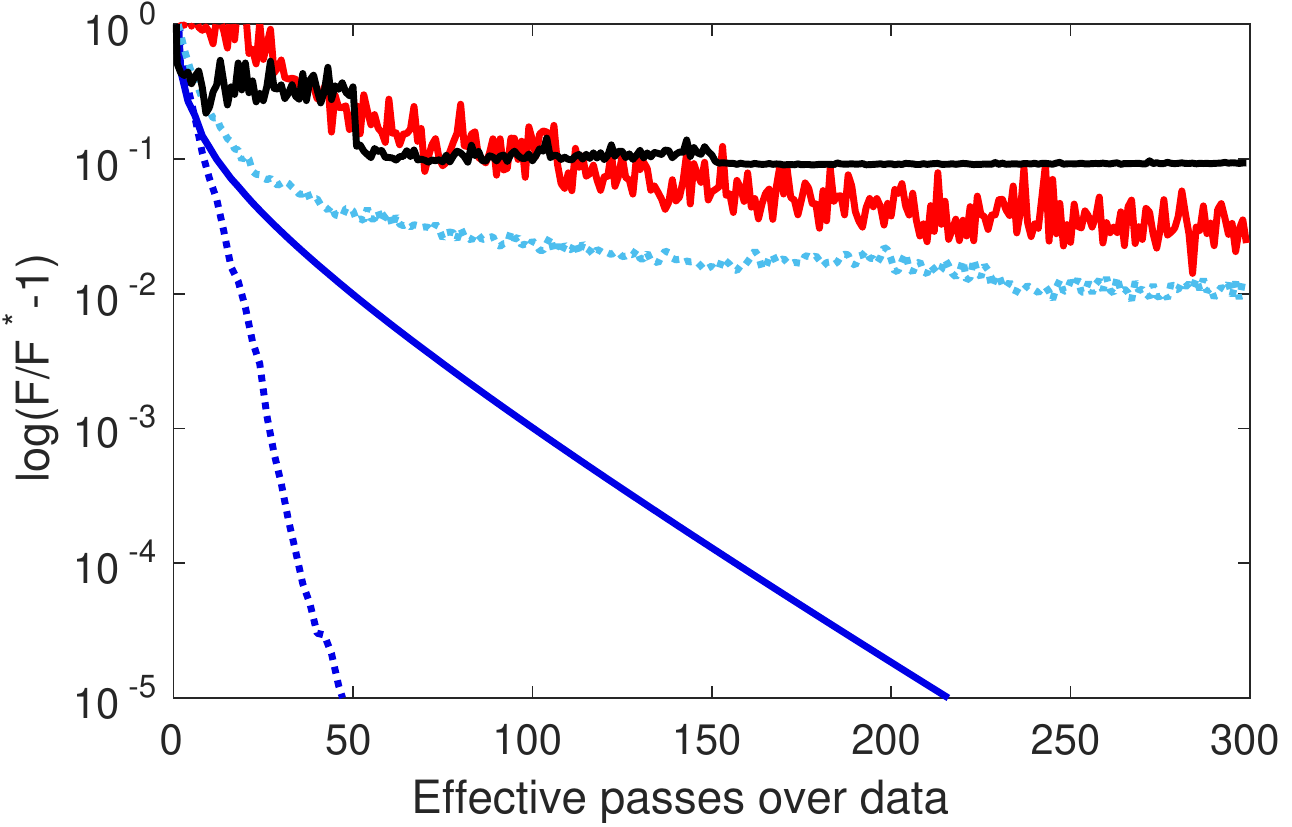}}
\subfloat[Dataset \textrm{alpha}]{\includegraphics[width=0.33\linewidth]{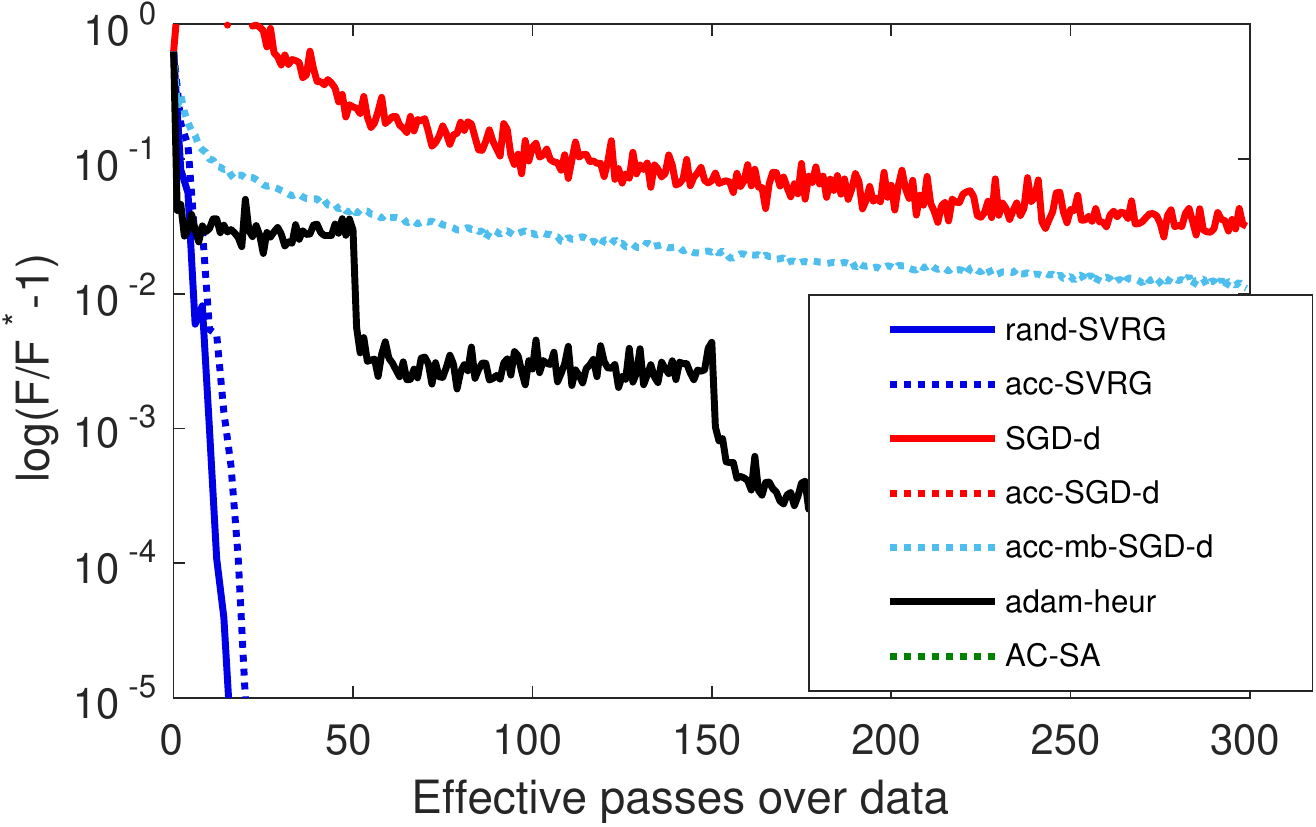}}
   \caption{Same experiment as in Figure~\ref{fig:nodropout} with squared hinge loss instead of logistic. \textrm{ACC-SA} and \textrm{acc-SGD-d} were unstable for this setting due to the large size of the noise region $\sigma^2/\sqrt{\mu L} = \sqrt{10n}\sigma^2$ and potentially large gradients of the loss function over the optimization domain.}\label{fig:nodropout3}
\end{figure}

From these experiments, we obtain the following conclusions:
\begin{itemize}
\item {Acceleration for SVRG is effective} on the datasets \textrm{gene}
and~\textrm{ckn-cifar} except on~\textrm{alpha}, where all SVRG-like methods
perform already well. This may be due to strong convexity hidden in~\textrm{alpha} leading to a regime where acceleration does not occur---that is, when the complexity is $O(n \log(1/\varepsilon))$, which is independent of the condition number.
Note that this algorithm is now implemented in the open-source Cyanure toolbox~\cite{mairal2019cyanure}\footnote{http://julien.mairal.org/cyanure/}.
\item {Acceleration is more effective when the problem is badly conditioned}. When $\lambda=1/100n$, acceleration brings several orders of magnitude improvement in complexity.
\item {Accelerated SGD is unstable} with the squared hinge loss. During the initial phase with constant step size $1/L$, the expected primal gap is in a region of radius $O(\sigma^2/\sqrt{\mu L}) \approx \sqrt{n} \sigma^2$, which is potentially huge, causing large gradients and instabilities.
\item {Accelerated minibatch SGD} performs best among the SGD methods and is competitive with SVRG in the low precision regime. The performance of Adam on these datasets is inconsistent; it performs best among SGD methods on \textrm{alpha}, but is significantly worse on \textrm{ckn-cifar}. Note also that \textrm{AC-SA} performs in general similarly to \textrm{acc-SGD-d}.  

\end{itemize}

\subsection{Evaluation of Algorithms with Perturbations}
We now consider the same setting as in the previous section, but we add DropOut
perturbations with rate $\delta$ in $\{0.01,0.1\}$.  As predicted by theory,
all approaches with constant step size do not converge. Therefore, we only
report the results for decreasing step sizes
in Figures~\ref{fig:dropout}, \ref{fig:dropout2}, and~\ref{fig:dropout3}.
We evaluate the loss function every $5$ data passes and we estimate the expectation~(\ref{eq:expectation}) by drawing $5$ random
perturbations per data point, resulting in $5n$ samples. The optimal value $F^\star$ is estimated by letting the methods
run for $1000$ epochs and selecting the best point found as a proxy of $F^\star$.

\begin{figure}[hbtp]
\centering
\subfloat[Dataset \textrm{gene}, $\delta=0.01$]{\includegraphics[width=0.33\linewidth]{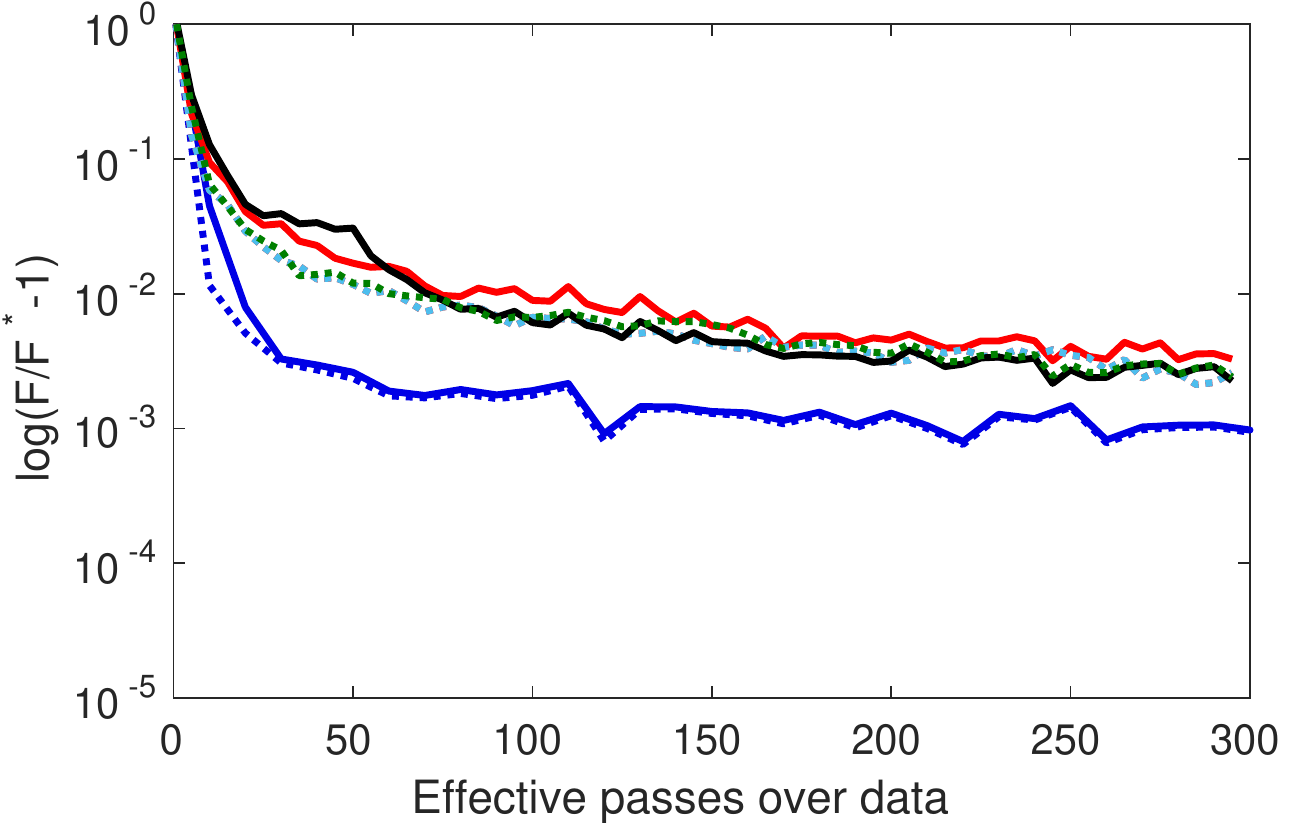}}
\subfloat[Dataset \textrm{ckn-cifar}, $\delta=0.01$]{\includegraphics[width=0.33\linewidth]{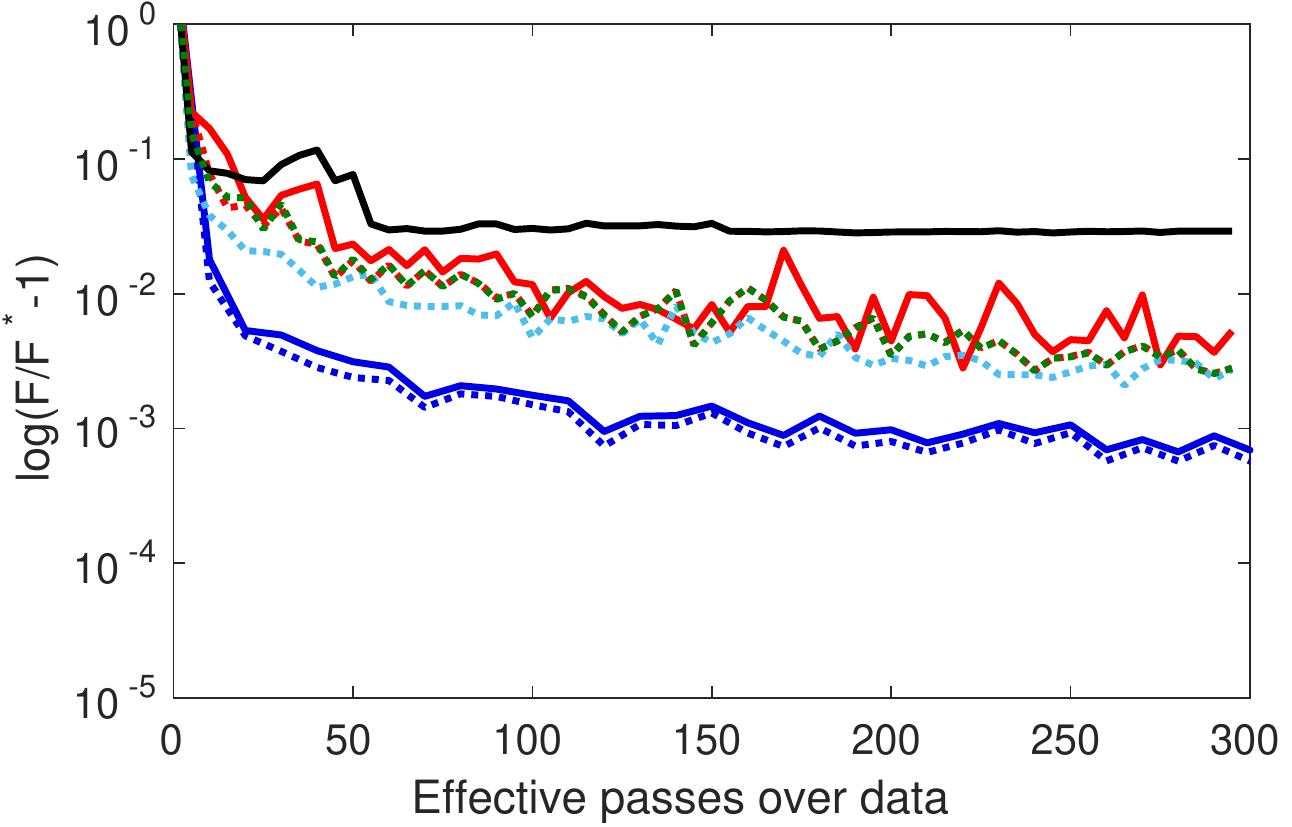}}
\subfloat[Dataset \textrm{alpha}, $\delta=0.01$]{\includegraphics[width=0.33\linewidth]{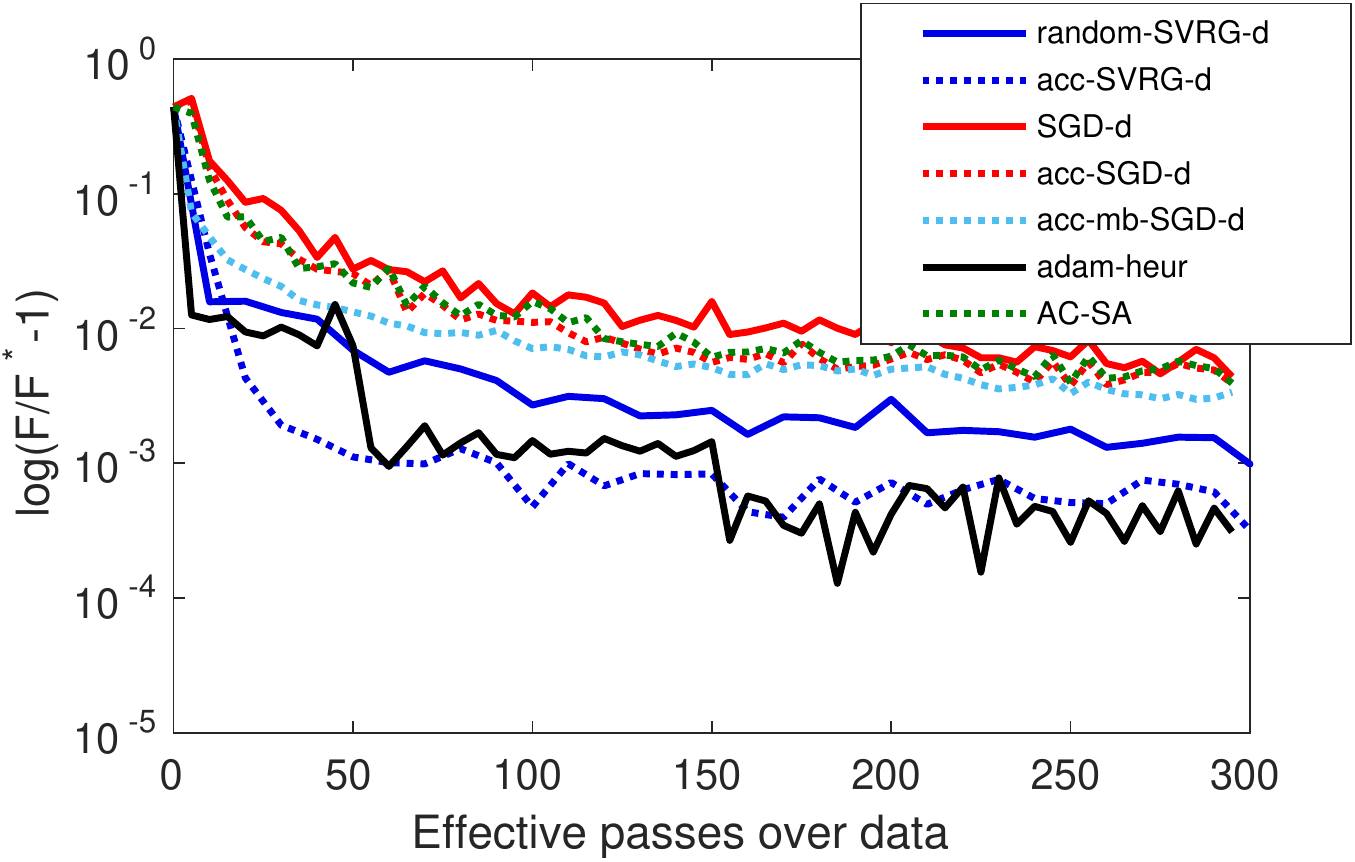}} \\
\subfloat[Dataset \textrm{gene}, $\delta=0.1$]{\includegraphics[width=0.33\linewidth]{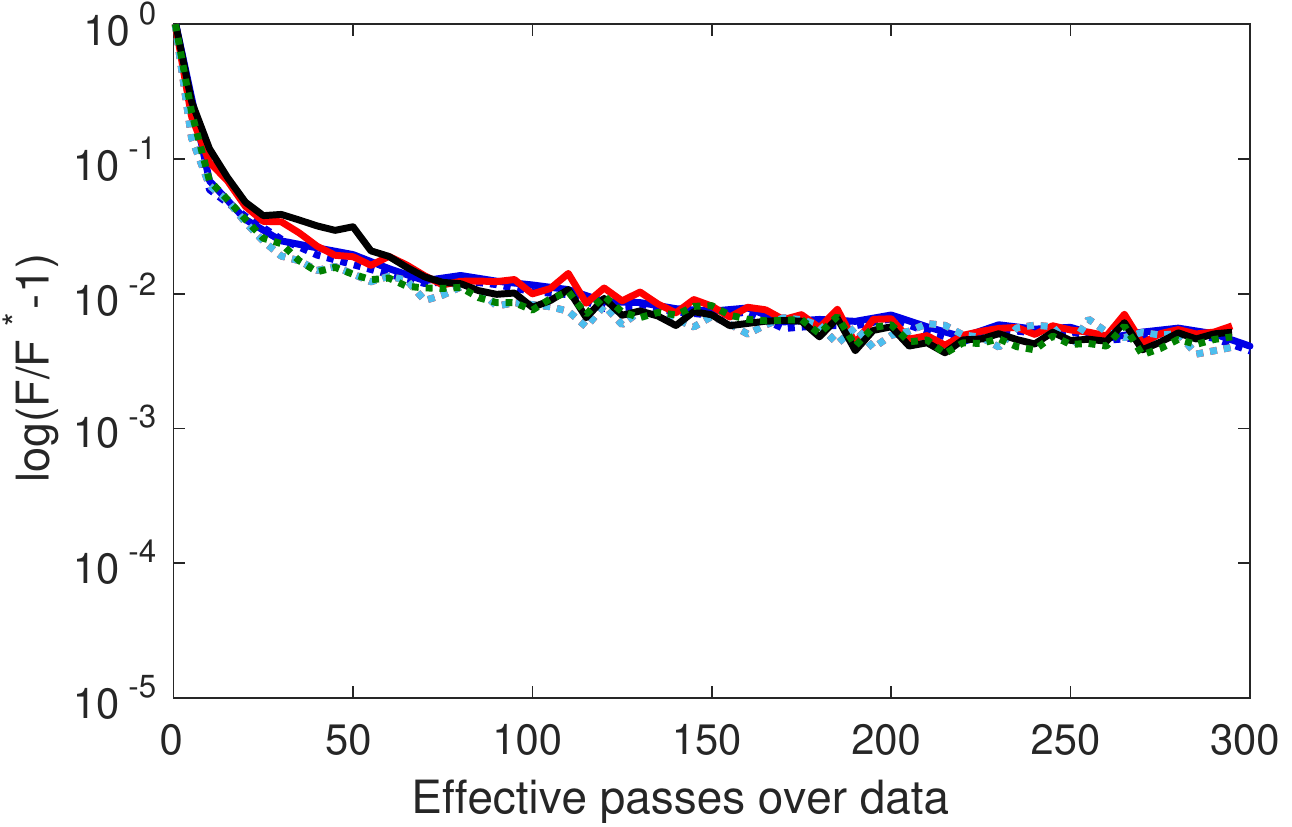}}
\subfloat[Dataset \textrm{ckn-cifar}, $\delta=0.1$]{\includegraphics[width=0.33\linewidth]{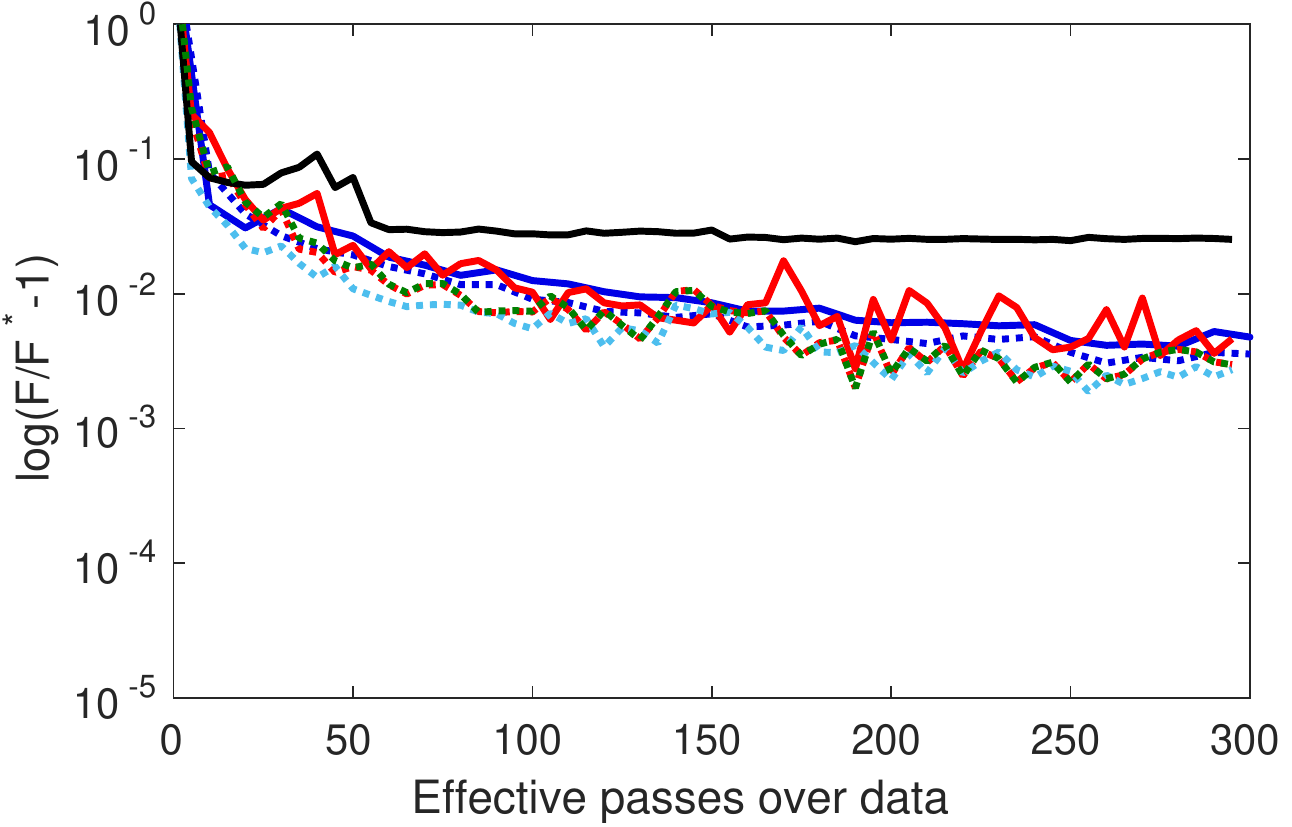}}
\subfloat[Dataset \textrm{alpha}, $\delta=0.1$]{\includegraphics[width=0.33\linewidth]{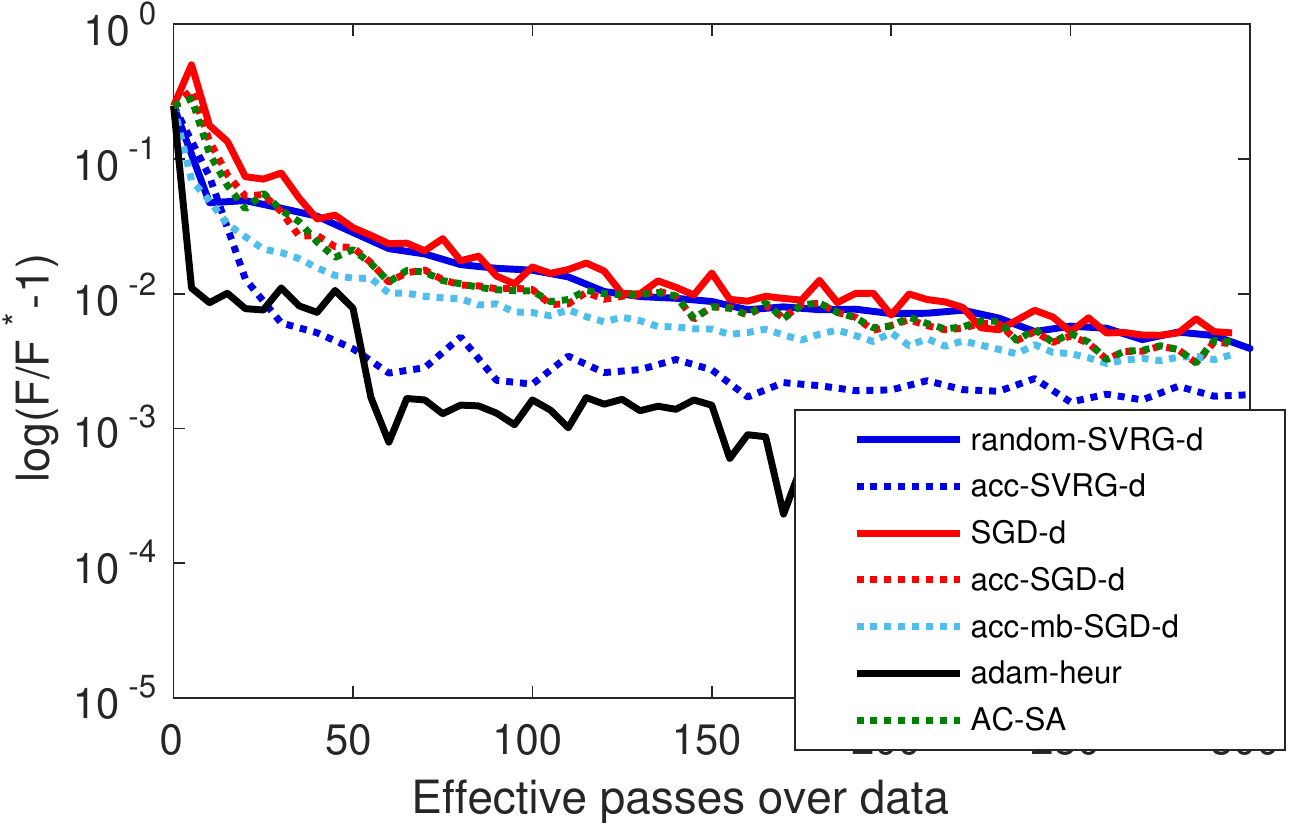}}\\
\caption{Optimization curves with DropOut rate $\delta$ when using the logistic loss and $\lambda=1/10n$. We plot the value of the objective function on a logarithmic scale as a function of the effective passes over the data. Best seen in color by zooming on a computer screen.}\label{fig:dropout}
\end{figure}

\begin{figure}[hbtp]
\centering
\subfloat[Dataset \textrm{gene}, $\delta=0.01$]{\includegraphics[width=0.33\linewidth]{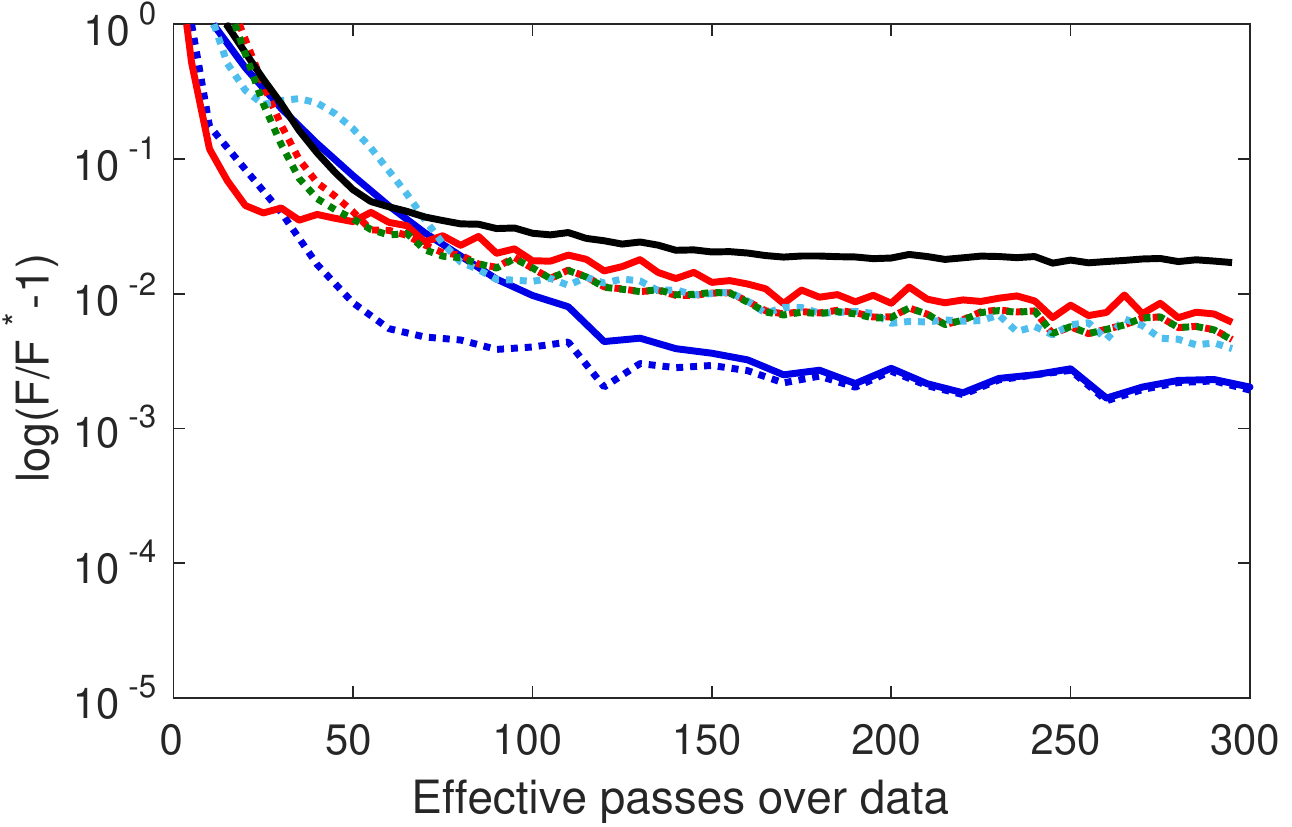}}
\subfloat[Dataset \textrm{ckn-cifar}, $\delta=0.01$]{\includegraphics[width=0.33\linewidth]{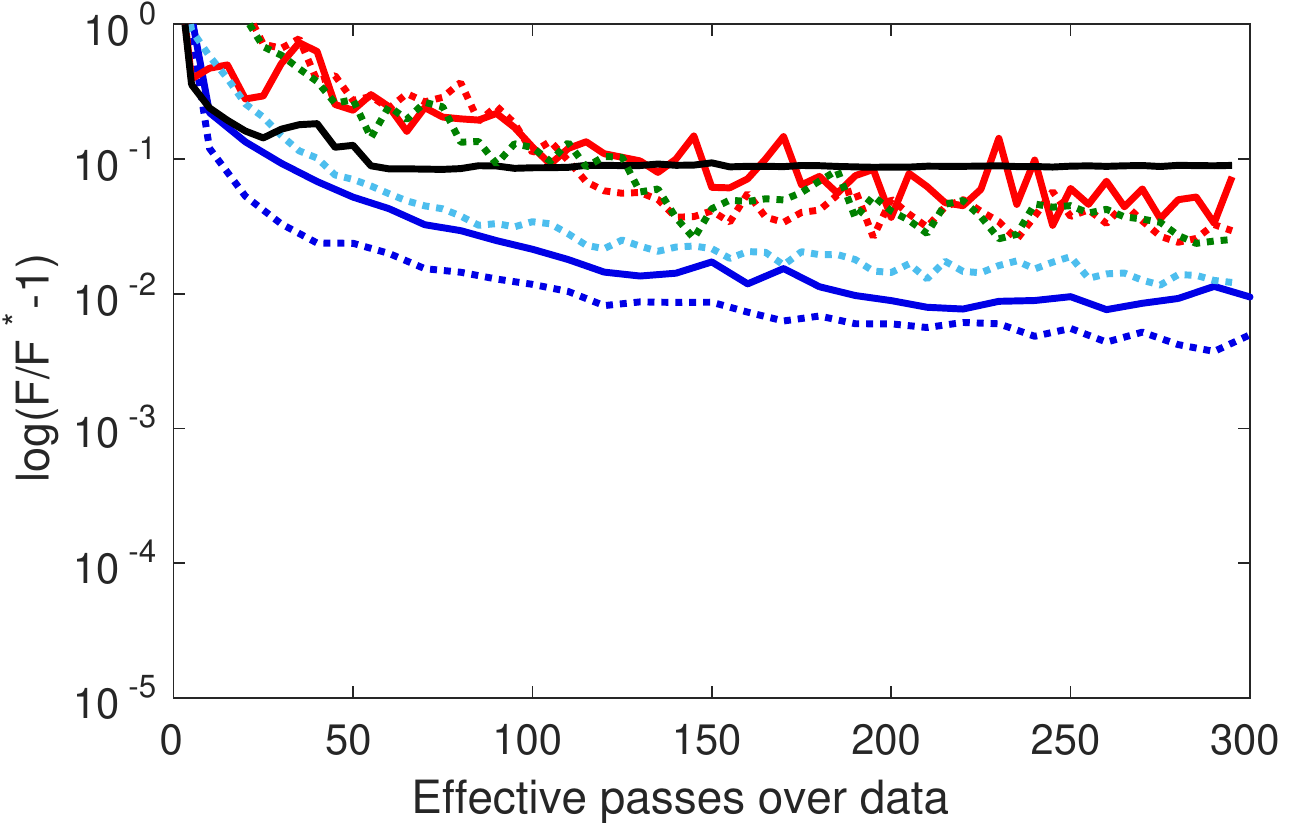}}
\subfloat[Dataset \textrm{alpha}, $\delta=0.01$]{\includegraphics[width=0.33\linewidth]{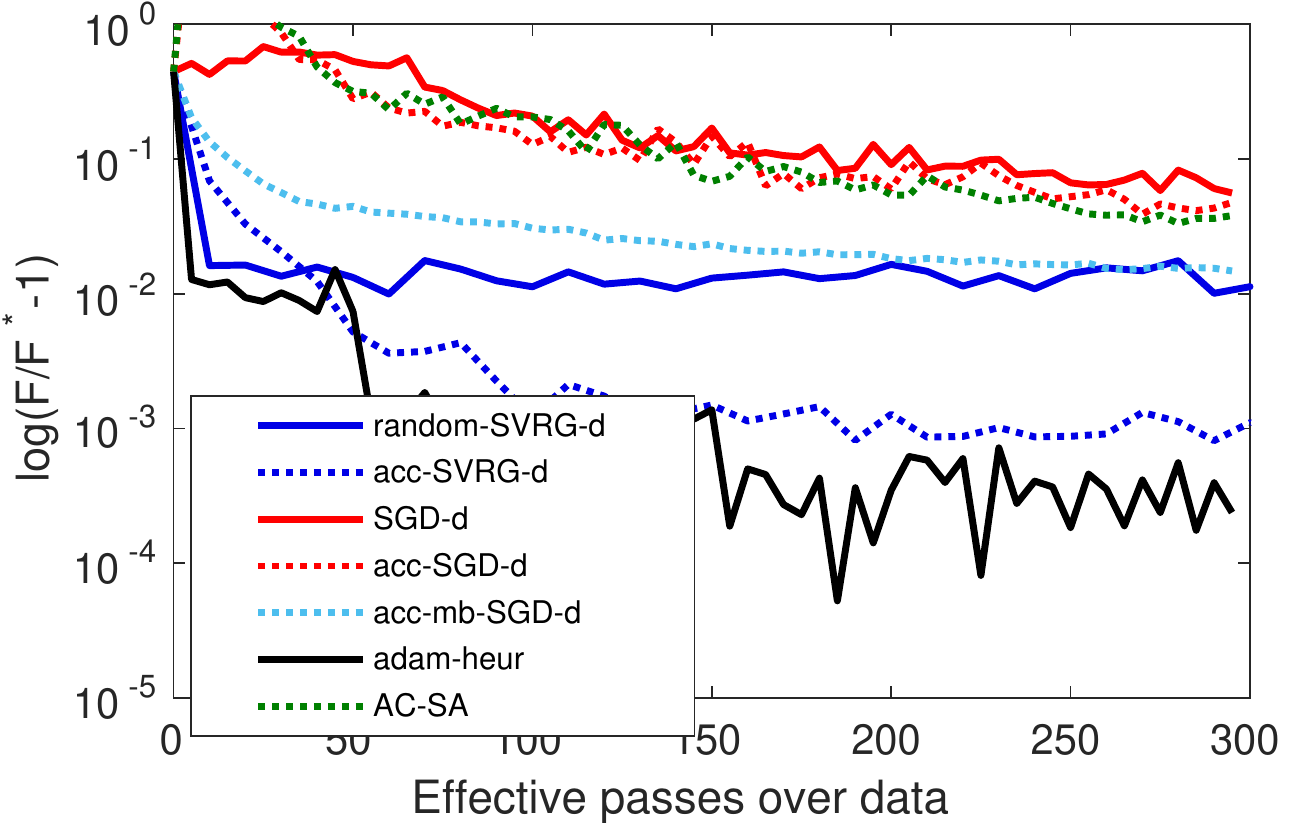}} \\
\subfloat[Dataset \textrm{gene}, $\delta=0.1$]{\includegraphics[width=0.33\linewidth]{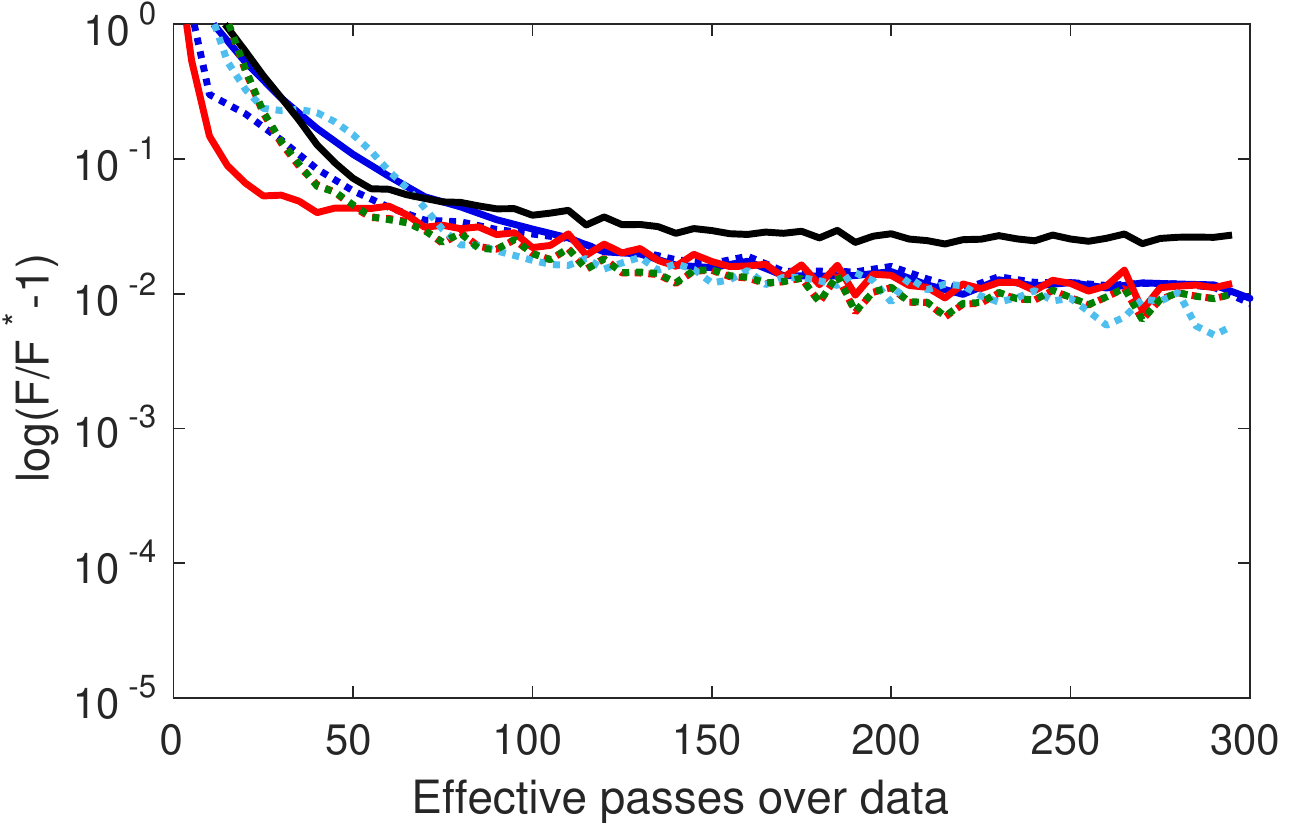}}
\subfloat[Dataset \textrm{ckn-cifar}, $\delta=0.1$]{\includegraphics[width=0.33\linewidth]{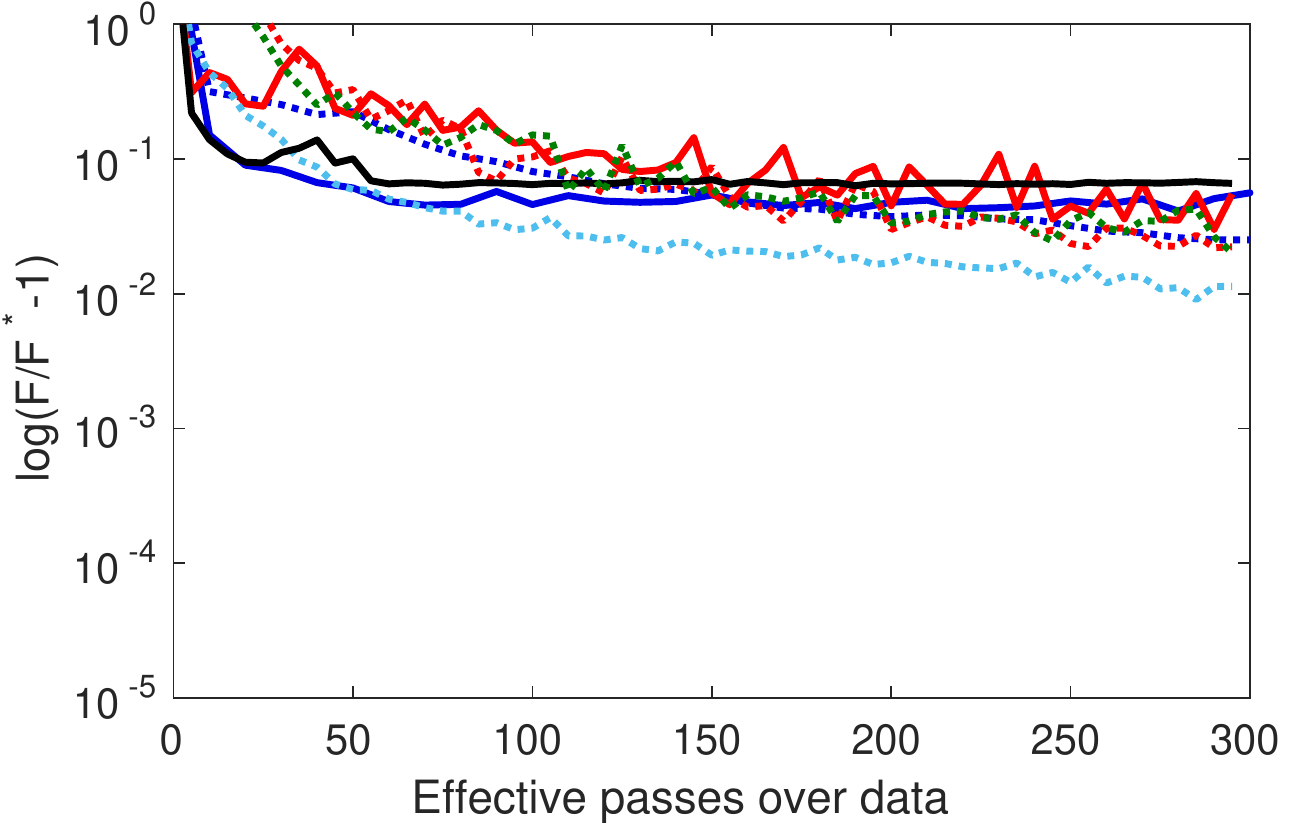}}
\subfloat[Dataset \textrm{alpha}, $\delta=0.1$]{\includegraphics[width=0.33\linewidth]{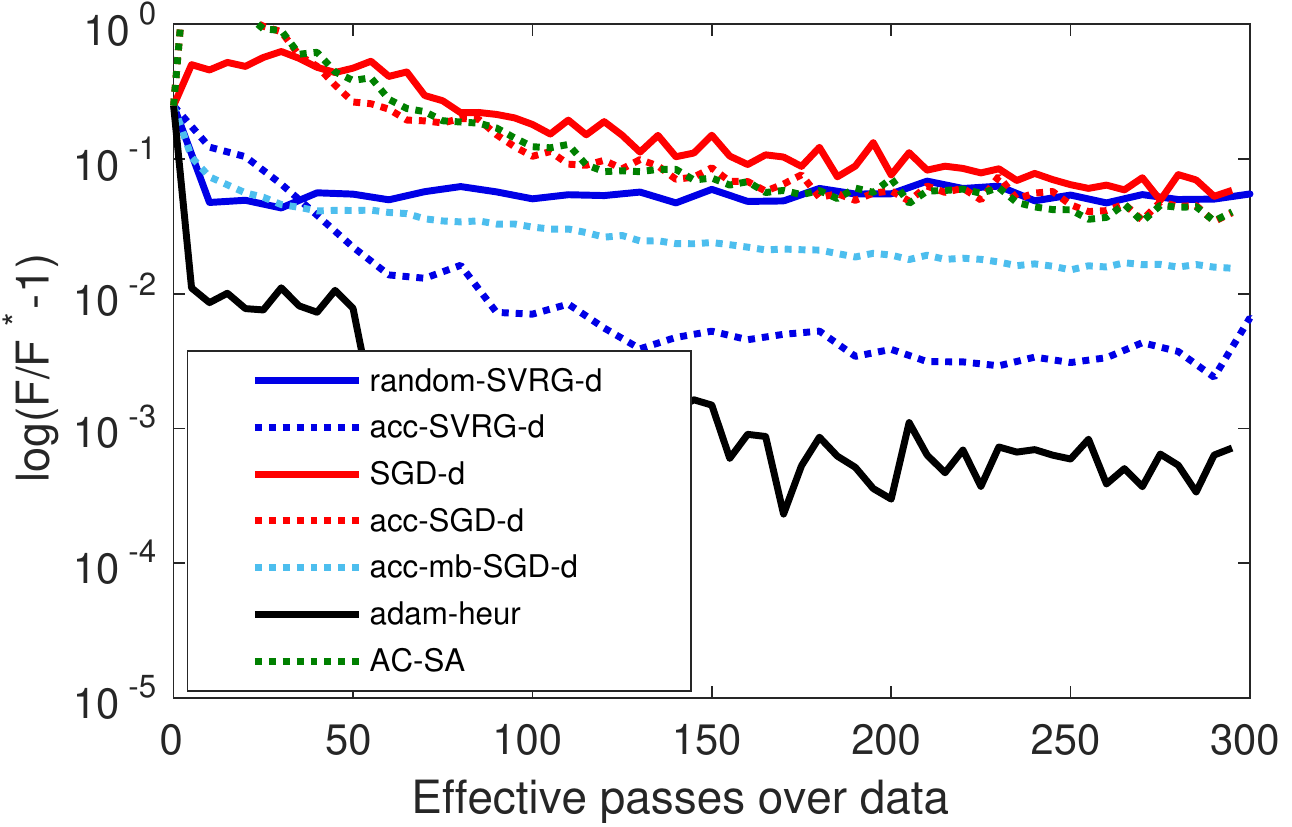}}\\
\caption{Same setting as in Figure~\ref{fig:dropout} but with $\lambda=1/100n$.}\label{fig:dropout2}
\end{figure}

\begin{figure}[hbtp]
\centering
\subfloat[Dataset \textrm{gene}, $\delta=0.01$]{\includegraphics[width=0.33\linewidth]{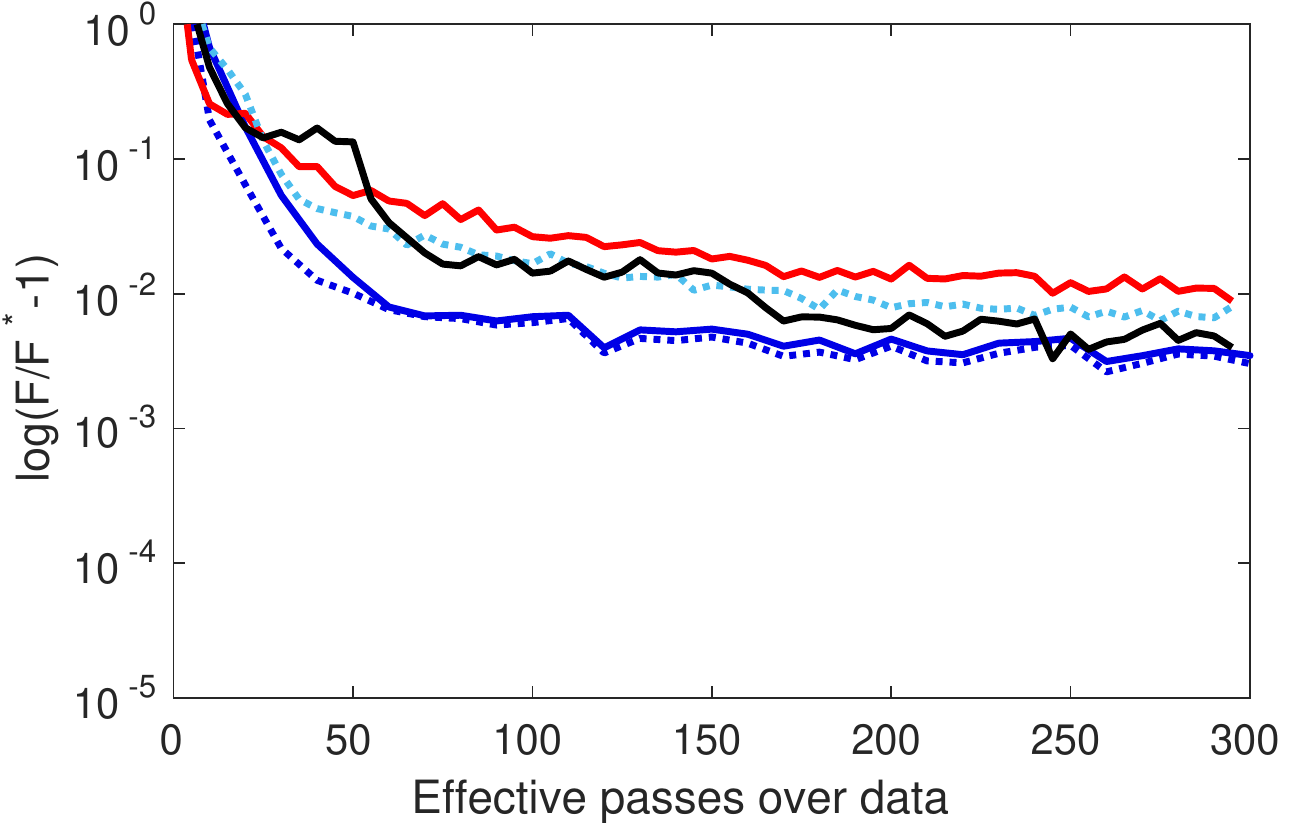}}
\subfloat[Dataset \textrm{ckn-cifar}, $\delta=0.01$]{\includegraphics[width=0.33\linewidth]{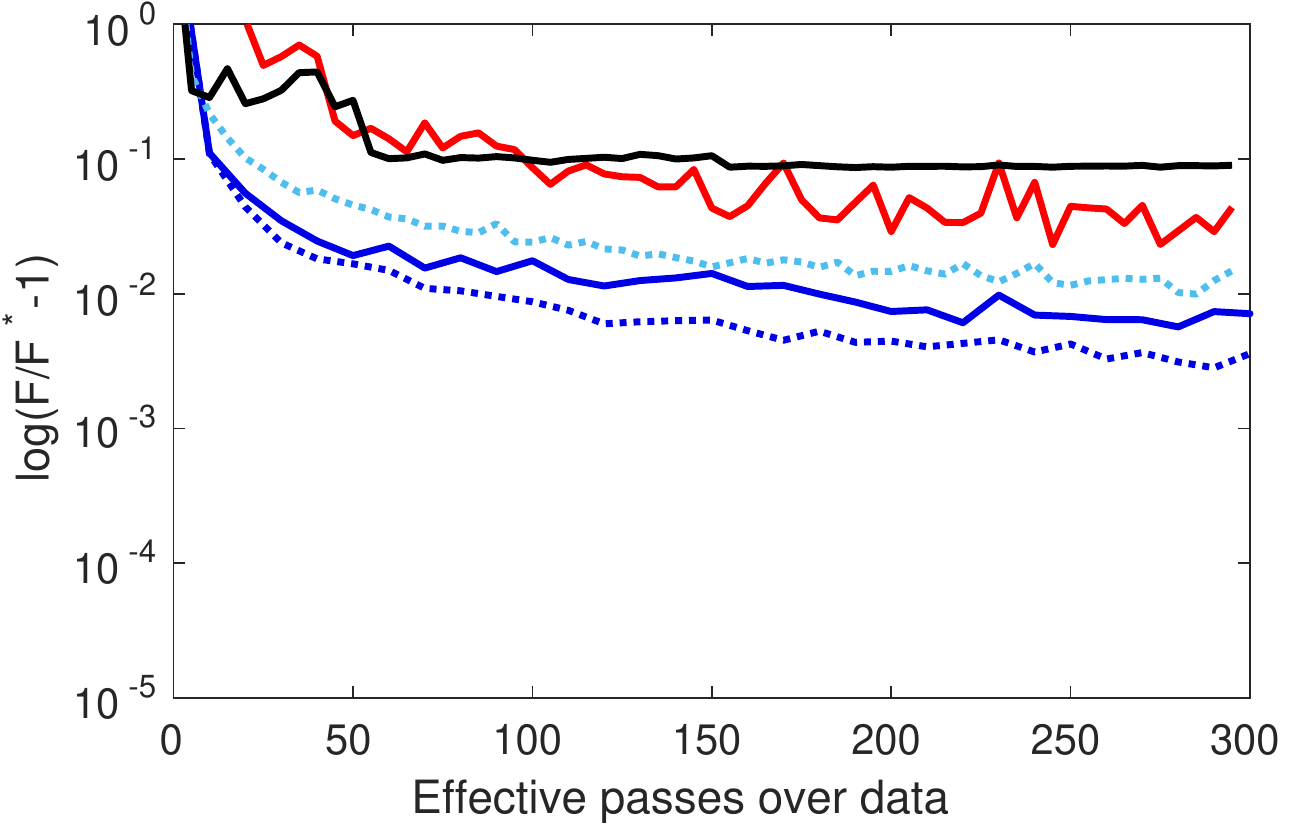}}
\subfloat[Dataset \textrm{alpha}, $\delta=0.01$]{\includegraphics[width=0.33\linewidth]{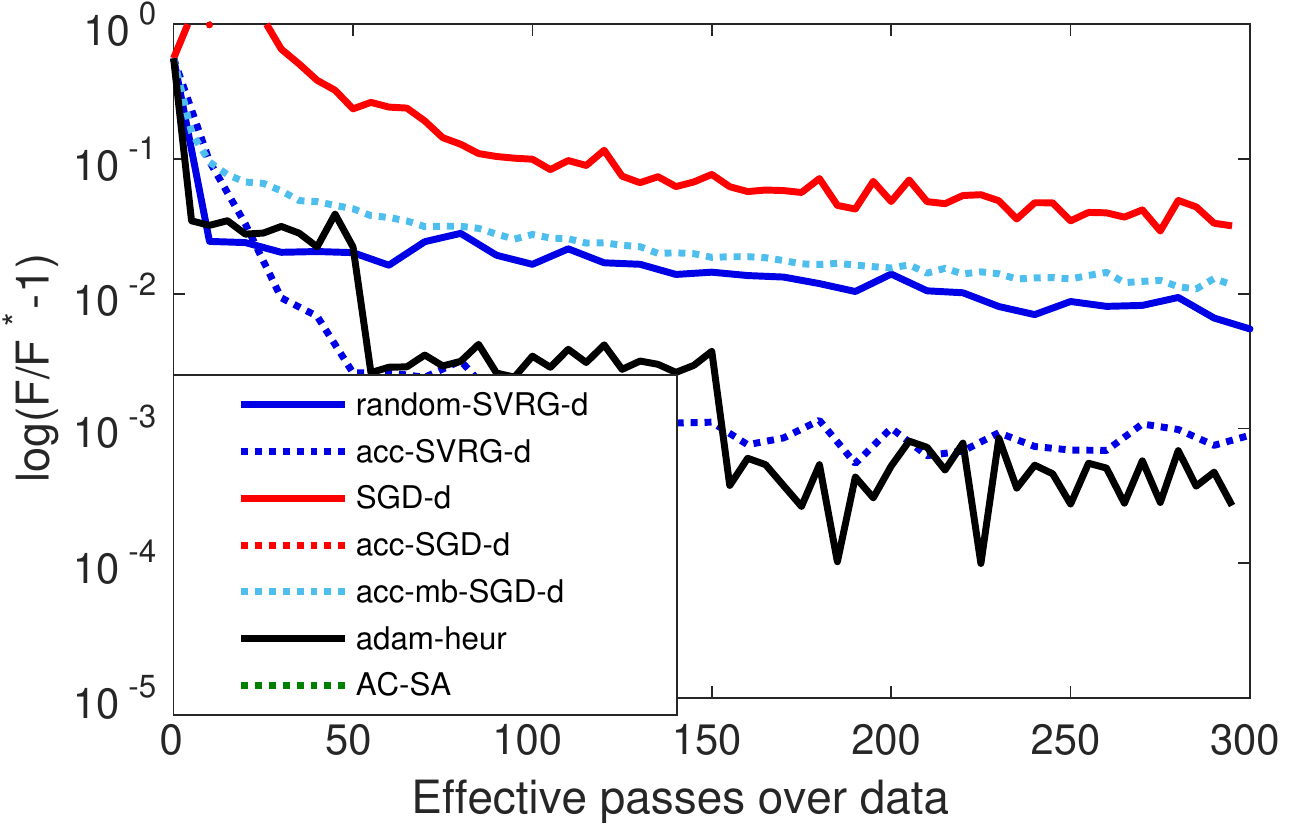}} \\
\subfloat[Dataset \textrm{gene}, $\delta=0.1$]{\includegraphics[width=0.33\linewidth]{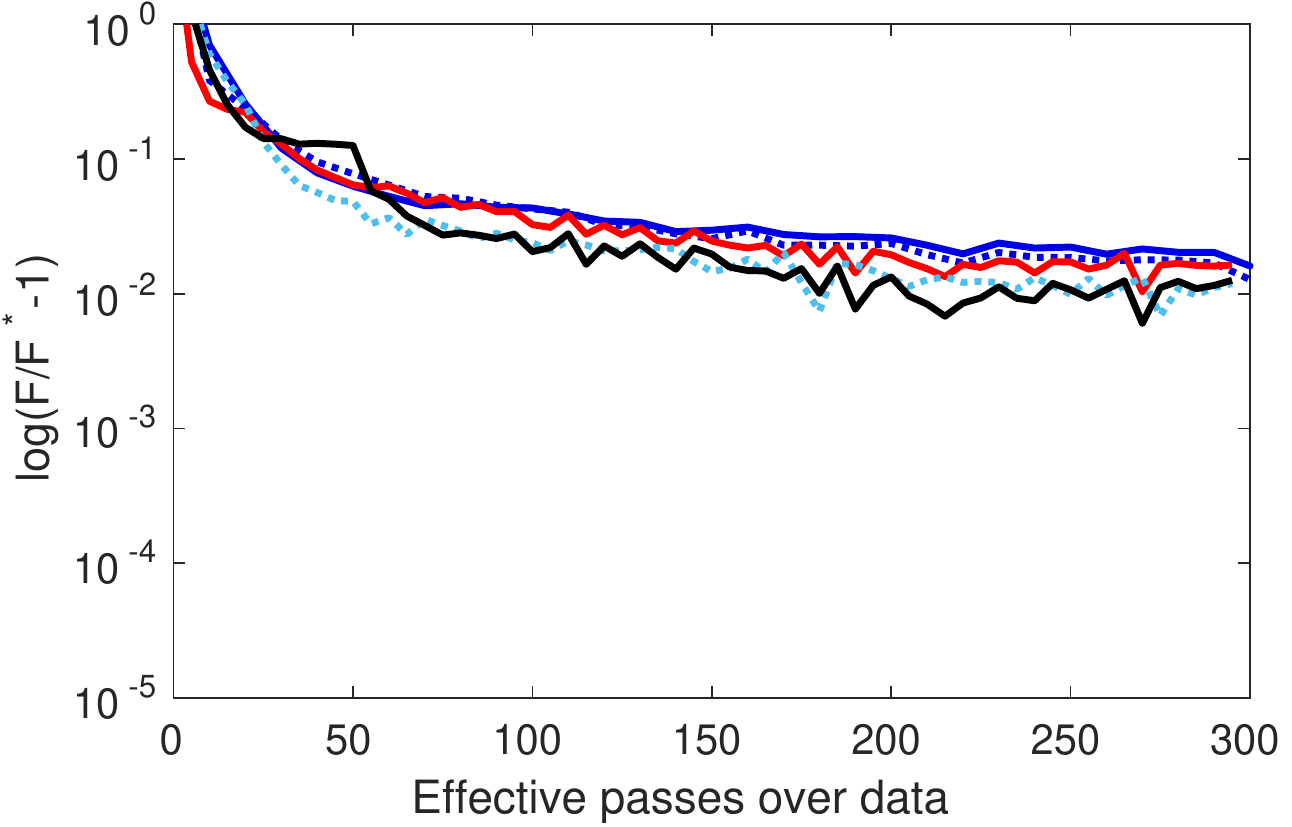}}
\subfloat[Dataset \textrm{ckn-cifar}, $\delta=0.1$]{\includegraphics[width=0.33\linewidth]{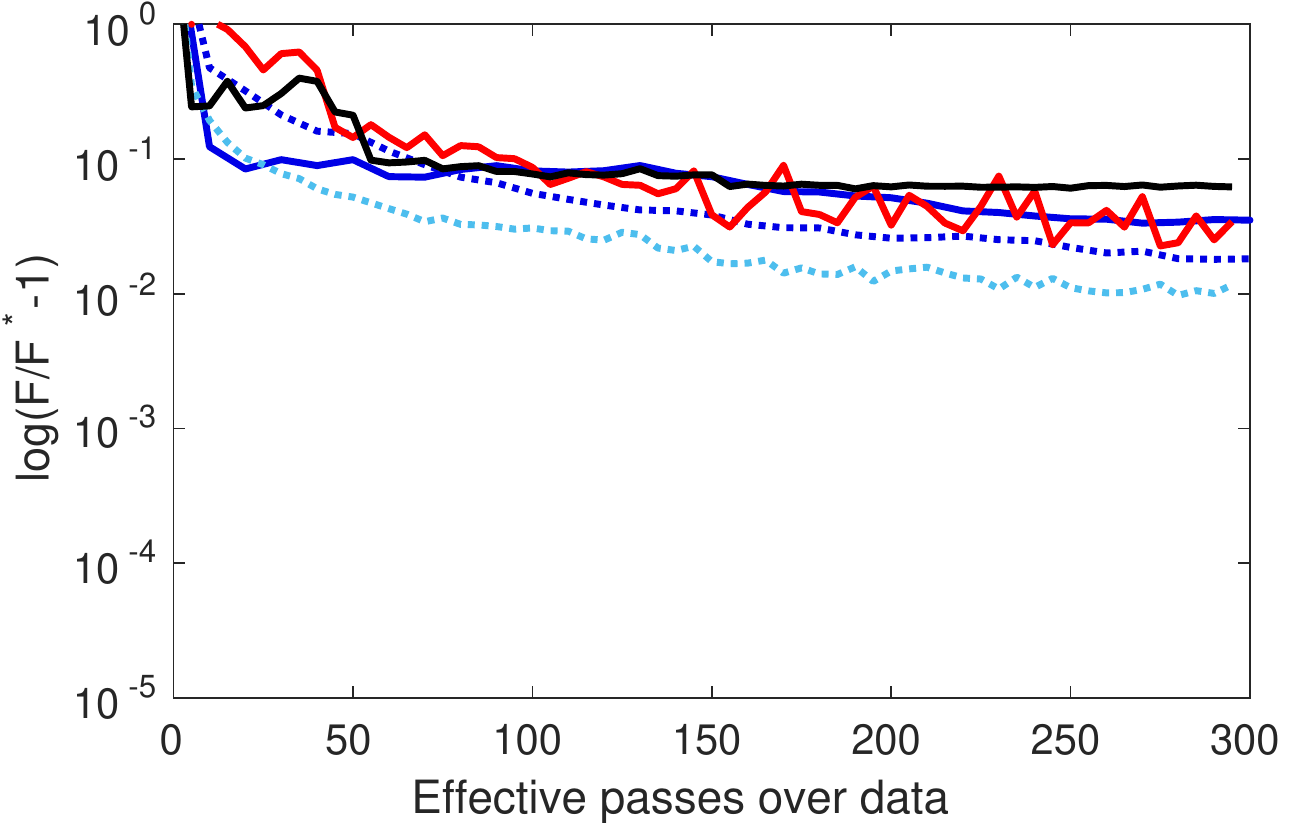}}
\subfloat[Dataset \textrm{alpha}, $\delta=0.1$]{\includegraphics[width=0.33\linewidth]{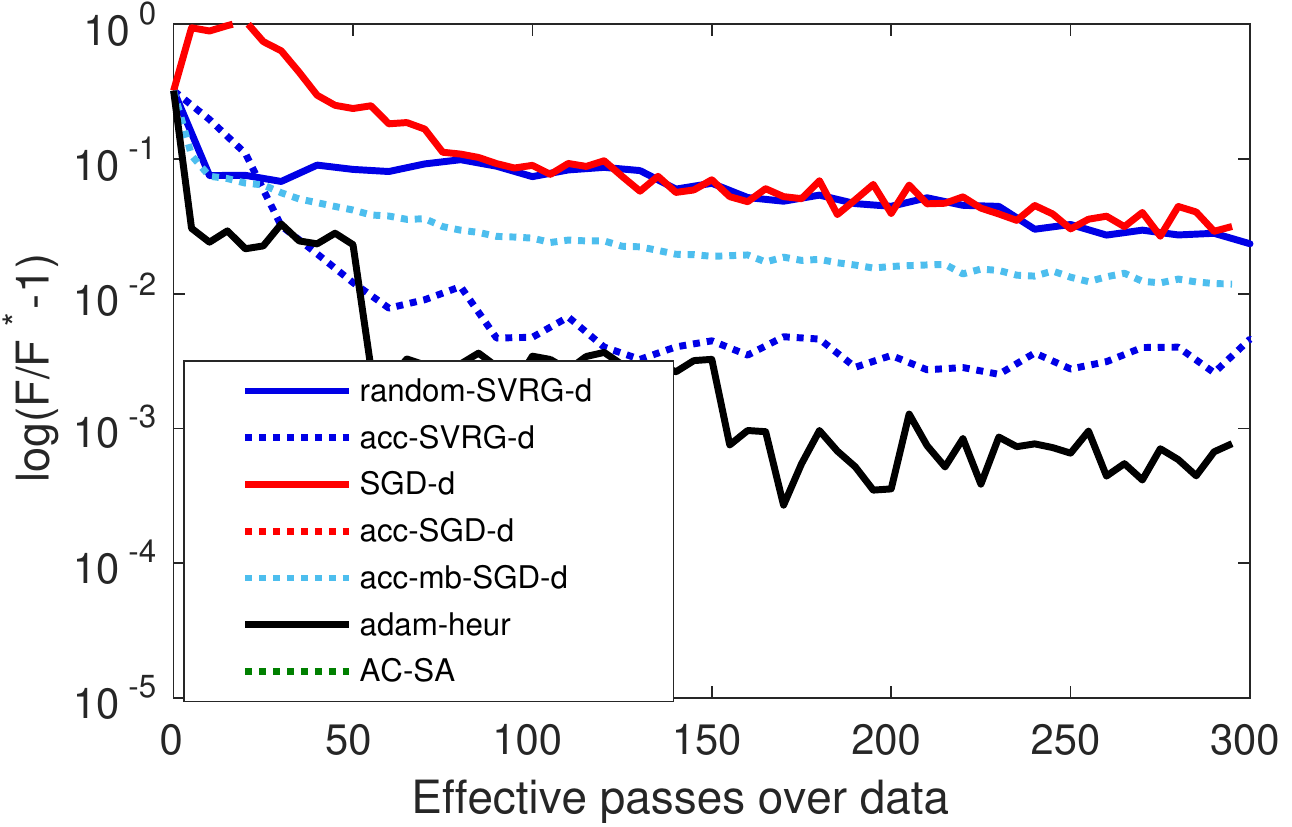}}\\
\caption{Same setting as in Figure~\ref{fig:dropout} but with the squared hinge loss.}\label{fig:dropout3}
\end{figure}

The conclusions of these experiments are the following:
\begin{itemize}
   \item \textrm{accelerated minibatch SGD} performs the best among SGD approaches in general except on \textrm{alpha} where \textrm{Adam} performs best.
\item \textrm{accelerated SVRG} performs better than \textrm{SVRG} in general, or they achieve the same performance. As in the deterministic case, the gains are typically more important in ill-conditioned cases.
\item \textrm{accelerated SVRG} performs better than SGD approaches in the low perturbation regime $\delta=0.01$ and only on the \textrm{alpha} dataset when $\delta=0.1$. Otherwise, the methods perform similarly.
\item not reported on these figures, high perturbation regimes, \eg, $\delta=0.3$ make variance reduction less useful since the noise due to data sampling becomes potentially of the same order as $\tilde{\sigma}^2$; Yet, benefits are still seen on the \textrm{alpha} dataset, whereas SGD approaches perform slightly better than SVRG approaches on \textrm{ckn-cifar} and~\textrm{gene}.
\end{itemize}

\section{Discussion}\label{sec:ccl}
In this paper, we have studied simple stochastic gradient-based rules with
or without variance reduction, and presented an accelerated algorithm dedicated
to finite-sums minimization under the presence of stochastic perturbations. The
approach we propose achieves the classical optimal worst-case
complexities for finite-sum optimization when there is no perturbation~\citep{arjevani2016dimension},
and exhibits an optimal dependency in the noise variance $\tilde{\sigma}^2$ for
convex and strongly convex problems.

Our work is based on stochastic variants of estimate sequences introduced
by~\citet{nesterov1983,nesterov}. The framework leads naturally to many algorithms
with relatively generic proofs of convergence, where convergence is proven
at the same time as the algorithm's design. With iterate averaging
techniques inspired by~\citet{ghadimi2013optimal}, we show that a large class
of variance-reduction stochastic optimization methods can be made robust
to stochastic perturbations. Estimate sequences also naturally lead to
several accelerated algorithms, some of them we did not present in this paper.
For instance, it is possible to show that replacing in~(\ref{eq:lk}) the lower
bound $\psi(x_k)+\psi'(x_k)^\top(x-x_k)$ by $\psi(x)$ itself---in a similar way
as we proceeded to obtain iteration~(\ref{eq:opt2}) from
iteration~(\ref{eq:opt1})---also leads to an accelerated algorithm with
similar guarantees as~(\ref{eq:opt3}).

Possibilities offered by estimate sequences are large, but our framework also
admits a few limitations, paving the way for future work.
In particular, our results are currently limited to Euclidean
metrics---meaning that our convergence rates typically depend on quantities
involving the Euclidean norm (e.g., strong convexity or $L$-smooth inequalities),
and one may expect extensions of our work to other metrics such as Bregman distances.
Estimate sequences admit indeed known extensions to such metrics, and can
also deal with higher-order smoothness assumptions than Lipschitz continuity of
the gradient~\citep{baes2009estimate}---\eg, cubic regularization~\citep{nesterov2006cubic}. We leave such directions for the future.

Another limitation we encountered was the inability to propose robust
accelerated variants of SAGA, MISO, or SDCA based on our stochastic estimate
sequences framework.  To address this problem, 
after the first version of this manuscript was made publicly available, 
we investigated in~\citep{kulunchakov2019generic} a significantly different
approach based on the Catalyst method~\citep{catalyst_jmlr}, allowing us to
accelerate stochastic first-order methods in a generic fashion, at the price of
a logarithmic factor in the optimal complexity---in other words, we were able
to obtain for SAGA, MISO, and SDCA a complexity close to~(\ref{eq:acc_compl}) up to a logarithmic factor in the condition number~$L_Q/\mu$.
We believe that estimate sequences may be useful to obtain the optimal complexity without this logarithmic term, but the construction
would be non-trivial and would rely on a different lower bound than the one we used in Section~\ref{sec:acc}. 

Finally, we note that the optimal complexities we have obtained with diminishing
step-sizes for strongly convex objectives can also be achieved by using instead
a constant step-size combined with mini-batch and restart strategies. As a constant
step-size yields a linear rate of convergence to a noise-dominated
region of radius $O(\tilde{\sigma}^2)$, we can indeed use the restart procedure
described in Section 3 of~\cite{kulunchakov2019generic}, which would yield 
the optimal complexity as well.


\acks{
 This work was supported by the ERC grant SOLARIS (number 714381) and by ANR 3IA MIAI@Grenoble Alpes, (ANR-19-P3IA-0003).
The authors would like to thank Anatoli Juditsky and the anonymous reviewers for interesting discussions that greatly improved the quality of this manuscript.}



\appendix
\section{Useful Mathematical Results}
\subsection{Simple Results about Convexity and Smoothness}
The next three lemmas are classical upper and lower bounds for smooth or strongly convex functions~\cite{nesterov}.
\begin{lemma}[\bfseries Quadratic upper bound for $L$-smooth functions]\label{lemma:upper}
Let $f: \Real^p \to \Real$ be $L$-smooth. Then, for all $x, x'$ in $\Real^p$,
\begin{equation*}
|f(x') - f(x) - \nabla f(x)^\top (x'-x)| \leq \frac{L}{2}\|x-x'\|_2^2.\label{eq:lipschitz}
\end{equation*}
\end{lemma}
\begin{lemma}[\bfseries Lower bound for strongly convex
functions]\label{lemma:lower}
Let $f: \Real^p \to \Real$ be a $\mu$-strongly convex function. Let $z$ be in $\partial f(x)$ for some $x$ in $\Real^p$. Then, the following inequality holds for all $x'$ in $\Real^p$:
\begin{displaymath}
f(x') \geq f(x) + z^\top(x'-x) + \frac{\mu}{2}\|x-x'\|_2^2.
\end{displaymath}
\end{lemma}
\begin{lemma}[\bfseries Second-order growth property]\label{lemma:second}
Let $f: \Real^p \to \Real$ be a $\mu$-strongly convex function and $\Xcal \subseteq \Real^p$ be a convex set.
Let $x^\star$ be the minimizer of $f$ on~$\Xcal$. Then, the following condition holds for all $x$ in $\Xcal$:
\begin{displaymath}
f(x) \geq f(x^\star) + \frac{\mu}{2}\|x-x^\star\|_2^2.
\end{displaymath}
\end{lemma}
\begin{lemma}[\bfseries Useful inequality for smooth and convex functions]\label{lemma:useful}
Consider an \hfill \break $L$-smooth $\mu$-strongly convex function~$f$ defined on $\Real^p$ and a parameter $\beta$ in $[0,\mu]$. Then, for all $x,y$ in $\Real^p$,
\begin{displaymath}
\| \nabla f(x) - \nabla f(y) - \beta(x-y)\|^2 \leq 2L (f(x)-f(y)-\nabla f(y)^\top(x-y)).
\end{displaymath}
\end{lemma}
\begin{proof}
Let us define the function $\phi(x) = f(x) - \frac{\beta}{2}\|x\|^2$, which is $(\mu-\beta)$-strongly convex. It is then easy to show that $\phi$ is $(L-\beta)$-smooth, according to~Theorem 2.1.5 in~\cite{nesterov}: indeed, for all $x, y$ in~$\Real^p$,
\begin{displaymath}
\begin{split}
\phi(x) = f(x) -  \frac{\beta}{2}\|x\|^2 & \leq f(y) + \nabla f(y)^\top(x-y) + \frac{{L}}{2}\|x-y\|^2  -  \frac{\beta}{2}\|x\|^2 \\
& = \phi(y) + \nabla \phi(y)^\top (x-y) + \frac{{L}-\beta}{2}\|x-y\|^2,
\end{split}
\end{displaymath}
and again according to Theorem 2.1.5 of~\cite{nesterov},
\begin{displaymath}
\begin{split}
\left\| \nabla \phi(x) - \nabla \phi(y)\right \|^2 & \leq 2{L}( \phi(x) - \phi(y) - \nabla \phi(y)^\top(x -y)) \\
& = 2{L}\left( f(x) - f(y) - \nabla f(y)^\top(x -y) - \frac{\beta}{2}\| x - y\|^2 \right)  \\
& \leq 2{L}\left( f(x) - f(y) - \nabla f(y)^\top(x -y)\right).
\end{split}
\end{displaymath}

\end{proof}

\subsection{Useful Results to Select Step Sizes}\label{appendix:step}
In this section, we present basic mathematical results regarding the choice of step sizes. The proofs of the first two lemmas are trivial by induction.
\begin{lemma}[\bf Relation between $(\delta_k)_{k \geq 0}$ and $(\Gamma_k=\prod_{t=1}^k(1-\delta_t))_{k \geq 0}$]\label{lemma:step}
Consider the following cases:
\begin{itemize}
\item $\delta_k=\delta$ (constant). Then $\Gamma_k= (1-\delta)^k$;
\item $\delta_k=1/(k+1)$. Then, $\Gamma_k = \delta_k = \frac{1}{(k+1)}$;
\item $\delta_k=2/(k+2)$. Then, $\Gamma_k = \frac{2}{(k+1)(k+2)}$;
\item $\delta_k = \min( 1/(k+1), \delta)$. then,
$$ \Gamma_k = \left\{ \begin{array}{ll}
(1-\delta)^k & ~\text{if}~k < k_0 ~~~\text{with}~~~ k_0 = \left\lceil \frac{1}\delta - 1 \right\rceil \\
\Gamma_{k_0-1} \frac{k_0}{k+1} & ~\text{otherwise}.
\end{array}
\right.$$
\item $\delta_k = \min( 2/(k+2), \delta)$. then,
$$ \Gamma_k = \left\{ \begin{array}{ll}
(1-\delta)^k & ~\text{if}~k < k_0 ~~~\text{with}~~~ k_0 = \left\lceil \frac{2}\delta - 2 \right\rceil \\
\Gamma_{k_0-1} \frac{k_0(k_0+1)}{(k+1)(k+2)} & ~\text{otherwise}.
\end{array}
\right.$$
\end{itemize}
\end{lemma}
\begin{lemma}[\bf Simple relation]\label{lemma:simple}
Consider a sequence of weights $(\delta_k)_{k \geq 0}$ in $(0,1)$. Then,
\begin{equation}
\sum_{t=1}^k \frac{\delta_t}{\Gamma_t} + 1 = \frac{1}{\Gamma_k} \qquad
\text{where} \qquad \Gamma_t \defin \prod_{i=1}^t (1-\delta_i).\label{eq:av}
\end{equation}
\end{lemma}
\begin{lemma}[\bf Convergence rate of $\Gamma_k$]\label{eq:rate_gamma}
Consider the same quantities defined in the previous lemma and consider the sequence $\gamma_k = (1-\delta_k)\gamma_\kmone + \delta_k \mu = \Gamma_k\gamma_0 + (1-\Gamma_k) \mu$ with $\gamma_0 \geq \mu$, and assume the relation $\delta_k=\gamma_k \eta$.
Then, for all $k \geq 0$,
\begin{equation}
\Gamma_k \leq \min \left( \left( 1-  \mu \eta\right)^k ,  \frac{1}{1 + {\gamma_0 \eta k}} \right). \label{eq:rate_gamma2}
\end{equation}
Besides,
\begin{itemize}
\item when $\gamma_0=\mu$, then $\Gamma_k = (1-\mu\eta)^k$.
\item when $\mu=0$, $\Gamma_k = \frac{1}{1 + {\gamma_0 \eta k}}$.
\end{itemize}
\end{lemma}
\begin{proof}
First, we have for all $k$, $\gamma_k \geq \mu$ such that $\delta_k \geq \eta \mu$, which leads then to $\Gamma_k \leq \left(1- \eta{\mu}\right)^k$. Besides,
$\gamma_k \geq \Gamma_k \gamma_0$ and thus $\Gamma_k = (1-\delta_k)\Gamma_\kmone \leq\left(1- {\Gamma_k \gamma_0}\eta\right)\Gamma_\kmone$.
Then,
$ \frac{1}{\Gamma_k} \left(1 - \Gamma_k{\gamma_0}\eta \right)  \geq  \frac{1}{\Gamma_\kmone},$ and
$$\frac{1}{\Gamma_k} \geq \frac{1}{\Gamma_\kmone} + {\gamma_0}\eta \geq 1 + \gamma_0 \eta k,$$ which is sufficient to obtain~(\ref{eq:rate_gamma2}).
Then, the fact that $\gamma_0=\mu$ leads to $\Gamma_k = (1-\mu\eta)^k$ is trivial, and the fact that $\mu=0$ yields $\Gamma_k = \frac{1}{1 + {\gamma_0 \eta k}}$ can be shown by induction.
Indeed, the relation is true for $\Gamma_0$ and then, assuming the relation is true for $k-1$, we have for $k \geq 1$,
\begin{displaymath}
\Gamma_k = (1-\delta_k)\Gamma_\kmone = (1-\eta \gamma_k) \Gamma_\kmone =(1-\eta \gamma_0 \Gamma_k) \Gamma_\kmone \geq \left(1-{\eta \gamma_0}\Gamma_k\right) \frac{1}{1+\gamma_0\eta (\kmone)},
\end{displaymath}
which leads to $\Gamma_k = \frac{1}{1 + {\gamma_0 \eta k}}$.
\end{proof}
\begin{lemma}[\bf Accelerated convergence rate of $\Gamma_k$]\label{eq:acc_rate_gamma}
Consider the same quantities defined in Lemma~\ref{lemma:simple} and consider the sequence $\gamma_k = (1-\delta_k)\gamma_\kmone + \delta_k \mu = \Gamma_k\gamma_0 + (1-\Gamma_k) \mu$ with $\gamma_0 \geq \mu$, and assume the relation $\delta_k=\sqrt{\gamma_k \eta}$.
Then, for all $k \geq 0$,
\begin{equation*}
\Gamma_k \leq \min \left( \left( 1- \sqrt{\mu \eta}\right)^k ,  \frac{4}{(2 + {\sqrt{\gamma_0 \eta} k})^2} \right). 
\end{equation*}
Besides,
when $\gamma_0=\mu$, then $\Gamma_k = (1-\sqrt{\mu\eta})^k$.
\end{lemma}
\begin{proof}
see Lemma 2.2.4 of~\cite{nesterov}.
\end{proof}

\subsection{Averaging Strategies}
\label{appendix:averaging}
Next, we show a generic convergence result and an appropriate averaging strategy given a recursive relation between quantities acting as Lyapunov function.
\begin{lemma}[\bf Averaging strategy]\label{lemma:averaging}
Assume that there is a sequence 
$(x_k)_{k \geq 1}$ generated by an algorithm that minimizes a convex function~$F$, and that there exist non-negative sequences $(T_k)_{k \geq 0}$, $(\delta_k)_{k \geq 1}$ in~$(0,1)$, $(\beta_k)_{k \geq 1}$ and a scalar $\alpha >0$ such that  for all $k \geq 1$,
\begin{equation}
\frac{\delta_k}{\alpha} \E[ F(x_k) - F^\star]  + T_k \leq (1-\delta_k) T_{\kmone} + \beta_{k}, \label{eq:av1}
\end{equation}
where the expectation is taken with respect to any random parameter used by the algorithm.
Then,
\begin{equation}
\E[ F(x_k) - F^\star]  + \frac{\alpha}{\delta_k} T_k \leq  \frac{\alpha \Gamma_k}{\delta_k}\left(T_{0} + \sum_{t=1}^k \frac{\beta_{t}}{\Gamma_t}\right)~~~\text{where}~~~\Gamma_k \defin \prod_{t=1}^k (1-\delta_t). \label{eq:av2}
\end{equation}
\paragraph{Generic averaging strategy.}
   For any point $\hat{x}_0$, consider the averaging sequence $(\hat{x}_k)_{k \geq 0}$, 
\begin{displaymath}
   \hat{x}_k = \Gamma_k \left( \hat{x}_0 + \sum_{t=1}^k \frac{\delta_t}{\Gamma_t}x_t \right) = (1-\delta_k) \hat{x}_{\kmone} + \delta_k x_k~~~\text{(for $k \geq 1$)}, \label{eq:hatxk}
\end{displaymath}
then, 
\begin{equation}
   \E[ F(\hat{x}_k) - F^\star] + \alpha {T_k} \leq \Gamma_k \left(F(\hat{x}_0)-F^\star + \alpha T_0 + \alpha \sum_{t=1}^k \frac{\beta_{t}}{\Gamma_t}\right). \label{eq:av3}
\end{equation}
\paragraph{Uniform averaging strategy.}
   Assume that $\delta_k=\frac{1}{k+1}$ and consider the average sequence $\hat{x}_k=\frac{1}{k}\sum_{i=1}^k x_i$. Then,
\begin{equation}
   \E[ F(\hat{x}_k) - F^\star] + \alpha {T_k} \leq \frac{\alpha}{k} \left( T_0 +  \sum_{t=1}^k \frac{\beta_{t}}{\Gamma_t}\right). \label{eq:av4}
\end{equation}
\end{lemma}
\begin{proof}
Given  that $T_k \leq (1-\delta_k) T_{\kmone} + \beta_{k}$, we obtain~(\ref{eq:av1}) by simply unrolling the recursion.
To analyze the effect of the averaging strategies, divide now~(\ref{eq:av1})
by~$\Gamma_k$:
\begin{displaymath}
\frac{\delta_k}{\alpha \Gamma_k} \E[ F(x_k) - F^\star]  + \frac{T_k}{\Gamma_k} \leq \frac{T_{\kmone}}{\Gamma_{\kmone}} + \frac{\beta_{k}}{\Gamma_k}.
\end{displaymath}
Sum from $t=1$ to~$k$ and notice that we have a telescopic sum:
\begin{equation}
\frac{1}{\alpha} \sum_{t=1}^k \frac{\delta_t}{\Gamma_t} \E[ F(x_t) - F^\star]  + \frac{T_k}{\Gamma_k} \leq T_0 + \sum_{t=1}^k \frac{\beta_{t}}{\Gamma_t}. \label{eq:aux_averaging}
\end{equation}
   Then, add $(1/\alpha)\E[F(\hat{x}_0)-F^\star]$ on both sides and multiply by $\alpha \Gamma_k$:
\begin{displaymath}
   \sum_{t=1}^k \frac{\delta_t\Gamma_k}{\Gamma_t} \E[ F(x_t) - F^\star]  + \Gamma_k \E[F(\hat{x}_0)-F^\star] + \alpha {T_k} \leq \Gamma_k \left(\alpha T_0 + \E[F(\hat{x}_0)-F^\star] + \alpha \sum_{t=1}^k \frac{\beta_{t}}{\Gamma_t}\right).
\end{displaymath}
By exploiting the relation~(\ref{eq:av}), we may then use Jensen's inequality and we obtain~(\ref{eq:av3}).

   Consider now the specific case $\delta_k=\frac{1}{k+1}$, which yields $\Gamma_k=\frac{1}{k+1}$.
   Multiply then Eq.~(\ref{eq:aux_averaging}) by $\alpha/k$ and use Jensen's inequality; we obtain Eq.~(\ref{eq:av4}).
\end{proof}

\section{Relation Between Iteration~(\ref{eq:opt2}) and MISO/SDCA}
\label{appendix:miso.sdca}

In this section, we derive explicit links between the proximal MISO
algorithm~\cite{catalyst}, a primal version of
SDCA~\cite{ShalevShwartz2016SDCAWD}, and iteration~(\ref{eq:opt2}) when used
with the gradient estimator~(\ref{eq:gk2}) without stochastic perturbations.
Under the big data condition $L/\mu \leq n$, consider indeed $\beta=\mu$, constant step-sizes $\eta_k=\eta = \frac{1}{n\mu}$, $\gamma_k=\mu$, and a uniform sampling distribution~$Q$; then, we obtain the following algorithm
\begin{equation*}
   \begin{split}
      \bar{x}_{k}  & \leftarrow (1-\mu \eta)\bar{x}_\kmone + \mu \eta x_{\kmone} - \eta \left(\nabla f_{i_k}(x_{k-1}) - z_{k-1}^{i_k} + \bar{z}_{k-1}    \right) ~~~\text{and}~~~x_k =\text{Prox}_{\frac{\psi}{\mu}}\left[ \bar{x}_k \right] \\
      \bar{z}_k & = \bar{z}_{k-1} + \frac{1}{n}(z_k^{i_k} - z_{k-1}^{i_k})~~~~~\text{and}~~~~~ z_k^{i_k} = \nabla f_{i_k}(x_{k-1})  - \mu x_{k-1}, 
   \end{split}
\end{equation*}
with $\bar{z}_0 = \bar{x}_0 = 0$. Then, since $\mu \eta = \frac{1}{n}$, it is easy to show that in fact $\bar{z}_k = \mu\bar{x}_k$ for all $k \geq 0$.
This is then exactly the proximal MISO algorithm~\citep[see][]{smiso}.
For the relation between primal variants of SDCA and MISO, see page 4 and Equation~(3) of~\citet{smiso}.

\section{Recovering Classical Results for Proximal SGD}
\label{app.rec}

In this section, we present several corollaries of Theorem~\ref{thm:conv} to recover classical results for proximal variants of the stochastic gradient descent method.
Throughout the section, we assume that the gradient estimates have variance bounded by $\sigma^2$:
$$\omega_k^2 = \E[ \|g_k - \nabla f(x_\kmone)\|^2]  \leq \sigma^2.$$
Convergence results for the deterministic case~$\sa\sq=0$ can be also recovered naturally from the corollaries.
We start by applying Theorem~\ref{thm:conv} with a constant step-size strategy $\eta_k=1/L$, which shows convergence to a noise-dominated region of radius $\sigma^2/L$.
In all the corollaries below, we use the notation from Theorem~\ref{thm:conv}.
\begin{corollary}[\bf Proximal variants of SGD with constant step-size, $\mu > 0$]\label{corollary:sgd_constant}
   Assume \newline that~$f$ is $\mu$-strongly convex, choose $\gamma_0=\mu$ and $\eta_k = 1/L$ with Algorithm~(\ref{eq:opt1}) or~(\ref{eq:opt2}). Then, for any point $\hat{x}_0$,
\begin{equation}
   \E\left[F(\hat{x}_k)- F^\star + d_k(x^\star) - d_k^\star\right] \leq   \left(1-\frac{\mu}{L}\right)^k\left(F(\hat{x}_0)- F^\star + d_0(x^\star) - d_0^\star\right) + \frac{\sigma^2}{L},\label{eq:sgd}
\end{equation}
   when using the averaging strategy from Theorem~\ref{thm:conv}. Note that $d_k(x^\star)-d_k^\star \geq \frac{\mu}{2}\|x_k-x^\star\|^2$ for all $k \geq 0$ with equality for Algorithm~(\ref{eq:opt1}).
\end{corollary}
Next, we show how to obtain converging algorithms by using decreasing step
sizes.
\begin{corollary}[\bf Proximal variants of SGD with decreasing step-sizes, $\mu > 0$] \label{corollary:sgd}~\newline
Assume that $f$ is $\mu$-strongly convex and that we target an accuracy
   $\varepsilon$ smaller than $2\sigma^2/L$. First, use a constant step-size $\eta_k=1/L$ with $\gamma_0=\mu$ within Algorithm~(\ref{eq:opt1}) or~(\ref{eq:opt2}), using $\hat{x}_0=x_0$,
   leading to the convergence rate~(\ref{eq:sgd}),  until $\E[F(\hat{x}_k)- F^\star + d_k(x^\star)-d_k^\star] \leq  2 \sigma^2/L$.
   Then, we restart the optimization procedure, using the previously obtained $\hat{x}_k, x_k$ as new initial points, with decreasing step-sizes $\eta_k
= \min \left(\frac{1}{L},\frac{2}{\mu (k+2)}\right)$,  and generate new sequences~$(\hat{x}_k', x_k')_{k \geq 0}$. The total number of iterations to achieve $\E[F(\hat{x}_k')- F^\star] \leq  \varepsilon$ is upper bounded by
   \begin{equation}
      O\left( \frac{L}{\mu} \log\left(\frac{F(x_0)- F^\star + d_0(x^\star)-d_0^\star}{\varepsilon}\right)\right) + O\left( \frac{\sigma^2}{\mu \varepsilon}\right). \label{eq:sgd:aux}
   \end{equation}
   Note that $d_0(x^\star)-d_0^\star= \frac{\mu}{2}\|x_0-x^\star\|^2 \leq F(x_0)-F^\star$ for Algorithm~(\ref{eq:opt1}).
\end{corollary}

\begin{proof}
   Given the linear convergence rate~(\ref{eq:sgd}), the number of iterations of the first the constant step-size strategy is upper bounded by the left term of~(\ref{eq:sgd:aux}).
   Then, after restarting the algorithm, we may apply Theorem~\ref{thm:conv} with $\E[F(\hat{x}_0)-F^\star + d_0(x^\star)-d_0^\star] \leq 2 \sigma^2/L$.
With $\gamma_0=\mu$, we have $\gamma_k=\mu$ for all $k \geq 0$, and
the rate of $\Gamma_k$ is given by Lemma~\ref{lemma:step}, which yields for $k \geq k_0 = \left\lceil \frac{2L}{\mu} - 2 \right\rceil$,
\begin{equation*}
\begin{split}
   \E[F(\hat{x}_k')- F^\star]  & \leq \Gamma_k\left( \frac{2\sigma^2}{L} + \sigma^2 \sum_{t=1}^k \frac{\delta_t \eta_t}{\Gamma_t} \right) \\
& =  \Gamma_k \left( \frac{2 \sigma^2}{L} + \frac{\sigma^2}{L} \sum_{t=1}^{k_0-1} \frac{\delta_t}{\Gamma_t} + \sigma^2 \sum_{t=k_0}^{k} \frac{2\delta_t}{\Gamma_t \mu(t+2)}\right) \\
& =  \frac{k_0(k_0+1)}{(k+1)(k+2)} \left( \Gamma_{k_0-1}\frac{2 \sigma^2}{L} + \frac{\sigma^2}{L} \Gamma_{k_0-1}\sum_{t=1}^{k_0-1} \frac{\delta_t}{\Gamma_t}\right) + \sigma^2 \sum_{t=k_0}^{k} \frac{2\delta_t \Gamma_k}{\Gamma_t \mu(t+2)} \\
& =  \frac{k_0(k_0+1)}{(k+1)(k+2)} \left( \Gamma_{k_0-1}\frac{2 \sigma^2}{L} + (1-\Gamma_{k_0 -1 })\frac{\sigma^2}{L}\right) + \sigma^2 \sum_{t=k_0}^{k} \frac{2\delta_t \Gamma_k}{\Gamma_t \mu(t+2)} \\
& \leq  \frac{k_0(k_0+1)}{(k+1)(k+2)} \frac{2 \sigma^2}{L} +  \sigma^2 \frac{1}{(k+1)(k+2)} \left(\sum_{t=k_0+1}^{k} \frac{4 (t+1)(t+2)}{\mu(t+2)^2}\right) \\
& \leq  \frac{k_0}{(k+1)(k+2)} \frac{4 \sigma^2}{\mu} +   \frac{4\sigma^2}{\mu(k+2)},
\end{split}
\end{equation*}
where the second inequality uses the fact that $\frac{\mu}{2}\|x_0-x^\star\|^2 \leq F(x_0)-F^\star \leq \frac{2 \sigma^2}{L}$, and then we use Lemmas~\ref{lemma:step} and~\ref{lemma:simple}. The term on the right is of order $O(\sigma^2/\mu k)$ whereas the term on the left becomes of the same order or smaller whenever $k \geq k_0 = O(L/\mu)$.
This leads to the desired iteration complexity.
\end{proof}
We may now study the case $\mu=0$, first with a constant step size.
   The next corollary consists of simply applying the uniform averaging strategy of Lemma~\ref{lemma:averaging} to Proposition~\ref{prop:keyprop}, noting that $\delta_k=\frac{1}{k+1}$ for all $k \geq 0$ if $\mu=0$ and $\gamma_0=1/\eta$.
\begin{corollary}[\bf Proximal variants of SGD with constant step size, $\mu = 0$]\label{corollary:convex}
   Assume \newline that $f$ is convex, choose a constant step size $\eta_k = \eta \leq \frac{1}{L}$ with  Algorithm~(\ref{eq:opt1}) or~(\ref{eq:opt2}) with $\gamma_0=1/\eta$. 
Then,
\begin{equation}
\E\left[F(\hat{x}_k)- F^\star\right] 
   \leq   \frac{d_0(x^\star)-d_0^\star}{k}+ \eta{\sigma^2}, 
\label{eq:convex}
\end{equation}
   where $\hat{x}_k = \frac{1}{k}\sum_{i=1}^k x_i$. Note that $d_0(x^\star)-d_0^\star=\frac{1}{2\eta}\|x_0-x^\star\|^2$ for Algorithm~(\ref{eq:opt1}).
\end{corollary}

%
The noise dependency is now illustrated for Algorithm~(\ref{eq:opt1}) in the next corollary, obtained in a finite horizon setting.
\begin{corollary}[\bf Proximal variants of SGD with $\mu = 0$, finite horizon] \label{corollary:sgd2}
Consider the same setting as in the previous corollary. Assume that we have a budget of $K$ iterations for Algorithm~(\ref{eq:opt1}).
Choose a constant step size
$$\eta_k = \min\left(\frac{1}{L}, \sqrt{\frac{T_0}{K \sigma^2}}\right) ~~~\text{with}~~~ 
T_0 = 
\dr12 \|x_0-x^\star\|^2.$$
   Then, with $\gamma_0=1/\eta$ and when using the averaging strategy from Corollary~\ref{corollary:convex},
\begin{equation}
\E[F(\hat{x}_K)-F^\star] \leq \frac{LT_0}{K}  + 2\sigma \sqrt{\frac{T_0}{K}}. \label{eq:sgd_simple}
\end{equation}
\end{corollary}
This corollary is obtained by optimizing the right side of~(\ref{eq:convex})
with respect to $\eta$ under the constraint $\eta \leq 1/L$.
Considering both cases $\eta = 1/L$ and $\eta = \sqrt{{T_0}/{K \sigma^2}}$, it is easy to check that we have~(\ref{eq:sgd_simple}) in all cases.
Whereas this last result is not a practical one since the step size depends on
unknown quantities, it shows that our analysis is nevertheless able to recover
the optimal noise-dependency in $O(\sigma \sqrt{{T_0}/{K}})$,~\citep[see][]{nemirovski}.

\section{Proofs of the Main Results}
\label{appendix:applications}

\subsection{Proof of Proposition~\ref{prop:nonu}}
\label{appendix:prop.nonu}

\begin{proof}
The proof borrows a large part of the analysis of~\citet{proxsvrg} for
controlling the variance of the gradient estimate in the SVRG algorithm.
First, we note that all the gradient estimators we consider may be written in the generic form~(\ref{eq:gk2}), with $\beta=0$ for SAGA or SVRG.
Then, we will write $\tildenabla f_{i_k}(x_\kmone) = \nabla f_{i_k}(x_\kmone) + \zeta_k$, where $\zeta_k$ is a zero-mean variable with variance $\tilde{\sigma}^2$ drawn at iteration $k$,
and $z_k^i = u_k^i + \zeta_k^i$ for all $k,i$, where~$\zeta_k^i$ has zero-mean with variance $\tilde{\sigma}^2$ and was drawn during the previous iterations.
   Let us  denote by $\omega_k^2=\E[\|g_k-f(x_\kmone)\|^2]$ and let us introduce the quantity $A_k =\E\left[ \frac{1}{(q_{i_k} n)^2} \|\zeta_k\|^2 \right]$.
   Then,
\begin{align}
\nonumber
\omega_k^2 & = \E\left\| \frac{1}{q_{i_k} n}(\tildenabla f_{i_k}(x_{\kmone}) - \beta x_\kmone - z^{i_k}_{\kmone}) + \bar{z}_{\kmone} + \beta x_\kmone - \nabla f(x_\kmone)\right\|^2 \\\nonumber
& = \E\left\| \frac{1}{q_{i_k} n} (\nabla f_{i_k}(x_{\kmone}) - \beta x_\kmone - z^{i_k}_{\kmone}) + \bar{z}^{\kmone} + \beta x_\kmone - \nabla f(x_\kmone) \right\|^2  + \E\left[ \frac{1}{(q_{i_k} n)^2} \|\zeta_k\|^2 \right] \\\nonumber
& \leq \E\left\| \frac{1}{q_{i_k} n}(\nabla f_{i_k}(x_{\kmone}) -\beta x_{\kmone} - z^{i_k}_{\kmone})\right\|^2 +  A_k  \\\nonumber
& = \frac{1}{n}\sum_{i=1}^n\frac{1}{q_i n}\E\left[\|\nabla f_{i}(x_{\kmone}) -\beta x_{\kmone} - z_\kmone^i \|^2\right] + A_k \\\nonumber
& = \frac{1}{n}\sum_{i=1}^n\frac{1}{q_i n}\E\left[\|\nabla f_{i}(x_{\kmone}) -\beta x_{\kmone} - u_\star^i + u_\star^i - z_\kmone^i \|^2\right] + A_k \\\nonumber
& \leq \frac{2}{n}\sum_{i=1}^n\frac{1}{q_i n}\E\left[\|\nabla f_{i}(x_{\kmone}) -\beta x_{\kmone} - u_\star^i \|^2 \right] +   \frac{2}{n}\sum_{i=1}^n\frac{1}{q_i n}\E\left[\|z_\kmone^i - u_\star^i  \|^2\right] + A_k \\\nonumber
& \leq \frac{2}{n}\sum_{i=1}^n\frac{1}{q_i n}\E\left[\|\nabla f_{i}(x_{\kmone})\!-\! \nabla f_i(x^\star) \!-\!\beta (x_{\kmone}\!-\!x^\star)\|^2 \right] \!+\!   \frac{2}{n}\sum_{i=1}^n\frac{1}{q_i n}\E\left[\|u_\kmone^i \!-\! u_\star^i\|^2\right] + 3A_k \\\nonumber
& \leq \frac{4}{n}\sum_{i=1}^n\frac{L_i}{q_i n}\E\left[f_{i}(x_{\kmone})\!-\! f_i(x^\star) \!-\! \nabla f_i(x^\star)^\top (x_\kmone \!-\! x^\star) \right] \!+\!   \frac{2}{n}\sum_{i=1}^n\frac{1}{q_i n}\E\left[\|u_\kmone^i \!-\! u_\star^i\|^2\right] \!+\! 3 A_k \\
& \leq {4L_Q}\E\left[f(x_{\kmone})- f(x^\star) - \nabla f(x^\star)^\top (x_\kmone - x^\star) \right] +   \frac{2}{n}\sum_{i=1}^n\frac{1}{q_i n}\E\left[\|u_\kmone^i - u_\star^i\|^2\right] + 3 A_k,
\label{eq:aux1}
\end{align}
where the first inequality uses the relation $\E[\|X- \E[X]\|^2] \leq \E[\|X\|^2]$ for all random variable $X$, taking here expectation with respect to the index $i_k \sim Q$ and conditioning on~$\Fcal_\kmone$; the second inequality uses the relation $\|a+b\|^2 \leq 2\|a\|^2 + 2\|b\|^2$; the last inequality uses Lemma~\ref{lemma:useful}.

We have now two possibilities to control the quantity $A_k$ related to~$\zeta_k$. First, we may simply upper bound it as follows
\eq{
   A_k= \E\left[ \frac{1}{(q_{i_k} n)^2} \|\zeta_k\|^2 \right] \le \rho_Q\tilde{\sigma}^2.
}
Then, since $x^\star$ minimizes $F$, we have $0 \in \nabla f(x^\star) + \partial \psi(x^\star)$ and thus $-\nabla f(x^\star)$ is a subgradient in $\partial \psi(x^\star)$. By using as well the convexity inequality $\psi(x) \geq \psi(x^\star) - \nabla f(x^\star)^\top(x-x^\star)$, we have
\begin{equation}
f(x_{\kmone})- f(x^\star) - \nabla f(x^\star)^\top (x_\kmone - x^\star)
\leq  F(x_\kmone) - F^\star, 
\label{eq:reduc_variant}
\end{equation}
leading finally to~(\ref{eq:var2}).

The second possibility is to relate $A_k$ to $\tilde{\sigma}_\star^2$, under
   the assumption that each $f_i$ may be written as $f_i(x)=\E_{\xi}\brb{\tilde{f}_i(x,\xi)},i\in[1,\dots,n]$ with $\tilde{f}_i(.,\xi)$ $L_i$-smooth with $L_i \geq \mu$ for all~$\xi$.
Then, 
\eqm{
    \E&\left[\dr{1}{(q_{i_k} n)^2} \n{\zeta_k}
    \right] = 
    \E\brb{ 
        \dr{1}{(q_{i_k} n)^2} 
        \n{
            \tildenabla f_{i_k}(\xkm) -
            \nabla f_{i_k}(\xkm)
        }
    }
    \car\quad\quad=\!
    \E\left[
    \dr{1}{(q_{i_k} n)^2} \!
    \n{
        \tilde{\nb} f_{i_k}(\xkm) \!-\! 
        \tilde{\nb} f_{i_k}(\xo) \!+\! 
        \tilde{\nb} f_{i_k}(\xo) \!-\!
        \nb f_{i_k}(\xo) \!+\!
        \nb f_{i_k}(\xo) \!-\!
        \nb f_{i_k}(\xkm)}
    \right]
    \car\quad\quad\le
    \E\left[
    \dr{1}{(q_{i_k} n)^2}
    \brb{ \n{
        \tilde{\nb} f_{i_k}(\xkm) - 
        \tilde{\nb} f_{i_k}(\xo) +
        \tilde{\nb} f_{i_k}(\xo) -
        \nb f_{i_k}(\xo)
     }}\right]
    \car\quad\quad\le
    2\E\left[
    \dr{1}{(q_{i_k} n)^2}
    \brb{ 
    \n{
        \tilde{\nb} f_{i_k}(\xkm) - 
        \tilde{\nb} f_{i_k}(\xo)
    } + 
    \n{
        \tilde{\nb} f_{i_k}(\xo) -
        \nb f_{i_k}(\xo)
     }}\right]
    \car\quad\quad\le
    4\E\left[
    \dr{L_{i_k}}{(q_{i_k} n)^2}
    \brc{
        \fer{f_{i_k}}{\xkm} - 
        \bra{\nb f_{i_k}(\xo), \xkm-\xo}
     }\right] + 
    2\E\brb{
        \dr1{(q_{i_k} n)^2}
        \tilde{\sigma}_{i_k,\star}^2
    }
    \car\quad\quad\le
    4L_Q\E\left[
    \dr{1}{q_{i_k} n}
    \brc{
        \fer{f_{i_k}}{\xkm} - 
        \bra{\nb f_{i_k}(\xo), \xkm-\xo}
     }\right] + 
    2\rho_Q\tilde{\sigma}_{\star}^2
    \car\quad\quad=
    4L_Q \brc{
        \fer{f}{\xkm} - 
        \bra{\nb f(\xo), \xkm-\xo}
    } + 
    2\rho_Q\tilde{\sigma}_{\star}^2,
    \label{app.tsi.exp}
}
where we use the relation $\E[\|X- \E[X]\|^2] \leq \E[\|X\|^2]$ for the first inequality, the well-known inequality for a convex norm~$\n{a+b}\le 2\n{a}+2\n{b}$ for the second inequality and the definition~$\tilde{\sigma}_\star = \average \tilde{\sigma}_{i,\star}^2$.

Then, we may combine~(\ref{app.tsi.exp}) with~(\ref{eq:aux1}) and
use~(\ref{eq:reduc_variant}) to obtain~(\ref{eq:var3}).

\end{proof}

\subsection{Proof of Proposition~\ref{thm:lyapunov}}
\label{appendix:prop.lyapunov}
\begin{proof}
To make the notation more compact, we call
\begin{displaymath}
F_k = \E[F(x_k)-F^\star],\qquad D_k = \E[d_k(x^\star)-d_k^\star] \qquad \text{and}~~~~
C_k = \E\left[\frac{1}{n} \sum_{i=1}^n \frac{1}{q_i n} \| u^i_k - u^i_\star\|^2\right].
\end{displaymath}
Then, according to~Proposition~\ref{prop:nonu}, we have
\begin{displaymath}
\omega_k^2 \leq 4 L_Q F_\kmone + 2 C_\kmone + {3 \rho_Q \tilde{\sigma}^2},
\end{displaymath}
and according to Proposition~\ref{prop:keyprop},
\begin{equation}
\delta_k F_k + D_k \leq (1-\delta_k) D_\kmone  + 4 L_Q \eta_k \delta_k F_\kmone + 2 \eta_k \delta_k C_\kmone + {3\rho_Q \eta_k \delta_k \tilde{\sigma}^2}. \label{eq:key}
\end{equation}
Then, we note that both for the SVRG and SAGA/MISO/SDCA strategies, we have (with $\beta=0$ for SVRG),
\begin{displaymath}
\E[\| u_k^i - u^i_\star\|^2] = \left(1 - \frac{1}{n}\right)\E[\| u^i_\kmone - u^i_\star\|^2] + \frac{1}{n}\E \|\nabla f_i(x_k) - \nabla f_i(x^\star) + \beta(x_k - x^\star)\|^2.
\end{displaymath}
By taking a weighted average, this yields
\begin{equation*}
\begin{split}
C_k & \leq \left(1 - \frac{1}{n}\right)C_\kmone + \frac{1}{n^2}\sum_{i=1}^n\frac{1}{q_i n}\E\left[\|\nabla f_{i}(x_{k})- \nabla f_i(x^\star) -\beta (x_{k}-x^\star)\|^2 \right] \\
& \leq \left(1 - \frac{1}{n}\right)C_\kmone + \frac{1}{n^2}\sum_{i=1}^n\frac{2L_i}{q_i n}\E\left[f_i(x_k)- f_i(x^\star)-\nabla f_i(x^\star)^\top(x_k-x^\star) \right] \\
& \leq \left(1 - \frac{1}{n}\right)C_\kmone + \frac{2 L_Q F_k}{n},
\end{split}
\end{equation*}
where the second inequality comes from Lemma~\ref{lemma:useful} and the last one uses similar arguments as in the proof of Proposition~\ref{prop:nonu}.
Then, we add a quantity $\beta_k C_k$ on both sides of the relation~(\ref{eq:key}) with some $\beta_k > 0$ that we will specify later:
\begin{multline*}
\left(\delta_k-\beta_k \frac{2L_Q}{n}\right) F_k + D_k + \beta_k C_k \\ \leq (1-\delta_k) D_\kmone + \left(\beta_k\left(1 - \frac{1}{n}\right)+2 \eta_k \delta_k \right) C_\kmone + 4 L_Q \eta_k \delta_k F_\kmone  + {3\rho_Q \eta_k \delta_k\tilde{\sigma}^2},
\end{multline*}
and then choose  $\frac{\beta_k}{n} = \frac{5}{2} \eta_k \delta_k $, which yields
\begin{equation*}
\delta_k\left(1-5 L_Q \eta_k\right) F_k + D_k + \beta_k C_k \leq (1-\delta_k) D_\kmone + \beta_k\left(1 - \frac{1}{5n}\right) C_\kmone + 4 L_Q \eta_k \delta_k F_\kmone  + {3 \rho_Q \eta_k \delta_k\tilde{\sigma}^2}. 
\end{equation*}
Remember that $\tau_k = \min\left(\delta_k, \frac{1}{5n}\right)$, notice that the sequences $(\beta_k)_{k \geq 0}, (\eta_k)_{k \geq 0}$ and
$(\delta_k)_{k \geq 0}$ are non-increasing and note that ${4} \leq {5}(1-\frac{1}{5n})$ for all $n \geq 1$. Then,
\begin{multline*}
\delta_k\left(1-10 L_Q \eta_k\right) F_k + \underbrace{5L_Q \eta_k \delta_k +  D_k + \beta_k C_k}_{T_k} \\ \leq (1-\tau_k) \left(D_\kmone + \beta_\kmone C_\kmone + 5 L_Q \eta_\kmone \delta_\kmone F_\kmone\right)  + {3 \rho_Q \eta_k \delta_k\tilde{\sigma}^2},
\end{multline*}
which immediately yields~(\ref{eq:aux2}) with the appropriate definition of $T_k$, and by noting that $(1-10 L_Q\eta_k) \geq \frac{1}{6}$.
\end{proof}

\subsection{Proof of Theorem~\ref{thm:svrg}}
\label{appendix:thm:svrg}
\begin{proof}
The first part of the theorem is a direct application of Lemma~\ref{lemma:averaging} to
Proposition~\ref{thm:lyapunov}, by noting that~(\ref{eq:av1})
holds---when  replacing the notation $\delta_t$ by $\tau_t$
in~(\ref{eq:av1})---since for a fixed number of iterations~$K$, we have the relation
$\frac{\tau_k \delta_K}{6 \tau_K}\E[F(x_k)-F^\star] + T_k  \leq \left( 1 -
\tau_k \right)T_\kmone  + {3 \rho_Q \eta_k \delta_k\tilde{\sigma}^2}$ for all $k \leq K$.
Indeed, $\delta_k = \frac{\tau_k \delta_k}{\tau_k} \geq \frac{\tau_k \delta_K}{\tau_K}$ since the ratio $\delta_t/\tau_t$ is non-increasing.
   Then, we may now prove~(\ref{eq:T0}):
\begin{displaymath}
   \begin{split}
      T_0 & = {5} L_Q\eta_0 \delta_0 (F(x_0)-F^\star) + d_0(x^\star)-d_0^\star + \frac{5 \eta_0 \delta_0}{2}\frac{1}{n} \sum_{i=1}^n \frac{1}{q_i n} \| u^i_0 - u^i_\star\|^2 \\
       & \leq {5} L_Q\eta_0 \delta_0 (F(x_0)-F^\star) + d_0(x^\star)-d_0^\star \\
       & \qquad \qquad + \frac{5 \eta_0 \delta_0}{2}\frac{1}{n} \sum_{i=1}^n \frac{2 L_i}{q_i n} (f_i(x_0) - f_i(x^\star) - \nabla f_i(x^\star)^\top( x_0-x^\star))  \\
       & \leq {5} L_Q\eta_0 \delta_0 (F(x_0)-F^\star) + d_0(x^\star)-d_0^\star + {5 \eta_0 \delta_0}{L_Q} (f(x_0) - f(x^\star) - \nabla f(x^\star)^\top( x_0-x^\star))  \\
       & \leq {10} L_Q\eta_0 \delta_0 (F(x_0)-F^\star) + d_0(x^\star)-d_0^\star,  
   \end{split}
\end{displaymath}
   where the first inequality uses Lemma~\ref{lemma:useful}, and the second one uses the definition of $L_Q$, whereas the last one uses~(\ref{eq:reduc_variant}).
\end{proof}

\subsection{Proof of Corollary~\ref{corollary:svrg0}}
\label{subsec.corollary:svrg0}
\begin{proof}
First, notice that $\delta_k = \eta_k \gamma_k = \frac{\mu}{12 L_Q}$ and
that $\alpha=\frac{6 \tau_k}{\delta_k}$. Then, we apply Theorem~\ref{thm:svrg} and obtain
\begin{displaymath}
\begin{split}
   \E\left[F(\hat{x}_k)-F^\star  + \alpha T_k\right] & \leq \Theta_k \left( F(\hat{x}_0)-F^\star + \alpha T_0 + \frac{18 \rho_Q \tau_k \tilde{\sigma}^2}{\delta_k} \sum_{t=1}^k \frac{\eta_t \delta_t}{\Theta_t}     \right) \\
   & = \Theta_k \left( F(\hat{x}_0)-F^\star + \alpha T_0 + \frac{3 \rho_Q \tilde{\sigma}^2}{2L_Q} \sum_{t=1}^k \frac{\tau_t}{\Theta_t}     \right) \\
   & \leq \Theta_k \left( F(\hat{x}_0)-F^\star + \alpha T_0 \right)  + \frac{3 \rho_Q \tilde{\sigma}^2}{2L_Q}. \\
\end{split}
\end{displaymath}
\end{proof}


\subsection{Proof of Corollary~\ref{corollary:svrg2}}
\label{appendix:cor.svrg2}
\begin{proof}
Since the convergence rate~(\ref{eq:svrg_constant2}) applies for the first stage with a constant step size, the number of iterations to ensure the condition $\E[F(\hat{x}_k)-F^\star + 6 T_k] \leq 24\eta\rho_Q \tilde{\sigma}^2$ is upper bounded by $K$ with
\begin{displaymath}
   K = O\left( \left(n + \frac{L_Q}{\mu}\right) \log\left(\frac{F({x}_0)-F^\star + d_0(x^\star)-d_0^\star}{\varepsilon} \right)  \right),
\end{displaymath}
   when using the upper-bound~(\ref{eq:T0}) on~$T_0$. 
   Then, we restart the optimization procedure, using $x_0'=x_K$ and $\hat{x}_0' = \hat{x}'_K$, assuming from now on that $\E[F(\hat{x}_0')-F^\star + 6 T_0'] \leq 24\eta \rho_Q \tilde{\sigma}^2$, with decreasing step sizes
$\eta_k = \min\left( \frac{2}{\mu(k+2)}, {\eta} \right)$. 
Then, since $\delta_k = \mu \eta_k \leq \frac{1}{5n}$, we have that $\tau_k = \delta_k$ for all $k$, and
Theorem~\ref{thm:svrg} gives us---note that here $\Gamma_k=\Theta_k$---
\begin{displaymath}
   \E\left[F(\hat{x}_k')-F^\star\right] \leq \Gamma_k \left( F(\hat{x}_0')-F^\star + {6}T_0' + {18 \rho_Q \tilde{\sigma}^2} \sum_{t=1}^k \frac{\eta_t \delta_t}{\Gamma_t}     \right)~~~\text{with}~~~ \Gamma_k = \prod_{t=1}^k(1-\delta_t).
\end{displaymath}
Then, after taking the expectation with respect to the output of the first stage,
\begin{displaymath}
\E\left[F(\hat{x}_k')-F^\star\right] 
 \leq \Gamma_k \left( { 24 \rho_Q \eta \tilde{\sigma}^2} + {18 \rho_Q\tilde{\sigma}^2} \sum_{t=1}^k \frac{\eta_t \delta_t}{\Gamma_t}     \right).
\end{displaymath}
Denote now by $k_0$ the largest index such that $\frac{2}{\mu(k_0+2)} \geq {\eta} $ and thus $k_0 = \lceil 2/(\mu {\eta}) - 2 \rceil$.
Then, according to Lemma~\ref{lemma:step}, for $k \geq k_0$,
\begin{displaymath}
\begin{split}
  \E\left[F(\hat{x}_k)-F^\star\right] & \leq \Gamma_k \left( {24 \rho_Q \eta\tilde{\sigma}^2} + {18 \rho_Q {\eta} \tilde{\sigma}^2 }\sum_{t=1}^{k_0-1} \frac{\delta_t}{\Gamma_t} + {18 \rho_Q \tilde{\sigma}^2} \sum_{t=k_0}^{k} \frac{2\delta_t}{\mu \Gamma_t (t+2)}\right) \\
& \leq \frac{k_0(k_0+1)}{(k+1)(k+2)} \left( \Gamma_{k_0-1}{24 \rho_Q {\eta} \tilde{\sigma}^2} + {18 {\eta} \rho_Q \tilde{\sigma}^2 } \Gamma_{k_0-1}\sum_{t=1}^{k_0-1} \frac{\delta_t}{\Gamma_t}\right) \\
   & \qquad \qquad + {36\rho_Q \tilde{\sigma}^2} \sum_{t=k_0}^{k} \frac{\delta_t \Gamma_k}{\mu \Gamma_t (t+2)} \\
& \leq \frac{k_0(k_0+1)}{(k+1)(k+2)} {24 {\eta} \rho_Q \tilde{\sigma}^2} + {36 \rho_Q \tilde{\sigma}^2} \sum_{t=k_0}^{k} \frac{(t+1)(t+2)}{\mu (k+1)(k+2) (t+2)^2} \\
& \leq \frac{k_0{\eta}}{k+2} {24 \rho_Q \tilde{\sigma}^2} + \frac{36 \rho_Q \tilde{\sigma}^2}{\mu (k+2)} = O\left( \frac{\rho_Q \tilde{\sigma}^2}{\mu k}\right),\\
\end{split}
\end{displaymath}
which gives the desired complexity.

\end{proof}

\subsection{Proof of Corollary~\ref{thm:svrg.mu0}}
\label{app.thm:svrg.mu0}
\begin{proof}
Let us call $x_0'$ the point obtained by running one iteration
   of~(\ref{eq:opt1}) with step-size $\eta \leq \frac{1}{12L_Q}$ and gradient
estimator~$(1/n)\sum_{i=1}^n \tildenabla f_i(\xz)$, whose variance is $\tilde{\sigma}^2/n$.
Then, since $\delta_1=\Gamma_1=1/2$,  according to Theorem~\ref{thm:conv}, we have
\begin{equation}
   \E\left[F(x_0')-F^\star  + \frac{1}{2\eta}\|x_0' - x^\star\|^2\right] \leq \frac{1}{2\eta}\|x_0 - x^\star\|^2 +  \frac{\eta\tilde{\sigma}^2}{n}.  \label{eq:svrg.mu0.aux}
\end{equation}
   Then, we consider the main run of the algorithm, and apply Theorem~\ref{thm:svrg}, replacing $x_0$ by $x_0'$. With the chosen setup, we have $\delta_k = \frac{1}{k+1}$ and since $K \geq 5n$, we have $\delta_K=\tau_K$, such that~(\ref{eq:lyapunov.svrg}) becomes
\begin{equation*}
\E\left[F(\hat{x}_K)-F^\star\right] \leq \Theta_K \left( F(x_0')-F^\star + 6 T_0 + {18 \rho_Q\eta\tilde{\sigma}^2} \sum_{t=1}^k \frac{\delta_t}{\Theta_t}\right),
\end{equation*}
   and from~(\ref{eq:T0}), we have
   $$T_0 \leq {10} L_Q\eta (F(x_0')-F^\star) + \frac{1}{2\eta}\|x_0'-x^\star\|^2 \leq \frac{5}{6} (F(x_0')-F^\star) + \frac{1}{2\eta}\|x_0'-x^\star\|^2,$$
   which yields, combined with~(\ref{eq:svrg.mu0.aux}),
\begin{equation*}
   \E[F(x_0')-F^\star + 6 T_0] \leq 6\E\left[F(x_0')-F^\star + \frac{1}{2\eta}\|x_0'-x^\star\|^2 \right] \leq \frac{3}{\eta}\|x_0 - x^\star\|^2 +  \frac{6\eta\tilde{\sigma}^2}{n}.
\end{equation*}
   Note that Lemma~\ref{lemma:step} gives us that $\Theta_k =  (1-1/5n)^{5n-1}\frac{5n}{k+1} \leq \frac{3n}{k+1}$ for $k \geq 5n$ and since $1+\sum_{t=1}^K \frac{\tau_t}{\Theta_t} = \frac{1}{\Theta_K}$ according to Lemma~\ref{lemma:simple}, 
\begin{equation*}
   \begin{split}
      \E\left[F(\hat{x}_K)-F^\star\right] & \leq \Theta_K \left(  \frac{3}{\eta}\|x_0 - x^\star\|^2 +  \frac{6\eta\tilde{\sigma}^2}{n}+ {18 \rho_Q\eta\tilde{\sigma}^2} \sum_{t=1}^K \frac{ \delta_t}{\Theta_t}\right), \\
      & \leq \frac{9 n}{\eta (K+1)}\|x_0 - x^\star\|^2 + 6 \eta \tilde{\sigma}^2 \rho_Q \Theta_K\left( \frac{1}{n} + 3\sum_{t=1}^K \frac{ \tau_t}{\Theta_t} + 3\sum_{t=1}^{5n-1} \frac{ \delta_t}{\Theta_t} \right) \\
      & \leq \frac{9 n}{\eta (K+1)}\|x_0 - x^\star\|^2 + 6 \eta \tilde{\sigma}^2 \rho_Q \left( \frac{\Theta_K}{n} + 3(1-\Theta_K) + \frac{15 n}{K+1}\sum_{t=1}^{5n-1}  \delta_t \right) \\
      & \leq \frac{9 n}{\eta (K+1)}\|x_0 - x^\star\|^2 + 18 \eta \tilde{\sigma}^2 \rho_Q \left( 1 + \frac{5n}{K+1}\log(5n) \right) \\
      & \leq \frac{9 n}{\eta (K+1)}\|x_0 - x^\star\|^2 + 36 \eta \tilde{\sigma}^2 \rho_Q. 
   \end{split}
\end{equation*}
It remains to optimize it over~$\eta$ to get the left side of~(\ref{eq:cor.lyapunov.svrg}).
\end{proof}

\subsection{Proof of Lemma~\ref{lemma:acc}}
\label{appendix:lemma.acc}
\begin{proof}
Let us assume that the relation $y_\kmone = (1-\theta_\kmone) x_\kmone + \theta_\kmone v_\kmone$ holds and let us show that it also holds for $y_{k}$.
Since the estimate sequences $d_k$ are quadratic functions, we have
\begin{displaymath}
\begin{split}
v_k & = (1-\delta_k)\frac{\gamma_\kmone}{\gamma_k} v_{\kmone} + \frac{\mu \delta_k}{\gamma_k} y_\kmone - \frac{\delta_k}{\gamma_k}(g_k + \psi'(x_k)) \\
& = (1-\delta_k)\frac{\gamma_\kmone}{\gamma_k} v_{\kmone} + \frac{\mu \delta_k}{\gamma_k} y_\kmone - \frac{\delta_k}{\gamma_k\eta_k}(y_\kmone - x_k) \\
& = (1-\delta_k)\frac{\gamma_\kmone}{\gamma_k\theta_\kmone} \left(y_\kmone - (1-\theta_\kmone) x_\kmone  \right) + \frac{\mu \delta_k}{\gamma_k} y_\kmone - \frac{\delta_k}{\gamma_k\eta_k}(y_\kmone - x_k) \\
& = (1-\delta_k)\frac{\gamma_\kmone}{\gamma_k\theta_\kmone} \left(y_\kmone - (1-\theta_\kmone) x_\kmone  \right) + \frac{\mu \delta_k}{\gamma_k} y_\kmone - \frac{1}{\delta_k}(y_\kmone - x_k) \\
& = \left(\frac{(1-\delta_k)\gamma_\kmone}{\gamma_k\theta_\kmone}  + \frac{\mu \delta_k}{\gamma_k} - \frac{1}{\delta_k} \right)y_\kmone   -  \frac{(1-\delta_k)\gamma_\kmone(1-\theta_\kmone)}{\gamma_k\theta_\kmone} x_\kmone   + \frac{1}{\delta_k}x_k \\
& = \left( 1 + \frac{(1-\delta_k)\gamma_\kmone (1-\theta_\kmone)}{\gamma_k\theta_\kmone} - \frac{1}{\delta_k} \right)y_\kmone   -  \frac{(1-\delta_k)\gamma_\kmone(1-\theta_\kmone)}{\gamma_k\theta_\kmone} x_\kmone   + \frac{1}{\delta_k}x_k.
\end{split}
\end{displaymath}
Then note that $\theta_\kmone = \frac{\delta_k \gamma_\kmone}{\gamma_\kmone + \delta_k \mu}$ and thus, $\frac{\gamma_\kmone (1-\theta_\kmone)}{\gamma_k\theta_\kmone} = \frac{1}{\delta_k}$, and
\begin{displaymath}
\begin{split}
v_k & = x_\kmone +
\frac{1}{\delta_k}(x_k-x_\kmone).
\end{split}
\end{displaymath}
Then, we note that $x_k - x_\kmone = \frac{\delta_k}{1-\delta_k}(v_k  - x_k)$ and we are left with
\begin{displaymath}
y_k = x_k + \beta_k(x_k-x_\kmone) = \frac{\beta_k \delta_k}{1-\delta_k} v_k  +   \left(  1-\frac{\beta_k \delta_k}{1-\delta_k}\right) x_k.
\end{displaymath}
Then, it is easy to show that
\begin{displaymath}
\beta_k = \frac{(1-\delta_k)\delta_{k+1} \gamma_k}{\delta_k( \gamma_{k+1} + \delta_{k+1}\gamma_k)} = \frac{(1-\delta_k)\delta_{k+1} \gamma_k}{\delta_k( \gamma_{k} + \delta_{k+1}\mu)} = \frac{(1-\delta_k)\theta_k}{\delta_k} ,
\end{displaymath}
which allows us to conclude that  $y_k = (1-\theta_k) x_k + \theta_k v_k$ since the relation holds trivially for $k=0$.
\end{proof}

\subsection{Proof of Lemma~\ref{lemma:key_acc}}
\label{appendix:lemma.key_acc}

\begin{proof}
\begin{align*}
\E[F(x_k)] & = \E[f(x_k) + \psi(x_k)] \\
& \leq \E\left[ f(y_{k-1}) + \nabla f(y_{k-1})^\top (x_k - y_{k-1}) + \frac{L}{2}\|x_k-y_{k-1}\|^2 + \psi(x_k)\right] \\
& = \E\left[A_k\right]  +  \E\left[ (\nabla f(y_{k-1})-g_k)^\top (x_k - y_{k-1}) \right] \\
& = \E\left[A_k\right]  +  \E\left[ (\nabla f(y_{k-1})-g_k)^\top x_k \right] \\
& = \E\left[A_k\right]  +  \E\left[ (\nabla f(y_{k-1})-g_k)^\top (x_k - w_k) \right] \\
& \leq \E\left[A_k\right]  +  \E\left[ \|\nabla f(y_{k-1})-g_k \|\| x_k - w_k \| \right] \\
& \leq \E\left[A_k\right]  +  \E\left[ \eta_k\|\nabla f(y_{k-1})-g_k \|^2\right] \\
& \leq \E\left[A_k \right] + \eta_k \omega_k^2,
\end{align*}
where $A_k = f(y_{k-1}) + g_k^\top (x_k - y_{k-1}) + \frac{L}{2}\|x_k-y_{k-1}\|^2 + \psi(x_k)$ and $w_k = \Prox_{\eta_k\psi}[y_{k-1} - \eta_k \nabla f(y_{k-1})]$. The first inequality is due to the $L$-smoothness of $f$ (Lemma~\ref{lemma:upper}); then, the next three relations exploit the fact that $\E[(\nabla f(y_{k-1})-g_k)^\top z] = 0$ for all $z$ that is deterministic with respect to the algebra~$\mathcal{F}_{k-1}$; the third inequality uses the non-expansiveness of the proximal operator.
Using the definition~(\ref{eq:lk}) for~$l_k$, we proceed with
\begin{displaymath}
\begin{split}
\E[F(x_k)] & \leq \E\left[f(y_\kmone) + g_k^\top (x_k - y_\kmone) + \frac{L}{2}\|x_k-y_\kmone\|^2 + \psi(x_k)\right]  +  \eta_k \omega_k^2, \\
& = \E\left[l_k(y_\kmone) + \tilde{g}_k^\top (x_k - y_\kmone) + \frac{L}{2}\|x_k-y_\kmone\|^2\right]  +  \eta_k\omega_k^2, \\
& \leq \E\left[l_k(y_\kmone)\right] + \left(\frac{L\eta_k^2}{2} - \eta_k\right)\E\left[\|\tilde{g}_k\|^2\right]  +  \eta_k\omega_k^2, \\
\end{split}
\end{displaymath}
where we use the fact that $x_k = y_\kmone - \eta_k \tilde{g}_k$ and $\tilde{g}_k=g_k + \psi'(x_k)$.
\end{proof}

\subsection{Proof of Corollary~\ref{corollary:acc_sgd2}}\label{appendix:corollary:acc_sgd2}
\begin{proof}
The proof is similar to that of Corollary~\ref{corollary:sgd} for unaccelerated SGD.
The first stage with constant step-size requires $O\left( \sqrt{\frac{L}{\mu}} \log\left(\frac{F(x_0)- F^\star}{\varepsilon}\right)\right)$ iterations. Then, we restart the optimization
procedure, and assume that $\E\left[F(x_0)-F^\star + \frac{\mu}{2}\|x^\star-x_0\|^2\right] \leq \frac{2\sigma^2}{\sqrt{\mu L}}$.
With the choice of parameters, we have $\gamma_k = \mu$ and $\delta_k = \sqrt{\gamma_k \eta_k} = \min\left( \sqrt{ \frac{\mu}{L} }, \frac{2}{k+2} \right)$. We may then apply Theorem~\ref{thm:acc_sgd} where the value of $\Gamma_k$ is given by Lemma~\ref{lemma:step}. This yields for $k \geq k_0 = \left\lceil 2\sqrt{\frac{L}{\mu}} - 2 \right\rceil$,
\begin{align*}
\E[F({x}_k)\!-\! F^\star]  & \leq \Gamma_k\left( \E\left[F(x_0)-F^\star + \frac{\mu}{2}\|x_0-x^\star\|^2\right] + \sigma^2 \sum_{t=1}^k \frac{\eta_t}{\Gamma_t} \right) \\
& \leq  \Gamma_k \left( \frac{2 \sigma^2}{\sqrt{\mu L}} + \frac{\sigma^2}{L} \sum_{t=1}^{k_0-1} \frac{1}{\Gamma_t} + \sigma^2 \sum_{t=k_0}^{k} \frac{4}{\Gamma_t \mu(t+2)^2}\right) \\
& =  \frac{k_0(k_0+1)}{(k+1)(k+2)} \left( \Gamma_{k_0-1}\frac{2 \sigma^2}{\sqrt{ \mu L}} + \frac{\sigma^2}{L} \Gamma_{k_0-1}\!\sum_{t=1}^{k_0-1} \frac{1}{\Gamma_t}\right) + \sigma^2 \sum_{t=k_0}^{k} \!\frac{4\Gamma_k}{\Gamma_t \mu(t+2)^2} \\
& =  \frac{k_0(k_0+1)}{(k+1)(k+2)} \left( \Gamma_{k_0-1}\frac{2 \sigma^2}{\sqrt{ \mu L}} + (1-\Gamma_{k_0 -1 })\frac{\sigma^2}{\sqrt{\mu L}}\right) + \sigma^2 \sum_{t=k_0}^{k} \frac{4 \Gamma_k}{\Gamma_t \mu(t+2)^2} \\
& \leq  \frac{k_0(k_0+1)}{(k+1)(k+2)} \frac{2 \sigma^2}{\sqrt{\mu L}} +  \sigma^2 \frac{1}{(k+1)(k+2)} \left(\sum_{t=k_0+1}^{k} \frac{4 (t+1)(t+2)}{\mu(t+2)^2}\right) \\
& \leq  \frac{k_0}{(k+1)(k+2)} \frac{4 \sigma^2}{\mu} +   \frac{4\sigma^2}{\mu(k+2)}  \leq  \frac{8\sigma^2}{\mu(k+2)},
\end{align*}
where we use Lemmas~\ref{lemma:step} and~\ref{lemma:simple}.
This leads to the desired iteration complexity.
\end{proof}

\subsection{Proof of Corollary~\ref{corollary:acc_sgd3}}\label{appendix:corollary:acc_sgd3}
\begin{proof}
   Let us call $x_0'$ the point obtained by running on step of
   iteration~(\ref{eq:opt1}), which according to Theorem~\ref{thm:conv}
   satisfies, with $\gamma_0=1/\eta$,
   \begin{displaymath}
   \E\left[F(x_0')-F^\star  + \frac{1}{2\eta}\|x_0' - x^\star\|^2\right] \leq \frac{1}{2\eta}\|x_0 - x^\star\|^2 +  {\eta{\sigma}^2}.
   \end{displaymath}
Then, we note that
according to Lemma~\ref{eq:acc_rate_gamma}, we have
$$\Gamma_k \leq \frac{4}{\left(2 + k \sqrt{  \gamma_0\eta} \right)^2} \leq \frac{4}{\gamma_0 \eta  \left(1 + k\right)^2},$$
and we apply Theorem~\ref{thm:acc_sgd} to obtain the relation
\begin{displaymath}
\begin{split}
   \E[F({x}_K)-F^\star] & \leq \Gamma_K \E\left[ F(x_0')-F^\star + \frac{1}{2\eta}\|x_0'-x^\star\|^2 \right] + \sigma^2 \eta \Gamma_K \sum_{t=1}^K \frac{1}{\Gamma_t} \\
   & \leq \Gamma_K \left( \frac{\|x_0 - x^\star\|^2}{2\eta} + \eta \sigma^2  \right) + \sigma^2 \eta K \\
& \leq \frac{2}{(1+K)^2\eta}\|x_0 - x^\star\|^2 + \sigma^2 \eta (K+1). \\
\end{split}
\end{displaymath}
Optimizing with respect to $\eta$ under the constraint $\eta \leq 1/L$ gives~(\ref{eq:sgd_simple2}).
\end{proof}

\subsection{Proof of Proposition~\ref{prop:nonu2}}
\label{appendix:prop.nonu2}
\begin{proof}
\begin{equation*}
\begin{split}
\omega_k^2 & =  \E\left\|\frac{1}{q_{i_k}n} \left(\tildenabla f_{i_k}(y_{\kmone}) - \tildenabla f_{i_k}(\tilde{x}_\kmone) \right) + \tildenabla f(\tilde{x}_\kmone) - \nabla f(y_\kmone)\right\|^2 \\
& =  \E\left\|\frac{1}{q_{i_k}n} \left(\nabla f_{i_k}(y_{\kmone}) + \zeta_k - \zeta'_k - \nabla f_{i_k}(\tilde{x}_\kmone) \right) + \nabla f(\tilde{x}_\kmone) + \bar{\zeta}_\kmone - \nabla f(y_\kmone) \right\|^2, \\
& \leq  \E\left\|\frac{1}{q_{i_k}n} \left(\nabla f_{i_k}(y_{\kmone}) - \nabla f_{i_k}(\tilde{x}_\kmone) \right) + \nabla f(\tilde{x}_\kmone) + \bar{\zeta}_\kmone - \nabla f(y_\kmone) \right\|^2 + {2 \rho_Q \tilde{\sigma}^2}, \\
\end{split}
\end{equation*}
where $\zeta_k$ and $\zeta'_k$ are perturbations drawn at iteration $k$, and $\bar{\zeta}_\kmone$ was drawn last time $\tilde{x}_\kmone$ was updated.
Then, by noticing that for any deterministic quantity $Y$ and random variable $X$, we have $\E[\|X-\E[X] - Y\|^2] \leq \E[\|X\|^2] + \|Y\|^2$, taking expectation with respect to the index $i_k \sim Q$ and conditioning on~$\Fcal_\kmone$, we have
\begin{equation}
\begin{split}
\omega_k^2 & \leq  \E\left\|\frac{1}{q_{i_k}n} \left(\nabla f_{i_k}(y_{\kmone}) - \nabla f_{i_k}(\tilde{x}_\kmone) \right) \right\|^2 + \E[\|\bar{\zeta}_\kmone\|^2]  + {2 \rho_Q \tilde{\sigma}^2} \\
& \leq  \frac{1}{n}\sum_{i=1}^n \frac{1}{q_i n}\E\left\|\nabla f_{i}(y_{\kmone}) - \nabla f_{i}(\tilde{x}_\kmone)\right\|^2  + {3\rho_Q \tilde{\sigma}^2} \\
& \leq  \frac{1}{n}\sum_{i=1}^n \frac{2 L_i}{q_i n}\E\left[ f_{i}(\tilde{x}_\kmone) - f_{i}(y_{\kmone}) -  \nabla f_{i}(y_{\kmone})^\top (\tilde{x}_\kmone-y_\kmone)\right]  + {3 \rho_Q \tilde{\sigma}^2} \\
& \leq  \frac{1}{n}\sum_{i=1}^n {2 L_Q}\E\left[ f_{i}(\tilde{x}_\kmone) - f_{i}(y_{\kmone}) -  \nabla f_{i}(y_{\kmone})^\top (\tilde{x}_\kmone-y_\kmone)\right]  + {3 \rho_Q \tilde{\sigma}^2} \\
& =  {2 L_Q}\E\left[ f(\tilde{x}_\kmone) - f(y_{\kmone}) -  \nabla f(y_{\kmone})^\top (\tilde{x}_\kmone-y_\kmone)\right]  + {3 \rho_Q \tilde{\sigma}^2} \\
& =  {2 L_Q}\E\left[ f(\tilde{x}_\kmone) - f(y_{\kmone}) -  g_k^\top (\tilde{x}_\kmone-y_\kmone)\right]  + {3 \rho_Q \tilde{\sigma}^2},
\end{split} \label{eq:variance_svrg}
\end{equation}
where the second inequality uses the  upper-bound $\E[\|\bar{\zeta}\|^2] = \frac{\sigma^2}{n} \leq {\rho_Q \sigma^2}$, and the third one uses Theorem 2.1.5 in~\cite{nesterov}.
\end{proof}

\subsection{Proof of Lemma~\ref{lemma:key_acc_svrg}}\label{appendix:lemma:key_acc_svrg}
\begin{proof}
We can show that Lemma~\ref{lemma:key_acc} still holds and thus,
\begin{displaymath}
\begin{split}
\E[F(x_k)] & \leq
\E\left[l_k(y_\kmone) \right] + \left(\frac{L\eta_k^2}{2} - \eta_k\right)\E\left[\|\tilde{g}_k\|^2\right]  +  \eta_k\omega_k^2. \\
& \leq  \E\left[l_k(y_\kmone) + a_k f(\tilde{x}_\kmone) - a_k f(y_\kmone) + a_k g_k^\top (y_\kmone - \tilde{x}_\kmone)\right] \\
& \qquad\qquad\qquad\qquad + \E\left[\left(\frac{L\eta_k^2}{2} - \eta_k\right)\|\tilde{g}_k\|^2\right] + {3\rho_Q \eta_k\tilde{\sigma}^2}, \\
\end{split}
\end{displaymath}
Note also that
\begin{displaymath}
\begin{split}
l_k(y_\kmone) + f(\tilde{x}_\kmone) & - f(y_\kmone) = \psi(x_k) + \psi'(x_k)^\top(y_\kmone - x_k) + f(\tilde{x}_\kmone) \\
& \leq \psi(\tilde{x}_\kmone) - \psi'(x_k)^\top(\tilde{x}_\kmone - x_k) + \psi'(x_k)^\top(y_\kmone - x_k) + f(\tilde{x}_\kmone)  \\
& = F(\tilde{x}_\kmone) + \psi'(x_k)^\top(y_\kmone-\tilde{x}_\kmone).
\end{split}
\end{displaymath}
Therefore, by noting that $l_k(y_\kmone) + a_k f(\tilde{x}_\kmone) - a_k f(y_\kmone) \leq (1-a_k)l_k(y_\kmone) + a_k F(\tilde{x}_\kmone) + a_k \psi'(x_k)^\top(y_\kmone-\tilde{x}_\kmone)$,  we obtain the desired result.
\end{proof}

\subsection{Proof of Corollary~\ref{corollary:acc_svrg}}\label{appendix:corollary:acc_svrg}
\begin{proof}
The proof is similar to that of Corollary~\ref{corollary:acc_sgd2} for accelerated SGD.
The first stage with constant step-size $\eta$ requires $O\left( \left(n + \sqrt{\frac{nL_Q}{\mu}}\right) \log\left(\frac{F(x_0)- F^\star}{\varepsilon}\right)\right)$ iterations.
Then, we restart the optimization
procedure, and assume that $\E\left[F(x_0)-F^\star\right] \leq B$ with $B = 3\rho_Q\tilde{\sigma}^2 \sqrt{\eta/\mu n}$.

With the choice of parameters, we have $\gamma_k = \mu$ and $\delta_k = \sqrt{\frac{5\mu \eta_k}{3n}} = \min\left( \sqrt{\frac{5\mu \eta}{3n}}, \frac{2}{k+2} \right)$. We may then apply Theorem~\ref{thrm:acc_svrg} where the value of $\Gamma_k$ is given by Lemma~\ref{lemma:step}. This yields for $k \geq k_0 = \left\lceil \sqrt{\frac{12 n}{5 \mu \eta}} - 2 \right\rceil$,
\begin{equation*}
\begin{split}
\E[F({x}_k)- F^\star]  & \leq \Gamma_k\left( \E\left[F(x_0)-F^\star + \frac{\mu}{2}\|x_0-x^\star\|^2\right] + \frac{ 3 \rho_Q \tilde{\sigma}^2}{n} \sum_{t=1}^k \frac{\eta_t}{\Gamma_t} \right) \\
& \leq  \Gamma_k \left( 2 B + \frac{3 \rho_Q \tilde{\sigma}^2 \eta }{n} \sum_{t=1}^{k_0-1} \frac{1}{\Gamma_t} +  \frac{3 \rho_Q \tilde{\sigma}^2}{n}\sum_{t=k_0}^{k} \frac{12 n}{5\Gamma_t \mu(t+2)^2}\right) \\
& =  \frac{k_0(k_0+1)}{(k+1)(k+2)} \left( \Gamma_{k_0-1}2B + \frac{3 \rho_Q \tilde{\sigma}^2 \eta }{n} \sum_{t=1}^{k_0-1} \frac{\Gamma_{k_0-1}}{\Gamma_t}\right)\text{+}\frac{36 \rho_Q \tilde{\sigma}^2}{5\mu}\sum_{t=k_0}^{k} \frac{\Gamma_k}{\Gamma_t(t+2)^2} \\
& =  \frac{k_0(k_0+1)}{(k+1)(k+2)} \left( \Gamma_{k_0-1}2B + (1-\Gamma_{k_0 -1 }) \frac{3 \rho_Q \tilde{\sigma}^2 \eta }{n\delta_{k_0}}\right)\text{+}\frac{36 \rho_Q \tilde{\sigma}^2}{5\mu}\sum_{t=k_0}^{k} \frac{\Gamma_k}{\Gamma_t(t+2)^2} \\
& \leq  \frac{2 k_0(k_0+1)B}{(k+1)(k+2)}  +   \frac{8 \rho_Q \tilde{\sigma}^2}{\mu(k+1)(k+2)} \left(\sum_{t=k_0+1}^{k} \frac{(t+1)(t+2)}{(t+2)^2}\right) \\
& \leq  \frac{2 k_0 B}{k+2}  +   \frac{8 \rho_Q \tilde{\sigma}^2}{\mu(k+2)}, \\
\end{split}
\end{equation*}
where we use Lemmas~\ref{lemma:step} and~\ref{lemma:simple}. Then, note that $k_0 B \leq 6 \rho_Q \tilde{\sigma}^2/\mu$ and we obtain
the right  iteration complexity.
\end{proof}

\subsection{Proof of Corollary~\ref{corollary:acc_svrg_convex2}}\label{appendix:corollary:acc_svrg_convex2}
\begin{proof}
Let us call $x_0'$ the point obtained by running one iteration
   of~(\ref{eq:opt1}) with step-size $\eta \leq \frac{1}{3L_Q}$ and gradient
estimator~$(1/n)\sum_{i=1}^n \tildenabla f_i(\xz)$, whose variance is $\tilde{\sigma}^2/n$.
   Following the proof of Corollary~\ref{thm:svrg.mu0}, the relation~(\ref{eq:svrg.mu0.aux}) holds.
   Then, we consider the main run of the algorithm, and apply Theorem~\ref{thrm:acc_svrg}, replacing $x_0$ by $x_0'$, which yields, combined with~(\ref{eq:svrg.mu0.aux})
\begin{equation*}
   \begin{split}
      \E\left[F(x_k)-F^\star\right] & \leq \Gamma_k \left(F(x_0')-F^\star + \frac{1}{2\eta}\|x_0'-x^\star\|^2 + \frac{3\rho_Q\tilde{\sigma}^2 }{n}\sum_{t=1}^k \frac{\eta_t}{\Gamma_t} \right)  \\
       & \leq \Gamma_k \left(\frac{1}{2\eta}\|x_0'-x^\star\|^2 + \eta\frac{\tilde{\sigma}^2}{n} + \frac{3\rho_Q\tilde{\sigma}^2 }{n}\sum_{t=1}^k \frac{\eta_t}{\Gamma_t} \right). 
   \end{split}
\end{equation*}
   Then, we note that $\delta_k = \min\left( \sqrt{\frac{5 \Gamma_k}{3n}}  , \frac{1}{3n}\right)$ such that $\Gamma_k=\left(1-\frac{1}{3n} \right)^k$ for $k \leq k_0$, where $k_0$ is the index such that
$\left(1-\frac{1}{3n}\right)^{k_0+1} \leq \frac{1}{15n} < \left(1-\frac{1}{3n}\right)^{k_0}$,
   which gives us $(3n-1)\log(15n) \leq k_0 \leq 3n (\log (15n))$.
   For $k > k_0$, we are in a constant step size regime, and we may then use 
 Lemma~\ref{eq:acc_rate_gamma} to obtain
 \begin{displaymath}
    \Gamma_k = \Gamma_{k_0} \frac{4}{\left(2 + (k -k_0)\sqrt{\frac{5 \gamma_{k_0}\eta}{3 n}}\right)^2} \leq \Gamma_{k_0} \frac{4}{(k -k_0)^2{\frac{5 \Gamma_{k_0}}{3 n}}} \leq \frac{3n}{(k  -k_0)^2}.
 \end{displaymath}
   Then, noticing that $K \geq 2 k_0+1$, we have $K-k_0 \geq (K+1)/2$, and we conclude that
\begin{equation*}
   \E\left[F(x_K)-F^\star\right] \leq \frac{3n\|x_0'-x^\star\|^2}{2 \eta (K  - k_0)^2} + \frac{3\eta\rho_Q\tilde{\sigma}^2(K + 1)}{n} \leq \frac{6n\|x_0'-x^\star\|^2}{\eta (K+1)^2} + \frac{3\eta\rho_Q\tilde{\sigma}^2(K + 1)}{n}. 
\end{equation*}
   Then, it remains to optimize with respect to $\eta$, under the constraint $\eta \leq 1/(3L_Q)$, which provides~(\ref{eq:svrg.convex2}).

\end{proof}

\vskip 0.2in

\bibliography{bib}

\begin{thebibliography}{57}
\providecommand{\natexlab}[1]{#1}
\providecommand{\url}[1]{\texttt{#1}}
\expandafter\ifx\csname urlstyle\endcsname\relax
  \providecommand{\doi}[1]{doi: #1}\else
  \providecommand{\doi}{doi: \begingroup \urlstyle{rm}\Url}\fi

\bibitem[Agarwal et~al.(2012)Agarwal, Wainwright, Bartlett, and
  Ravikumar]{agarwal2009information}
A.~Agarwal, M.~J. Wainwright, P.~L. Bartlett, and P.~K. Ravikumar.
\newblock Information-theoretic lower bounds on the oracle complexity of convex
  optimization.
\newblock \emph{IEEE Transactions on Information Theory}, 58\penalty0
  (5):\penalty0 3235--3249, 2012.

\bibitem[Allen-Zhu(2017)]{accsvrg}
Z.~Allen-Zhu.
\newblock Katyusha: The first direct acceleration of stochastic gradient
  methods.
\newblock In \emph{Proceedings of Symposium on Theory of Computing (STOC)},
  2017.

\bibitem[Arjevani and Shamir(2016)]{arjevani2016dimension}
Y.~Arjevani and O.~Shamir.
\newblock Dimension-free iteration complexity of finite sum optimization
  problems.
\newblock In \emph{Advances in Neural Information Processing Systems (NIPS)},
  2016.

\bibitem[Aybat et~al.(2019)Aybat, Fallah, Gurbuzbalaban, and
  Ozdaglar]{aybat2019universally}
Necdet~Serhat Aybat, Alireza Fallah, Mert Gurbuzbalaban, and Asuman Ozdaglar.
\newblock A universally optimal multistage accelerated stochastic gradient
  method.
\newblock In \emph{Advances in Neural Information Processing Systems
  (NeurIPS)}, 2019.

\bibitem[Baes(2009)]{baes2009estimate}
M.~Baes.
\newblock Estimate sequence methods: extensions and approximations.
\newblock \emph{ETH technical report}, 2009.

\bibitem[Beck and Teboulle(2009)]{fista}
A.~Beck and M.~Teboulle.
\newblock A fast iterative shrinkage-thresholding algorithm for linear inverse
  problems.
\newblock \emph{SIAM Journal on Imaging Sciences}, 2\penalty0 (1):\penalty0
  183--202, 2009.

\bibitem[Bietti and Mairal(2017)]{smiso}
A.~Bietti and J.~Mairal.
\newblock Stochastic optimization with variance reduction for infinite datasets
  with finite-sum structure.
\newblock In \emph{Advances in Neural Information Processing Systems (NIPS)},
  2017.

\bibitem[Bottou et~al.(2018)Bottou, Curtis, and
  Nocedal]{bottou2018optimization}
L.~Bottou, F.~E. Curtis, and J.~Nocedal.
\newblock Optimization methods for large-scale machine learning.
\newblock \emph{SIAM Review}, 60\penalty0 (2):\penalty0 223--311, 2018.

\bibitem[Cohen et~al.(2018)Cohen, Diakonikolas, and
  Orecchia]{cohen2018acceleration}
M.~B. Cohen, J.~Diakonikolas, and L.~Orecchia.
\newblock On acceleration with noise-corrupted gradients.
\newblock In \emph{Proceedings of the International Conferences on Machine
  Learning (ICML)}, 2018.

\bibitem[Defazio et~al.(2014{\natexlab{a}})Defazio, Bach, and
  Lacoste-Julien]{saga}
A.~Defazio, F.~Bach, and S.~Lacoste-Julien.
\newblock Saga: A fast incremental gradient method with support for
  non-strongly convex composite objectives.
\newblock In \emph{Advances in Neural Information Processing Systems (NIPS)},
  2014{\natexlab{a}}.

\bibitem[Defazio et~al.(2014{\natexlab{b}})Defazio, Caetano, and
  Domke]{defazio2014finito}
A.~Defazio, T.~Caetano, and J.~Domke.
\newblock Finito: A faster, permutable incremental gradient method for big data
  problems.
\newblock In \emph{Proceedings of the International Conferences on Machine
  Learning (ICML)}, 2014{\natexlab{b}}.

\bibitem[Devolder(2011)]{devolder_2011}
O.~Devolder.
\newblock Stochastic first order methods in smooth convex optimization.
\newblock CORE Discussion Papers 2011070, Universit\'e catholique de Louvain,
  Center for Operations Research and Econometrics (CORE), 2011.

\bibitem[Fang et~al.(2018)Fang, Li, Lin, and Zhang]{spider2018}
C.~Fang, C.~J. Li, Z.~Lin, and T.~Zhang.
\newblock {SPIDER}: Near-optimal non-convex optimization via stochastic
  path-integrated differential estimator.
\newblock In \emph{Advances in Neural Information Processing Systems
  (NeurIPS)}, 2018.

\bibitem[Ghadimi and Lan(2012)]{ghadimi2012optimal}
S.~Ghadimi and G.~Lan.
\newblock Optimal stochastic approximation algorithms for strongly convex
  stochastic composite optimization {I}: A generic algorithmic framework.
\newblock \emph{SIAM Journal on Optimization}, 22\penalty0 (4):\penalty0
  1469--1492, 2012.

\bibitem[Ghadimi and Lan(2013)]{ghadimi2013optimal}
S.~Ghadimi and G.~Lan.
\newblock Optimal stochastic approximation algorithms for strongly convex
  stochastic composite optimization {II}: Shrinking procedures and optimal
  algorithms.
\newblock \emph{SIAM Journal on Optimization}, 23\penalty0 (4):\penalty0
  2061--2089, 2013.

\bibitem[Hiriart-Urruty and
  Lemar\'echal(1996)]{hiriart_urruty_lemarechal_1993ii}
J.-B. Hiriart-Urruty and C.~Lemar\'echal.
\newblock \emph{Convex analysis and minimization algorithms. {II}.}
\newblock Springer, 1996.

\bibitem[Hofmann et~al.(2015)Hofmann, Lucchi, Lacoste-Julien, and
  McWilliams]{hofmann_variance_2015}
T.~Hofmann, A.~Lucchi, S.~Lacoste-Julien, and B.~McWilliams.
\newblock Variance reduced stochastic gradient descent with neighbors.
\newblock In \emph{Advances in Neural Information Processing Systems (NIPS)},
  2015.

\bibitem[Hu et~al.(2009)Hu, Pan, and Kwok]{kwok2009}
C.~Hu, W.~Pan, and J.~T. Kwok.
\newblock Accelerated gradient methods for stochastic optimization and online
  learning.
\newblock In \emph{Advances in Neural Information Processing Systems (NIPS)},
  2009.

\bibitem[Johnson and Zhang(2013)]{johnson2013accelerating}
R.~Johnson and T.~Zhang.
\newblock Accelerating stochastic gradient descent using predictive variance
  reduction.
\newblock In \emph{Advances in Neural Information Processing Systems (NIPS)},
  2013.

\bibitem[Kingma and Ba(2014)]{kingma2014adam}
D.~P. Kingma and J.~Ba.
\newblock Adam: A method for stochastic optimization.
\newblock \emph{preprint arXiv:1412.6980}, 2014.

\bibitem[Kovalev et~al.(2020)Kovalev, Horv{\'a}th, and
  Richtarik]{kovalev2020don}
Dmitry Kovalev, Samuel Horv{\'a}th, and Peter Richtarik.
\newblock Don’t jump through hoops and remove those loops: Svrg and katyusha
  are better without the outer loop.
\newblock In \emph{Algorithmic Learning Theory}, pages 451--467, 2020.

\bibitem[Kulunchakov and Mairal(2019{\natexlab{a}})]{kulunchakov2019estimate}
A.~Kulunchakov and J.~Mairal.
\newblock Estimate sequences for variance-reduced stochastic composite
  optimization.
\newblock In \emph{Proceedings of the International Conferences on Machine
  Learning (ICML)}, 2019{\natexlab{a}}.

\bibitem[Kulunchakov and Mairal(2019{\natexlab{b}})]{kulunchakov2019generic}
A.~Kulunchakov and J.~Mairal.
\newblock A generic acceleration framework for stochastic composite
  optimization.
\newblock In \emph{Advances in Neural Information Processing Systems
  (NeurIPS)}, 2019{\natexlab{b}}.

\bibitem[Lan(2012)]{Lan2012}
G.~Lan.
\newblock An optimal method for stochastic composite optimization.
\newblock \emph{Mathematical Programming}, 133\penalty0 (1):\penalty0 365--397,
  2012.

\bibitem[Lan and Zhou(2018{\natexlab{a}})]{conjugategradient}
G.~Lan and Y.~Zhou.
\newblock An optimal randomized incremental gradient method.
\newblock \emph{Mathematical Programming}, 171\penalty0 (1--2):\penalty0
  167--215, 2018{\natexlab{a}}.

\bibitem[Lan and Zhou(2018{\natexlab{b}})]{lan_zhou_distributed2018}
G.~Lan and Y.~Zhou.
\newblock Random gradient extrapolation for distributed and stochastic
  optimization.
\newblock \emph{{SIAM} Journal on Optimization}, 28\penalty0 (4):\penalty0
  2753--2782, 2018{\natexlab{b}}.

\bibitem[Lei et~al.(2017)Lei, Ju, Chen, and Jordan]{scsg2017}
L.~Lei, C.~Ju, J.~Chen, and M.~I. Jordan.
\newblock Non-convex finite-sum optimization via {SCSG} methods.
\newblock In \emph{Advances in Neural Information Processing Systems (NIPS)},
  2017.

\bibitem[Lin et~al.(2015)Lin, Mairal, and Harchaoui]{catalyst}
H.~Lin, J.~Mairal, and Z.~Harchaoui.
\newblock {A Universal Catalyst for First-Order Optimization}.
\newblock In \emph{Advances in Neural Information Processing Systems (NIPS)},
  2015.

\bibitem[Lin et~al.(2018)Lin, Mairal, and Harchaoui]{catalyst_jmlr}
H.~Lin, J.~Mairal, and Z.~Harchaoui.
\newblock Catalyst acceleration for first-order convex optimization: from
  theory to practice.
\newblock \emph{Journal of Machine Learning Research (JMLR)}, 18\penalty0
  (212):\penalty0 1--54, 2018.

\bibitem[Lin et~al.(2014)Lin, Chen, and Pe{\~{n}}a]{lin_pena2014}
Q.~Lin, X.~Chen, and J.~Pe{\~{n}}a.
\newblock A sparsity preserving stochastic gradient methods for sparse
  regression.
\newblock \emph{Computational Optimization and Applications}, 58\penalty0
  (2):\penalty0 455--482, 2014.

\bibitem[Lu and Xiao(2015)]{Lu2015}
Z.~Lu and L.~Xiao.
\newblock On the complexity analysis of randomized block-coordinate descent
  methods.
\newblock \emph{Mathematical Programming}, 152\penalty0 (1):\penalty0 615--642,
  2015.

\bibitem[Mairal(2015)]{miso}
J.~Mairal.
\newblock Incremental majorization-minimization optimization with application
  to large-scale machine learning.
\newblock \emph{SIAM Journal on Optimization}, 25\penalty0 (2):\penalty0
  829--855, 2015.

\bibitem[Mairal(2016)]{mairal2016end}
J.~Mairal.
\newblock End-to-end kernel learning with supervised convolutional kernel
  networks.
\newblock In \emph{Advances in Neural Information Processing Systems (NIPS)},
  2016.

\bibitem[Mairal(2019)]{mairal2019cyanure}
J.~Mairal.
\newblock Cyanure: An open-source toolbox for empirical risk minimization for
  {Python}, {C++}, and soon more.
\newblock \emph{preprint arXiv:1912.08165}, 2019.

\bibitem[Mairal et~al.(2014)Mairal, Bach, and Ponce]{mairal2014sparse}
J.~Mairal, F.~Bach, and J.~Ponce.
\newblock Sparse modeling for image and vision processing.
\newblock \emph{Foundations and Trends in Computer Graphics and Vision},
  8\penalty0 (2-3):\penalty0 85--283, 2014.

\bibitem[Meinshausen and B{\"u}hlmann(2010)]{meinshausen2010stability}
N.~Meinshausen and P.~B{\"u}hlmann.
\newblock Stability selection.
\newblock \emph{Journal of the Royal Statistical Society: Series B (Statistical
  Methodology)}, 72\penalty0 (4):\penalty0 417--473, 2010.

\bibitem[Moreau(1962)]{moreau1962fonctions}
J.-J. Moreau.
\newblock Fonctions convexes duales et points proximaux dans un espace
  hilbertien.
\newblock \emph{CR Acad. Sci. Paris S{\'e}r. A Math}, 255:\penalty0 2897--2899,
  1962.

\bibitem[Moreau(1965)]{moreau1965}
J.-J. Moreau.
\newblock Proximit{\'e} et dualit{\'e} dans un espace hilbertien.
\newblock \emph{Bull. Soc. Math. France}, 93\penalty0 (2):\penalty0 273--299,
  1965.

\bibitem[Nemirovski et~al.(2009)Nemirovski, Juditsky, Lan, and
  Shapiro]{nemirovski}
A.~Nemirovski, A.~Juditsky, G.~Lan, and A.~Shapiro.
\newblock Robust stochastic approximation approach to stochastic programming.
\newblock \emph{{SIAM} Journal on Optimization}, 19\penalty0 (4):\penalty0
  1574--1609, 2009.

\bibitem[Nesterov(1983)]{nesterov1983}
Y.~Nesterov.
\newblock A method of solving a convex programming problem with convergence
  rate {$O$(1/$k^2$)}.
\newblock \emph{Soviet Mathematics Doklady}, 27\penalty0 (2):\penalty0
  372--376, 1983.

\bibitem[Nesterov(2004)]{nesterov}
Y.~Nesterov.
\newblock \emph{Introductory Lectures on Convex Optimization: A Basic Course}.
\newblock Springer, 2004.

\bibitem[Nesterov(2013)]{nesterov2013gradient}
Y.~Nesterov.
\newblock Gradient methods for minimizing composite functions.
\newblock \emph{Mathematical Programming}, 140\penalty0 (1):\penalty0 125--161,
  2013.

\bibitem[Nesterov and Polyak(2006)]{nesterov2006cubic}
Y.~Nesterov and B.~T. Polyak.
\newblock Cubic regularization of {Newton} method and its global performance.
\newblock \emph{Mathematical Programming}, 108\penalty0 (1):\penalty0 177--205,
  2006.

\bibitem[Nguyen et~al.(2017)Nguyen, Liu, Scheinberg, and
  Tak{\'a}{\v{c}}]{sarah}
L.~M. Nguyen, J.~Liu, K.~Scheinberg, and M.~Tak{\'a}{\v{c}}.
\newblock Sarah: A novel method for machine learning problems using stochastic
  recursive gradient.
\newblock In \emph{Proceedings of the International Conferences on Machine
  Learning (ICML)}, 2017.

\bibitem[Paquette et~al.(2018)Paquette, Lin, Drusvyatskiy, Mairal, and
  Harchaoui]{paquette2018catalyst}
C.~Paquette, H.~Lin, D.~Drusvyatskiy, J.~Mairal, and Z.~Harchaoui.
\newblock Catalyst for gradient-based nonconvex optimization.
\newblock In \emph{Proceedings of the International Conference on Artificial
  Intelligence and Statistics (AISTATS)}, 2018.

\bibitem[Schmidt et~al.(2015)Schmidt, Babanezhad, Ahmed, Defazio, Clifton, and
  Sarkar]{saganonu}
M.~Schmidt, R.~Babanezhad, M.~Ahmed, A.~Defazio, A.~Clifton, and A.~Sarkar.
\newblock Non-uniform stochastic average gradient method for training
  conditional random fields.
\newblock In \emph{Proceedings of the International Conference on Artificial
  Intelligence and Statistics (AISTATS)}, 2015.

\bibitem[Schmidt et~al.(2017)Schmidt, Le~Roux, and Bach]{schmidt2017minimizing}
M.~Schmidt, N.~Le~Roux, and F.~Bach.
\newblock Minimizing finite sums with the stochastic average gradient.
\newblock \emph{Mathematical Programming}, 162\penalty0 (1-2):\penalty0
  83--112, 2017.

\bibitem[Shalev-Shwartz(2016)]{ShalevShwartz2016SDCAWD}
S.~Shalev-Shwartz.
\newblock {SDCA} without duality, regularization, and individual convexity.
\newblock In \emph{Proceedings of the International Conferences on Machine
  Learning (ICML)}, 2016.

\bibitem[Shalev-Shwartz and Zhang(2016)]{accsdca}
S.~Shalev-Shwartz and T.~Zhang.
\newblock Accelerated proximal stochastic dual coordinate ascent for
  regularized loss minimization.
\newblock \emph{Mathematical Programming}, 155\penalty0 (1):\penalty0 105--145,
  2016.

\bibitem[Srivastava et~al.(2014)Srivastava, Hinton, Krizhevsky, Sutskever, and
  Salakhutdinov]{srivastava_dropout:_2014}
N.~Srivastava, G.~E. Hinton, A.~Krizhevsky, I.~Sutskever, and R.~Salakhutdinov.
\newblock Dropout: a simple way to prevent neural networks from overfitting.
\newblock \emph{Journal of Machine Learning Research}, 15\penalty0
  (1):\penalty0 1929--1958, 2014.

\bibitem[Vapnik(2000)]{vapnik2013nature}
V.~Vapnik.
\newblock \emph{The nature of statistical learning theory}.
\newblock Springer, 2000.

\bibitem[Wainwright et~al.(2012)Wainwright, Jordan, and
  Duchi]{wainwright2012privacy}
M.~J. Wainwright, M.~I. Jordan, and J.~C. Duchi.
\newblock Privacy aware learning.
\newblock In \emph{Advances in Neural Information Processing Systems (NIPS)},
  2012.

\bibitem[Xiao and Zhang(2014)]{proxsvrg}
L.~Xiao and T.~Zhang.
\newblock A proximal stochastic gradient method with progressive variance
  reduction.
\newblock \emph{{SIAM} Journal on Optimization}, 24\penalty0 (4):\penalty0
  2057--2075, 2014.

\bibitem[Zheng and Kwok(2018)]{zheng2018lightweight}
S.~Zheng and J.~T. Kwok.
\newblock Lightweight stochastic optimization for minimizing finite sums with
  infinite data.
\newblock In \emph{Proceedings of the International Conferences on Machine
  Learning (ICML)}, 2018.

\bibitem[Zheng et~al.(2016)Zheng, Song, Leung, and
  Goodfellow]{zheng2016improving}
S.~Zheng, Y.~Song, T.~Leung, and I.~Goodfellow.
\newblock Improving the robustness of deep neural networks via stability
  training.
\newblock In \emph{Proceedings of the Conference on Computer Vision and Pattern
  Recognition (CVPR)}, 2016.

\bibitem[Zhou(2019)]{zhou2018direct}
K.~Zhou.
\newblock Direct acceleration of {SAGA} using sampled negative momentum.
\newblock In \emph{Proceedings of the International Conference on Artificial
  Intelligence and Statistics (AISTATS)}, 2019.

\bibitem[Zhou et~al.(2018)Zhou, Shang, and Cheng]{MiG2018}
K.~Zhou, F.~Shang, and J.~Cheng.
\newblock A simple stochastic variance reduced algorithm with fast convergence
  rates.
\newblock In \emph{Proceedings of the International Conferences on Machine
  Learning (ICML)}, 2018.

\end{thebibliography}

\end{document}